%% file: main.tex
\documentclass{article}
\usepackage{float}
\usepackage{graphicx}
\usepackage{adjustbox}
\usepackage{amsmath}
\usepackage{algorithm}
\usepackage[noend]{algpseudocode}
\usepackage{multirow}
\usepackage{booktabs}
\usepackage{enumitem}
\usepackage{comment}
\usepackage{mathtools}

\usepackage{wrapfig}
    \PassOptionsToPackage{numbers, compress}{natbib}
    \usepackage[final]{neurips_2024}

\usepackage[utf8]{inputenc} %
\usepackage[T1]{fontenc}    %
\usepackage{hyperref}       %
\usepackage{url}            %
\usepackage{booktabs}       %
\usepackage{amsfonts}       %
\usepackage{nicefrac}       %
\usepackage{microtype}      %
\usepackage{xcolor}         %

\usepackage{bbm}
\usepackage{amsthm}
\usepackage{natbib}
\usepackage{amsmath}

\DeclareMathOperator*{\argmin}{arg\,min}

\newtheorem{thm}{Theorem}
\newtheorem{lemma}{Lemma}

\newtheorem{cor}{Corollary}
\newtheorem*{remark}{Remark}

\input{newc}
\newcommand{\bdw}{\mathbf{w}}
\newcommand{\rmin}{r_m}

\title{Stabilizing Linear Passive-Aggressive Online Learning with Weighted Reservoir Sampling}

\author{%
  Skyler Wu \\
  Booz Allen Hamilton\\
  Harvard University\\
  Stanford University\\
  \texttt{wu\_skyler@bah.com} \\
  \And
  Fred Lu \\
  Booz Allen Hamilton\\
  University of Maryland, Baltimore County\\
  \texttt{lu\_fred@bah.com}\\
  \And
  Edward Raff\\
  Booz Allen Hamilton\\
  University of Maryland, Baltimore County\\
  \texttt{raff\_edward@bah.com} \\
  \And
  James Holt\\
  Laboratory for Physical Sciences\\
  \texttt{holt@lps.umd.edu} \\
}

\begin{document}

\maketitle

\begin{abstract}
  Online learning methods, like the seminal Passive-Aggressive (PA) classifier, are still highly effective for high-dimensional streaming data, out-of-core processing, and other throughput-sensitive applications. Many such algorithms rely on fast adaptation to individual errors as a key to their convergence. While such algorithms enjoy low theoretical regret, in real-world deployment they can be sensitive to individual outliers that cause the algorithm to over-correct. When such outliers occur at the end of the data stream, this can cause the final solution to have unexpectedly low accuracy. We design a weighted reservoir sampling (WRS) approach to obtain a stable ensemble model from the sequence of solutions without requiring additional passes over the data, hold-out sets, or a growing amount of memory. Our key insight is that good solutions tend to be error-free for more iterations than bad solutions, and thus, the number of passive rounds provides an estimate of a solution's relative quality. Our reservoir thus contains $K$ previous intermediate weight vectors with high survival times. We demonstrate our WRS approach on the Passive-Aggressive Classifier (PAC) and First-Order Sparse Online Learning (FSOL), where our method consistently and significantly outperforms the unmodified approach. We show that the risk of the ensemble classifier is bounded with respect to the regret of the underlying online learning method.
\end{abstract}

\section{Introduction}

Online learning algorithms are especially attractive when working with high-volume and high-dimensional streaming data, out-of-core processing, and other throughput-sensitive applications \cite{zhao2020unified}. For example, the seminal Vowpal Wabbit uses importance-weighted online learning algorithms \cite{karampatziakis2011online} to reach high quality solutions quickly, with an optional second pass using LBFGS to refine the solution. The MOA library still uses the Pegasos algorithm as its linear classifier \cite{bifet2010moa}. Most relevantly, online learning algorithms are particularly appealing for binary classification tasks, such as web spam classification \cite{zhao2020unified}. Such algorithms often enjoy fast theoretical convergence rates due to their fast adaptation to errors on individual data points, as opposed to batch or offline learning.

However, in real-world deployment, \textbf{online algorithms can be very sensitive to noisy observations in the data stream and over-correct, resulting in out-of-sample performance dropping precipitously between timesteps.} Indeed, in many cases (see Figure \ref{fig:exhibition}), an online learning algorithm might achieve over $90\%$ test accuracy after a given timestep, but then see its test accuracy drop by $20-30\%$ after over-correcting on the next observation in the data stream. In many real-world settings, it may be infeasible computationally or memory-wise to maintain a hold-out evaluation set to select the highest-performance solutions learned by our online algorithm. It may also be practically infeasible to train our classifier over multiple passes of a given dataset or when online algorithms are used for ``any-time'' ready predictions.

In this paper, we introduce a weighted reservoir sampling (WRS) \cite{efraimidis2006weighted} based approach that dramatically mitigates the aforementioned accuracy fluctuations. Our proposed method, WRS-Augmented Training (WAT), neither requires a hold-out evaluation set nor additional passes over the training data. Most importantly, WAT can be used to stabilize \textit{any} passive-aggressive online learning algorithm.
We demonstrate the promise of our WAT method on the Passive Aggressive Classifier (PAC) \cite{crammer2006online} and First-Order Sparse Online Learning (FSOL) \cite{zhao2020unified} methods --- creating two new methods PAC-WRS and FSOL-WRS. \textbf{Strikingly, across 16 benchmark datasets, WAT is able to mitigate test accuracy fluctuations in all 16 datasets on FSOL and 14 datasets on PAC.}
To analyze the theoretical effectiveness of our method, we situate our approach in the online-to-batch conversion literature, enabling us to obtain generalization bounds on i.i.d. data streams.

\begin{figure}[t]
     \centering
     \includegraphics[width=0.70\textwidth,
     ]{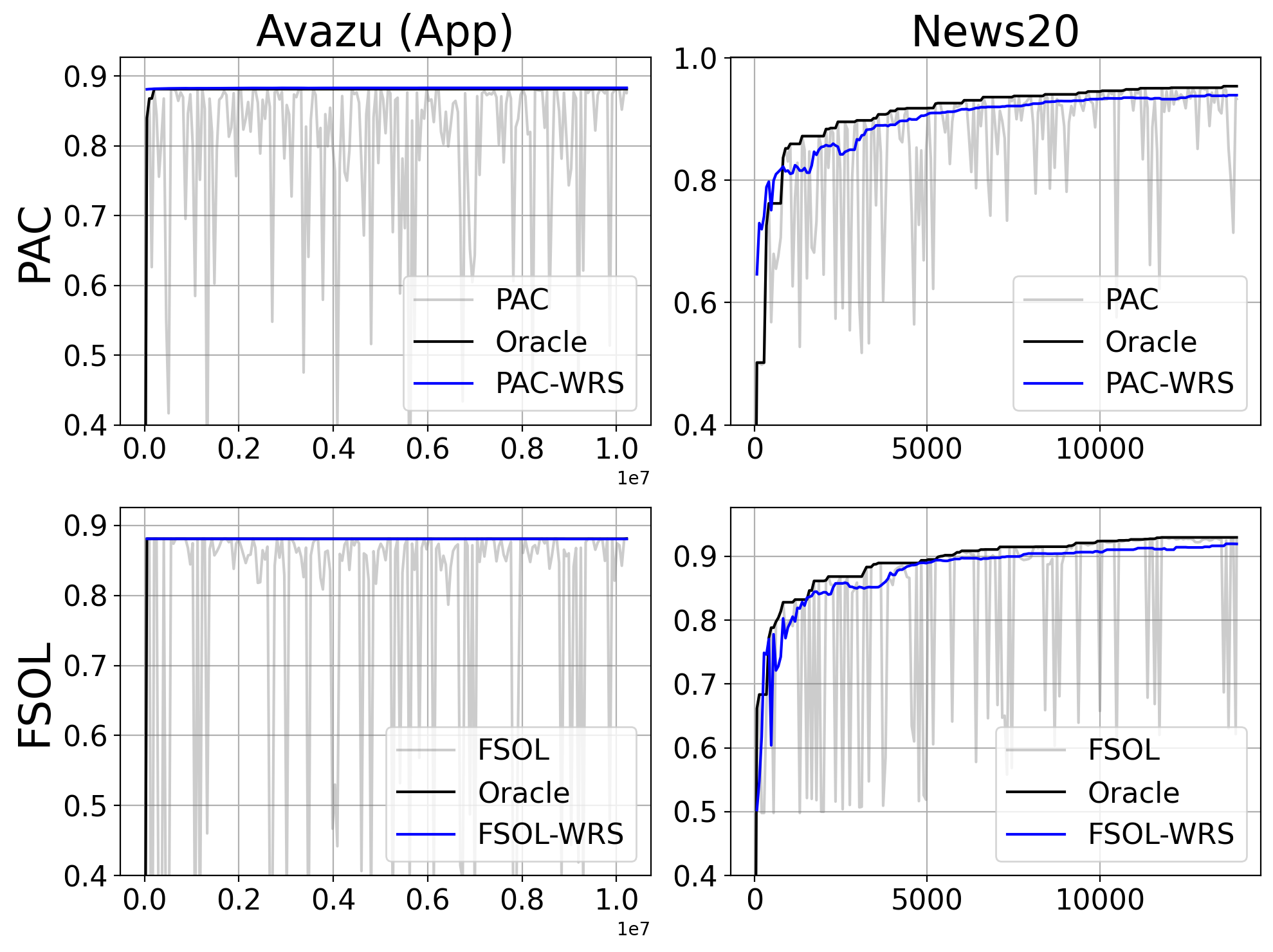}
     \caption{Test accuracies ($y$-axis) over timestep ($x$-axis) for PAC-WRS and FSOL-WRS on Avazu (App) and News20. \textbf{Light grey lines:} test accuracies of the baseline methods --- PAC or FSOL --- at each timestep. \textbf{Solid black lines:} test accuracies of the ``oracle" models, computed as the cumulative maximum of the baselines. \textbf{Solid blue lines:} test accuracies of WRS-enhanced models. Note massive fluctuations of grey lines and stability of blue lines. All variants shown are using standard sampling weights for WRS, with simple-averaging.}
     \label{fig:exhibition}
\end{figure}

\section{Review of related work}

Our work is motivated in part by a real-world need in malware detection~\cite{Raff2020d}, in which large datasets make online methods particularly attractive~\cite{2018arXiv180404637A,sorel}, and a naturally noisy labeling process inhibits standard passive-aggressive methods~\cite{MALDICT,agtr,Joyce2021,10.1145/3605764.3623907}. We show results for the EMBER malware benchmark in Appendix \ref{ember}.

\textbf{Online learning.}
In online learning for linear binary classification, one maintains a solution vector $\bdw_t \in \mathbb{R}^D$, where $D$ is the dimensionality of our data points (usually quite large), and $t$ represents the timestep. At each timestep $t$, we observe a single observation $(\bdx_t, y_t)$ from a high-throughput data stream, with feature vector $\bdx_t \in \mathbb{R}^D$ and class label $y_t \in \{ +1, -1 \}$. Given $(\bdx_t, y_t)$, an online learning algorithm will make a minor (potentially no effect) update to $\bdw_t$ to output $\bdw_{t+1}$, before receiving the next $(\bdx_{t+1}, y_{t+1})$ in the data stream. The classification rule using $\bdw_t$ on any test point $\bdx^*$ is simply $\hat{y}^* = \sign(\bdw_t^\top \bdx^*)$. 
The goal is that as $t \to \infty$, the sequence of $\bdw_t$ will enjoy low cumulative loss.
Towards this end, many online learning algorithms with various update rules have been proposed in the literature, including the Passive-Aggressive Classifier (PAC)~\cite{crammer2006online}, the Adaptive Regularization of Weight Vectors (AROW) methods~\cite{crammer2013adaptive} and the Adaptive Subgradient method (ADAGRAD)~\cite{duchi2011adaptive}. Some online learning algorithms have also been designed to specifically learn sparse solutions for $\bdw_t$ (proportion of zero entries), such as Truncated Gradient method~\cite{langford2009sparse}, Stochastic MIrror Descent Algorithm made Sparse method (SMIDAS)~\cite{shalev2009stochastic}, Regularized Dual Averaging (RDA) method~\cite{xiao2009dual}, and First-Order and Second-Order Sparse Online Learning methods (FSOL and SSOL)~\cite{zhao2020unified}. In general, these sparsity-inducing methods are powered by some combination of solution truncation and L1 norm minimization.
\textbf{Passive-aggressive online learning.}
Within the family of online learning algorithms, a \textit{passive-aggressive} algorithm is one whose update rule \textit{makes no update} to $\bdw_t$ if $(\bdx_t, y_t)$ is classified correctly with sufficient margin.
That is, we \textit{passively} leave $\bdw_{t+1} = \bdw_t$. If a margin error occurs, we \textit{aggressively} update $\bdw_{t+1}$ such that the error is fully correct (though a regularization penalty $C$ will tamper the degree of aggressiveness). 
Usually, correct classification with sufficient margin is defined using the \textit{hinge loss} $\ell$ --- the algorithm remains passive at timestep $t$ if 
$ 
\ell\left(\bdw_t; \left(\bdx_t, y_t\right)\right) = \max\left( 1 - y_t  \bdw_t^\top \bdx_t, 0 \right) = 0. $
Next, we introduce the two passive-aggressive algorithms that we will use to test our WRS-Augmented Training method.
\textit{Passive-Aggressive Classifier (PAC).}
Introduced by Crammer et al., PAC is actually a family of three algorithms: PA, PA-I, and PA-II \cite{crammer2006online}. The base PA algorithm update rule seeks to solve the following constrained optimization problem:
\[ \bdw_{t+1} = \argmin_{\bdw \in \mathbb{R}^D} \frac{1}{2}\| \bdw - \bdw_t \|^2 \;\; \text{s.t.} \;\; \ell(\bdw; (\bdx_t, y_t)) = 0.\]
Because of the hard constraint of forcing $\bdw_{t+1}$ to satisfy $\ell(\bdw; (\bdx_t, y_t)) = 0$, the optimization is particularly vulnerable to noisy data.
As such, Crammer et al. introduce a new constrained optimization function with a slackness hyperparameter $C_{err}$ to allow for some residual hinge loss and induce less aggressive, but presumably more stable updates:
\[ \bdw_{t+1} = \argmin_{\bdw \in \mathbb{R}^D} \frac{1}{2}\| \bdw - \bdw_t \|^2 + C_{err} \epsilon^m \;\; \text{s.t.} \;\; \ell(\bdw; (\bdx_t, y_t)) \leq \epsilon \;\; \text{and} \;\; \epsilon \geq 0,\]
where setting $m=1$ corresponds to PA-I and $m=2$ to PA-II. From initial testing, PA-I and PA-II performed very similarly, with a slight edge to PA-II. As such, for this paper, we will focus on PA-II, which performs the following closed-form update when in aggressive mode \cite{crammer2006online}:
$ \bdw_{t+1} = \bdw_t + \frac{\ell(\bdw_t; (\bdx_t, y_t))}{\| \bdx_t \|^2 + \frac{1}{2C_{err}}} y_t \bdx_t. $
For the remainder of this paper, ``PAC" will refer to PA-II.

\textit{First-Order Sparse Online Learning (FSOL).}
Introduced by Zhao et al. \cite{zhao2020unified}, FSOL is a passive-aggressive algorithm which attempts to find sparse solutions for $\bdw_t$. Governed by a learning rate $\eta$ and a sparsity parameter $\lambda$, FSOL keeps track of two vectors $\bdtheta_t, \bdw_t \in \mathbb{R}^D$ and performs the following update rules when in aggressive mode \cite{zhao2020unified}:
\[ \bdtheta_{t+1} = \bdtheta_t + \eta y_t \bdx_t; \;\; \bdw_{t+1} = \sign(\bdtheta_{t+1}) \odot [|\bdtheta_{t+1}| - \lambda_t]_+, \]
where $\lambda_t = \eta \lambda$ and $[\bdv]_+$ takes the maximum of each element in $\bdv$ and $0$. Zhao et al. note that the above update rules are identical to that of Xiao's RDA method with soft $1$-norm regularization \cite{xiao2009dual, zhao2020unified}.

\textbf{Weighted Reservoir Sampling (WRS).}
Suppose we have a collection of items $V = \{\bdv_1, \dots, \bdv_T\}$, with corresponding nonnegative weights $w_1, \dots, w_T$. Our goal is to collect a size-$K$ weighted random sample from $V$ \textit{in one pass} (imagine this process is indexed by time), where the population size $T = |V|$ is potentially unknown. Introduced by Efraimidis and Spirakis \cite{efraimidis2006weighted}, \textit{weighted random sampling with a reservoir}, which we shorten to \textit{weighted reservoir sampling} (WRS), is an algorithm that allows us to collect such a size-$K$ weighted random sample under the aforementioned conditions. Specifically, as we are making our one pass through the items in $V$, at each timestep $t$, we maintain and update a \textit{reservoir} --- a temporary storage unit with a maximum capacity of $K$ items, with each item in the reservoir a potential candidate for our final size-$K$ sample. At time $T$, the $K$ items that are currently in the reservoir will constitute our sample of size-$K$. We invite the interested reader to look at Algorithm A in \cite{efraimidis2006weighted} for specific details.

\textbf{Online-to-batch conversion.}
Online learning algorithms such as PAC generally do not impose any restrictions on the distribution of the training data sequence. Their \textit{regret} bounds aim to control the cumulative loss $M_T$ of the algorithm over any sequence of data, compared with a minimal-loss fixed model $\hat{\bdw}$:
$\mathrm{reg}_T \coloneqq \sum_{t=1}^T \ell(\bdw_t; z_t) - \sum_{t=1}^T \ell(\hat{\bdw}; z_t) = M_T - \sum_{t=1}^T \ell(\hat{\bdw}; z_t)$.
When using a fixed model to classify unseen data, we need to impose an i.i.d. assumption on the data stream in order to measure the \textit{risk},
or expected generalization error,
of the model.
Note that the distribution $\mathcal{D}$ itself is arbitrary and can still permit outliers or mixtures.
Then the population \textit{risk} is defined as
$R_\mathcal{D}(\bdw) \coloneqq \mathbb{E}_{z\sim \mathcal{D}}[\ell(\bdw; z)]$.

To theoretically describe our algorithm, we leverage work on online-to-batch conversion, which takes an online algorithm with known regret bounds on an i.i.d. sequence of data and extracts a stable final model with low risk. For example, in the online perceptron algorithm, earlier work studied the \textit{pocket} approach, which selects the longest-surviving model in the sequence as the final model~\cite{freund1998large,muselli1997convergence}.
Other well-known approaches use the uniform average of the whole model sequence or the best-performing model over a validation set~\cite{cesa2004generalization}.

As will be seen, our method generalizes these approaches to utilize multiple long-survival models as an ensemble model.
Furthermore, we will introduce novel improvements, including a limited-size reservoir with probabilistic sampling.
The risk bounds for our WAT model leverage improved techniques from~\cite{cesa2008improved,dekel2008online}. Our experimental results also demonstrate that our novel conversion technique outperforms prior baselines (see Appendix \ref{other_averaging_schemes}).

\section{Our method: WRS-Augmented Training (WAT)}

In one extreme, if a candidate solution $\bdw_t$ from a passive-aggressive algorithm had \textit{perfect classification with sufficient margin} on any given data point, then the \textit{subsequent number of passive steps} taken after time $t$ (i.e., number of timesteps that our algorithm is in passive mode before going aggressive again) would be infinite. In the other extreme, if a candidate solution $\bdw_t$ had extremely low performance, then our passive-aggressive algorithm is likely to go aggressive very soon after time $t$, implying a very small subsequent number of passive steps after time $t$. In short, our key insight is that high-performing solutions $\bdw_t$ tend to be error-free for more iterations than low-performing solutions. As such, the subsequent number of passive steps taken after the formation of $\bdw_t$ provides an estimate of $\bdw_t$'s relative quality (i.e., test accuracy).

However, we do not want to take the intermediate solution $\bdw_t$ that had the most passive updates as this, too, can be noise (and luck) sensitive. Ideally, we would like to sample from the merging distribution of $\bdw_t$ as they occur, and take an average of those solutions to obtain a singular, highly robust, solution vector that performs well with little variance. But, we do not wish to store all $\bdw_t$ due to intractability. 

Putting these thoughts together, our WRS-Augmented Training (WAT) method functions as follows. Given a base passive-aggressive algorithm (e.g., PAC or FSOL), we will run said algorithm through our data stream $\{ (\bdx_t, y_t )\}_{t=1}^\top$, as normal, but keep a size-$K$ reservoir of promising candidate solutions. The reservoir approach allows us to run through our data stream and collect a weighted random sample of candidate solutions of size-$K$, \textit{weighted by their subsequent number of passive steps} and without storing all intermediate solutions.

Naturally, this setup is suited for Efraimidis and Spirakis's WRS algorithm. Procedurally, every timestep that our algorithm goes aggressive, we obtain a new active candidate solution. Right before we apply our aggressive mode update rule, we will add the outgoing candidate solution to our size-$K$ reservoir (and remove a current resident of the reservoir, if necessary) following the steps of the WRS algorithm. At any timestep $t$, we can form an ensemble solution $\bdw_{\textnormal{WRS}}$ by taking an average of the candidate solutions currently in our reservoir. Hopefully, at any timestep $t$, $\bdw_{\textnormal{WRS}}$ will have more stable test performance than the current active candidate solution $\bdw_t$.

\begin{algorithm}[ht!]
    \caption{WRS-Augmented Training (WAT)}
    \label{alg:General-WRS}
    \begin{algorithmic}[1]
        \renewcommand{\algorithmicrequire}{\textbf{Input:}}
        \renewcommand{\algorithmicensure}{\textbf{Output:}}
        \Require $\bdw$ - initial solution vector, $\{ (\bdx_t, y_t )\}_{t=1}^T$ - data stream, $K$ - reservoir size, $\textnormal{WS}$ - weighting scheme (``Standard" or ``Exponential"), $\textnormal{AS}$ - averaging scheme (``Simple Average" or ``Weighted Average"), $\textnormal{VZ}$ - voting-based zeroing (True or False), and other base-method-specific hyperparameters $\mathcal{H}$ (e.g., for PAC or FSOL)
        \Ensure WRS-stabilized solution vector $\bdw_{\textnormal{WRS}}$.
        \Function{WAT}{$\bdw, \{ (\bdx_t, y_t )\}_{t=1}^T; K, \textnormal{WS}, \textnormal{AS}, \textnormal{VZ}, \mathcal{H}$}
        \State \# Initializing intermediate data structures
        \State $s \to 0$ \Comment{Counter for subsequent passive steps of current solution candidate}
        \State $\mathcal{R} \to []$ \Comment{Size-$K$ reservoir for storing promising solutions, as array}
        \State $\bdb, \bdk \to [], []$ \Comment{Size-$K$ arrays for weights $b_r$ and auxiliary $k_r$ values for solutions in $\mathcal{R}$}
        \State \# At each timestep, we observe $(\bdx_t, y_t)$
        \For{$t \gets 1$ \textbf{to} $T$}
            \If{$\ell(\bdw; \bdx_t, y_t) > 0$} \Comment{If made error, in aggressive mode}
            \State \# Terminate current solution, probabilistically add to reservoir using WRS \cite{efraimidis2006weighted}
            \State Draw $u^* \sim \textnormal{Unif}(0, 1)$
            \If{$\textnormal{WS} == ``\textnormal{Standard}"$} 
            \State $b^* \gets s$
            \ElsIf{$\textnormal{WS} == ``\textnormal{Exponential}"$}
            \State $b^* \gets \exp(s)$
            \EndIf
            \State $k^* \gets (u^{*})^{\frac{1}{b^{*} + \epsilon}}$ \Comment{$\epsilon = 10^{-8}$ to prevent division by $0$}
            \State $\tau \gets \min_{j \in 1, \dots, K} \bdk[j]$; $i \gets \argmin_{j \in 1, \dots, K}\bdk[j]$
            \If{$k^* > \tau$ or $\mathcal{R}$ is not full with $K$ solutions}
            \State $\mathcal{R}[i] \gets \bdw, \bdb[i] \gets b^*, \bdk[i] \gets k^*$
            \EndIf
            \State \# Base method update rule
            \State $\bdw \gets \bdw + g(\dots)$ \Comment{$g(\dots)$ specific to base algorithm (e.g., PAC or FSOL)}
            \State $s \gets 0$ \Comment{Reset number of subsequent passive steps}
            \Else \Comment{correctly-classified, still in passive mode}
            \State $s \gets s + 1$ \Comment{Increment number of subsequent passive steps}
            \EndIf
            \State \# Forming our PAC-WRS solution
            \If{$\textnormal{AS} == ``\textnormal{Simple Average}"$}
            \State $\bdw_{\textnormal{WRS}} \gets \frac{1}{K} \sum_{j =1}^K \mathcal{R}[j]$ \Comment{Simple average of solutions in reservoir}
            \ElsIf{$\textnormal{AS} == ``\textnormal{Weighted Average}"$}
            \State $\bdw_{\textnormal{WRS}} \gets \sum_{j = 1}^K \mathbf{b}[j] \mathcal{R}[j] / (\sum_{j = 1}^K \mathbf{b}[j])$ \Comment{Weighted avg. of solutions in reservoir}
            \EndIf
            \State \# Voting-based zeroing for extra sparsity
            \If{$\textnormal{VZ}$ is \texttt{True}}
            \State Zero out entries in $\bdw_{\textnormal{WRS}}$ where the majority of solutions $\mathcal{R}[j]$ contain zeroes.
        \EndIf
        \EndFor
        \State \textbf{Return} $\bdw_{\textnormal{WRS}}$.
        \EndFunction
    \end{algorithmic}
\end{algorithm}

\subsection{WAT variants}

We will experiment with a few additional variants of the WAT method. First, instead of weighting using the subsequent number of passive steps (which we denote as \textit{standard weights}), what if we weight using the exponentiated subsequent number of passive steps (which we denote as \textit{exponential weights})? The idea with exponential weights is that it is closer in practice to deterministically picking the candidate solutions with the largest number of subsequent passive steps, while still maintaining some stochasticity --- i.e., a ``greedier" policy. Second, when constructing our ensemble solution $\bdw_{\textnormal{WRS}}$, should we take a \textit{simple average} of the residents in our reservoir or a \textit{weighted average}? Third, there are reasonable concerns that constructing $\bdw_{\textnormal{WRS}}$ via averaging might negate the sparsity advantages of a method like FSOL, due to different candidates in the reservoir containing $0$s in different entries. However, if the majority of candidate solutions in the reservoir contain $0$s at a given entry, what if we tried zeroing out said entry in $\bdw_{\textnormal{WRS}}$, as it is likely to be uninformative? We denote this add-on as \textit{voting-based zeroing}. 

We present the WAT method in full detail in Algorithm \ref{alg:General-WRS}. To clarify, the elements in the data structures $\mathcal{R}$, $\mathbf{b}$, and $\mathbf{k}$ (all of which are size-$K$ arrays) are paired with each other, so that when we add/remove an element in $\mathcal{R}$, the corresponding elements in $\mathbf{b}$ and $\mathbf{k}$ are removed as well. To be fully clear, $\mathcal{R}$ contains candidate solution vectors $\mathcal{R}[1], \dots, \mathcal{R}[K]$, with each $\mathcal{R}[k] \in \mathbb{R}^D$. The vector $\bdb$ contains scalar values $\bdb[1], \dots, \bdb[K]$, and likewise for the vector $\bdk$.

\subsection{Theoretical Analysis}
\label{math-theory}

One of the goals of our method is to choose a set of effective models as the algorithm runs live,
without the need for expensive evaluation on a validation set.
In this section we first provide \textit{validity} arguments for using observed survival as a proxy to select high-accuracy models.
These are finite-sample bounds based on reasoning about the reservoir $\mathcal{R}$.
Note that we are interested in studying generalization to unseen data, a complementary setting to prior work which shows convergence in terms of training set error for the perceptron model~\cite{muselli1997convergence}. After establishing validity, we will turn to learning bounds for our ensemble model risk in the i.i.d. setting. Proofs are deferred to Appendix \ref{appendix-proofs}.

\input{maths/1_validity}

\input{maths/2_learning}

\begin{table}[!h]
    \caption{Sizes, dimensions, and sparsities of all datasets used for numerical experiments.}
    \centering
    \begin{adjustbox}{width=1\textwidth}
    \small
    \begin{tabular}{l|rrrrr}
    \toprule
    \textbf{Dataset} & $\mathbf{D}$ & $\mathbf{N}$ & $\mathbf{N_{train}}$ & $\mathbf{N_{test}}$ & \textbf{Sparsity} \\
    \midrule
    \href{https://www.kaggle.com/c/avazu-ctr-prediction/data}{Avazu (App)} \cite{misc_avazu} (via \href{https://www.csie.ntu.edu.tw/~cjlin/libsvmtools/datasets/multiclass.html#mnist8m}{LIBSVM}) & 1000000 & 14596137 & 10217295 & 4378842 & 0.999985 \\
    \href{https://www.kaggle.com/c/avazu-ctr-prediction/data}{Avazu (Site)} \cite{misc_avazu} (via \href{https://www.csie.ntu.edu.tw/~cjlin/libsvmtools/datasets/multiclass.html#mnist8m}{LIBSVM}) & 1000000 & 25832830 & 18082981 & 7749849 & 0.999985 \\
    \href{https://ailab.criteo.com/ressources/}{Criteo} (via \href{https://www.csie.ntu.edu.tw/~cjlin/libsvmtools/datasets/multiclass.html#mnist8m}{LIBSVM}) & 1000000 & 51882752 & 36317926 & 15564826 & 0.999961 \\
    Dexter \cite{misc_dexter_168} & 20000 & 600 & 420 & 180 & 0.995319 \\
    Dorothea \cite{misc_dorothea_169} & 100000 & 1150 & 805 & 345 & 0.990909 \\
    KDD2010 (Algebra) \cite{misc_kdd} (via \href{https://www.csie.ntu.edu.tw/~cjlin/libsvmtools/datasets/multiclass.html#mnist8m}{LIBSVM}) & 20216830 & 8918054 & 6242637 & 2675417 & 0.999998 \\
    MNIST8 (4+9) \cite{liang2021screening, loosli-canu-bottou-2006} (via \href{https://www.csie.ntu.edu.tw/~cjlin/libsvmtools/datasets/multiclass.html#mnist8m}{LIBSVM}) & 784 & 1591785 & 1114249 & 477536 & 0.757170 \\
    News20 \cite{keerthi2005modified} (via \href{https://www.csie.ntu.edu.tw/~cjlin/libsvmtools/datasets/multiclass.html#mnist8m}{LIBSVM}) & 1355191 & 19954 & 13967 & 5987 & 0.999664 \\
    Newsgroups (Binary, CS) \cite{Lang95, scikit-learn} (via \href{https://scikit-learn.org/stable/modules/generated/sklearn.datasets.fetch_20newsgroups_vectorized.html}{sklearn}) & 101631 & 18311 & 12817 & 5494 & 0.999049 \\
    PCMAC \cite{li2018feature} & 3289 & 1943 & 1360 & 583 & 0.985418 \\
    RCV1 \cite{lewis2004rcv1} (via \href{https://www.csie.ntu.edu.tw/~cjlin/libsvmtools/datasets/multiclass.html#mnist8m}{LIBSVM}) & 47236 & 697641 & 488348 & 209293 & 0.998451 \\
    Real-Sim (\href{https://people.cs.umass.edu/~mccallum/data.html}{McCallum} via \href{https://www.csie.ntu.edu.tw/~cjlin/libsvmtools/datasets/multiclass.html#mnist8m}{LIBSVM})& 20958 & 72201 & 50540 & 21661 & 0.997549 \\
    SST-2 \cite{socher2013recursive} & 13757 & 67337 & 47135 & 20202 & 0.999421 \\
    URL \cite{ma2009identifying} (via \href{https://www.csie.ntu.edu.tw/~cjlin/libsvmtools/datasets/multiclass.html#mnist8m}{LIBSVM}) & 3231961 & 2396130 & 1677291 & 718839 & 0.999964 \\
    W8A \cite{platt1998fast} (via \href{https://www.csie.ntu.edu.tw/~cjlin/libsvmtools/datasets/multiclass.html#mnist8m}{LIBSVM}) & 300 & 59245 & 41471 & 17774 & 0.957585 \\
    Webspam \cite{webb2006introducing} (via \href{https://www.csie.ntu.edu.tw/~cjlin/libsvmtools/datasets/multiclass.html#mnist8m}{LIBSVM}) & 254 & 350000 & 244999 & 105001 & 0.664833 \\
    \bottomrule
    \end{tabular}
    \end{adjustbox}
    \label{table:datasets}
\end{table}

\section{Numerical experiments}
\label{numerical-experiments}

We combine WAT with base PAC and FSOL, forming two new methods PAC-WRS and FSOL-WRS\footnote{Code available at \url{https://github.com/FutureComputing4AI/Weighted-Reservoir-Sampling-Augmented-Training}}.
We evaluate their performances across 16 binary classification datasets listed in Table \ref{table:datasets}. Please see Appendix \ref{methods-details-appendix-pac} for more dataset details.
We are interested in three metrics: 
    \textbf{1) Final test accuracy:} proportion of test set data points correctly-classified by solution obtained after making one pass through the training data.
    \textbf{2) Final sparsity:} proportion of zeroes in our classification solution obtained after making one pass through the training data.
    \textbf{3) Relative oracle performance (ROP):} let $p_{t,\text{base}}$ be the test accuracy of our base model (either PAC or FSOL) at time $t$ and $p_{t}$ be the test accuracy of our model of interest (either base PAC, base FSOL, PAC-WRS, or FSOL-WRS). Then, define
    $p_{t,\text{oracle}} = \max_{\text{base}}  p_{t,\text{base}}$
    to be the cumulative maximum test accuracy of the base method at time $t$. In other words, $p_{t,\text{oracle}}$ represents the highest performance we could obtain if we had an oracle telling us which candidate solution vector we encountered was best.
    Then, define $\text{ROP} = \frac{1}{T}\sum_{t=1}^\top (p_{t,\text{oracle}} - p_{t})$. Intuitively, if a method has very stable test accuracy over time with minimal fluctuations, then ROP should be close to $0$, or negative (i.e. achieving a higher test accuracy than the oracle). In contrast, a large, positive ROP suggests fluctuations in test accuracy.

Because we are interested in the stability of WAT's test-set accuracy over the entirety of the training run, it is most illustrative to look at figures when possible. Representative figures will be shown in the main paper, with all figures for all datasets in the appendix. 
When running any algorithm (PAC, FSOL, PAC-WRS, or FSOL-WRS) on any dataset, we perform a random 70/30 train-test split. We begin by tuning the $C_{err}$, $\eta$, and $\lambda$ hyperparameters for base PAC and FSOL, with details located in Appendix \ref{methods-details-appendix-pac}. For PAC-WRS and FSOL-WRS, we use the hyperparameters for the corresponding base models, but try all possible WAT variants of weighting scheme (standard or exponential), averaging scheme (simple vs. weighted), voting-based zeroing (True or False), and reservoir size $K \in \{ 1, 4, 16, 64 \}$. We perform five trials for each PAC-WRS and FSOL-WRS variant with randomized shuffles of the training and test data, running through the training data only once for each trial. All candidate solutions are initialized as the $\bdzero$ vector. Please see more experimental details in Appendix \ref{methods-details-appendix-pac}.

\subsection{Stabilizing test performance}

Figure \ref{fig:exhibition} shows visually how PAC-WRS and FSOL-WRS are highly-effective at stabilizing base PAC and FSOL's wildly-fluctuating test accuracy on Avazu (App) and News20. Corresponding figures for all 16 datasets and PAC-WRS/FSOL-WRS variants can be found in Appendix \ref{over-time-appendix}. From Figure \ref{fig:exhibition-ROP}, we see that PAC-WRS is highly-effective at reducing ROP compared to base PAC, and that \textbf{the larger the reservoir size $K$, the more stable the resultant test accuracy for $\bdw_{\text{WRS}}$ becomes.} Furthermore, looking more carefully at MNIST8 (4+9), we observe that many PAC-WRS variants were able to achieve \textit{negative} ROP values, suggesting that $\bdw_{\text{WRS}}$ could achieve \textit{higher test accuracies} than even the oracle. Corresponding figures for all 16 datasets can be found in Appendix \ref{errorbars-appendix}. 

\begin{figure}[t]
     \centering
     \includegraphics[width=0.98\textwidth, height=0.28\textwidth]{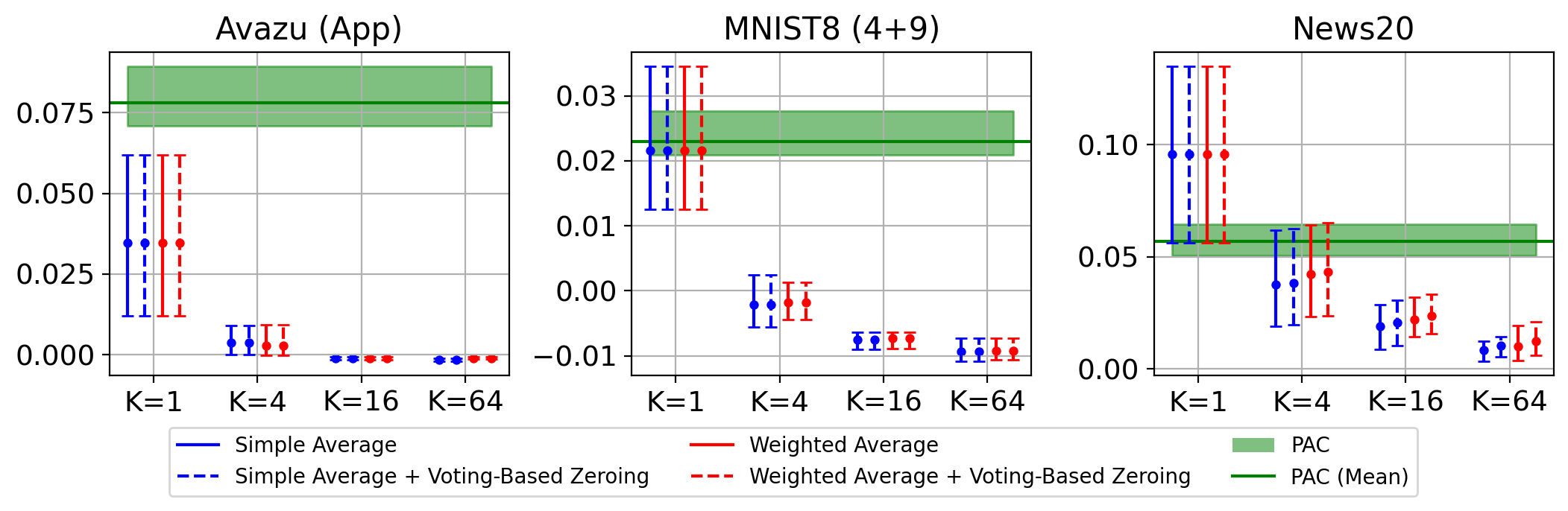}
     \caption{ Relative oracle performances  ($y$-axis) of base PAC and PAC-WRS using standard weights over reservoir sizes $K$ ($x$-axis) on 3 representative datasets. Error bars represent the minimum and maximum values achieved across 5 randomized trials. \textbf{Blue:} WRS-augmented variants via simple average of reservoir members. \textbf{Red:} WRS-augmented variants via weighted average of reservoir members. \textbf{Dotted lines:} indicates voting-based zeroing was performed for additional sparsity. \textbf{Lower values indicate more stable performance.} \normalsize}
     \label{fig:exhibition-ROP}
\end{figure}

\begin{wraptable}[10]{r}{0.5\textwidth}
\vspace{-10pt}
\caption{ Numbers of datasets out of 16 where each PAC-WRS or FSOL-WRS variant with $K=64$ outperformed its corresponding base method (PAC or FSOL), as measured by ROP averaged across 5 randomized trials.\normalsize}
\label{table:num-WRS-win}
\centering
\begin{adjustbox}{width=0.5\textwidth}
\begin{tabular}{c|cc|cc}
  & \multicolumn{2}{c|}{\textbf{Simple Average}} & \multicolumn{2}{c}{\textbf{Weighted Average}} \\
  \midrule
  & Standard & Exp. & Standard & Exponential \\
  \midrule
  \textbf{PAC} & 13 & 14 & 13 & 11 \\
  \textbf{FSOL} & 16 & 13 & 16 & 11 \\
  \hline
\end{tabular}
\end{adjustbox}
\end{wraptable}

From Table \ref{table:num-WRS-win}, we observe that FSOL-WRS with standard weights and simple averaging \textbf{successfully stabilized test accuracy in all 16 tested datasets compared to base FSOL}, as measured via ROP. 

PAC-WRS with exponential weights and simple averaging successfully stabilized test accuracy in 14 of 16 tested datasets compared to base PAC. Following the best practices in \cite{JMLR:v17:benavoli16a, JMLR:v7:demsar06a}, we perform Wilcoxon signed-rank tests for statistical significance on the differences in ROP between PAC/FSOL-WRS versus base PAC and FSOL, taking into account performance on all 16 datasets. 
At a significance level of $\alpha = 0.05$, we find that 
both PAC/FSOL-WRS achieve statistically-significant reductions in ROP compared to their corresponding base models when equipped with standard weights ($p < 0.0386$ in each case, see Table \ref{table:wilcoxon-ROP-WRS} in Appendix \ref{wilcoxon-appendix}). 

These results also suggest that standard weights are preferable to exponential weights. This makes sense because using exponential weights may be too ``greedy," causing the algorithm to overly trust in the number of passive steps as an indicator of test performance,
polluting the reservoir with poor ``lucky" candidate solutions,
and refusing to remove them later.

\textbf{A note on final test accuracy.} From Figure \ref{fig:exhibition-final-test-acc}, we see that
the final test accuracies achieved by FSOL-WRS are not only higher on average than base FSOL, but also have significantly lower variance. On these datasets, there is no significant difference between FSOL-WRS variants with or without voting-based zeroing. As expected, final test accuracy seems to be 
higher for larger values of reservoir size $K$. However,
Wilcoxon signed-rank tests to compare PAC-WRS and FSOL-WRS's improvements in final test accuracy over base PAC and FSOL
indicate that
FSOL-WRS with standard weights yields a statistically-significant improvement in final test accuracy compared to base FSOL, but PAC-WRS does not compared to base PAC (see Table \ref{table:final-test-acc-WRS-wilcoxon} in Appendix \ref{wilcoxon-appendix}). 
One hypothesis for this discrepancy is that it is relatively unlikely for any given training data point to cause a massive drop in test accuracy. Furthermore, subsequent data points will usually help the base model self-correct (see Appendix \ref{over-time-appendix}). Thus, this explains why the base method's mean final test accuracy after a particular fixed timestep (e.g., $N_{\text{train}}$) will not differ significantly from that of the WRS-augmented method, especially not across only 5 trials.

\begin{figure}[t]
     \centering
     \includegraphics[width=0.98\textwidth, height=0.28\textwidth]{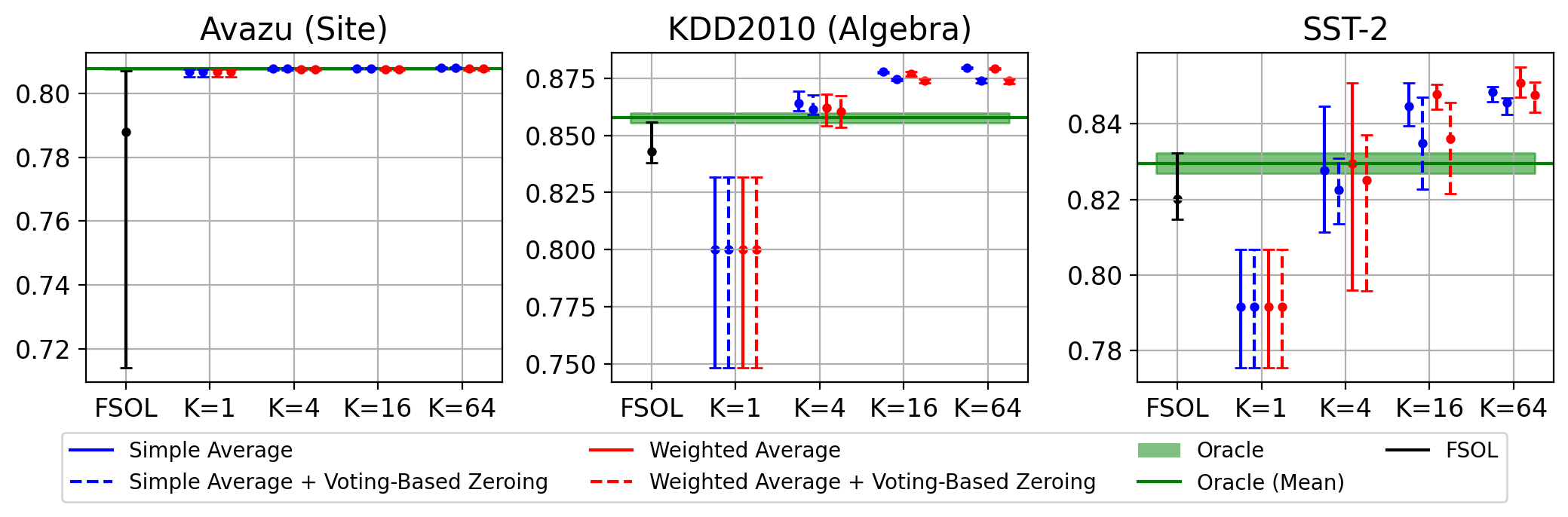}
     \caption{ Final test accuracies ($y$-axis) of base FSOL and FSOL-WRS using standard weights over reservoir sizes $K$ ($x$-axis) on 3 representative datasets. Error bars represent the minimum and maximum values achieved across 5 randomized trials. See Figure \ref{fig:exhibition-ROP} for legend description.\normalsize}
     \label{fig:exhibition-final-test-acc}
\end{figure}

Nonetheless, there are enough data points in the training stream that these massive fluctuations in test accuracy could still happen thousands of times throughout the training process, corroborating
what we saw in Figure~\ref{fig:exhibition}.
If such a drop in test accuracy were to occur 
at an unlucky timestep when the model training is stopped, 
the consequences could be unacceptable.
Furthermore, in continuously updated ``any-time'' environments,
it is hard to assess if one of these drops has occurred. As such,
WAT is valuable as a simple and effective way of preventing such fluctuations in test accuracy.

\begin{figure}[h!]
     \centering
     \includegraphics[width=0.98\textwidth, height=0.28\textwidth]{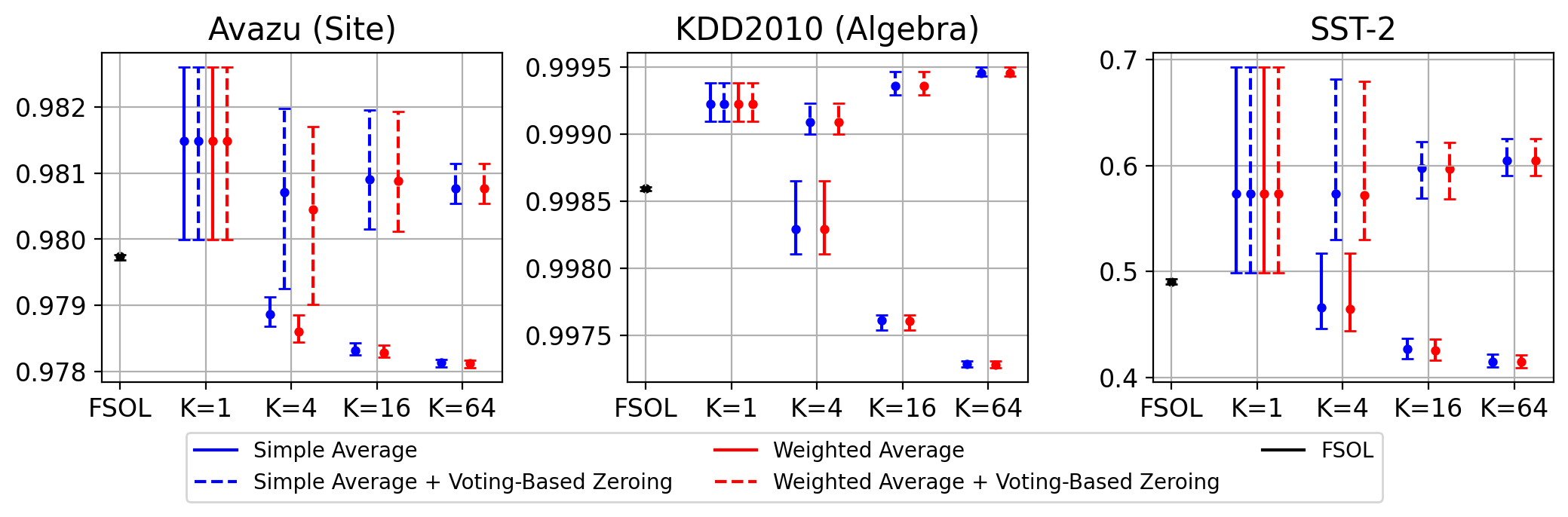}
     \caption{ Final sparsities ($y$-axis) of base FSOL and FSOL-WRS using standard weights over reservoir sizes $K$ ($x$-axis) on 3 representative datasets. Error bars represent the minimum and maximum values achieved across 5 randomized trials. See Figure \ref{fig:exhibition-ROP} for legend description. \normalsize}
     \label{fig:exhibition-sparsity}
\end{figure}

\begin{wraptable}[8]{r}{0.45\textwidth}
\vspace{-28pt}
\caption{ Numbers of datasets out of 16 where each top-$64$ PAC or FSOL variant outperformed its corresponding base method (PAC or FSOL), as measured by relative oracle performance averaged across 5 trials.\normalsize}
\label{table:top-K=64-ROP}
\centering
\begin{adjustbox}{width=0.45\textwidth}
\begin{tabular}{c|cc}
  \midrule
  & Simple Average & Weighted Average \\
  \midrule
  \textbf{PAC} & 10 & 12\\
  \textbf{FSOL} & 15 & 15\\
  \hline
\end{tabular}
\end{adjustbox}
\end{wraptable}

\textbf{Preservation of sparsity.} Finally, we aim to retain the favorable weight sparsity from FSOL. From Figure \ref{fig:exhibition-sparsity}, we observe that even with reservoir size $K=64$, the final sparsities achieved by FSOL-WRS are quite similar to that of base FSOL. Furthermore, we observe that with voting-based zeroing, FSOL-WRS often achieves even higher sparsity than the base model. Please see Appendix \ref{errorbars-appendix} for additional errorbar-type plots and Appendix \ref{over-time-appendix} for sparsity over timestep. One explanation for why even FSOL-WRS variants \textit{without} voting-based zeroing can achieve similar sparsity to base FSOL is that the final reservoir will likely contain candidate solutions from significantly earlier timesteps, when weights were sparser due to fewer aggressive updates.
Thus, WAT can maintain sparsity despite the use of weight averaging.
\subsection{Ablations and extensions}

\textbf{Top-$K$ ablation.} A natural question that one might ask is --- instead of probabilistically sampling candidate solutions, what if we deterministically picked the top-$K$ candidate solutions with the largest subsequent numbers of passive steps and took their simple and/or weighted average? Setting $K=64$, we see from Table \ref{table:top-K=64-ROP} that the best top-$64$ PAC/FSOL variants are not as effective as stabilizing test accuracy as the best PAC/FSOL-WRS variants, as shown in Table \ref{table:num-WRS-win} and measured via ROP. Furthermore, Wilcoxon signed-rank test 
$p$-values from Table \ref{table:p-val-top-K} also suggest that this top-$64$ ablation cannot produce as statistically-significant increases in test accuracy stability compared to WAT, as measured via ROP. Like WAT, top-$K$ is also ineffective at producing statistically-significant increases in final test accuracy on PAC. On FSOL, top-$K$ can produce statistically significant increases in final test accuracy, but these $p$-values are over an order of magnitude larger than FSOL-WRS's (see Table \ref{table:final-test-acc-WRS-wilcoxon} in Appendix \ref{wilcoxon-appendix}). As such, WAT is still the best for stabilizing test accuracy.

\begin{table}[!h]
\caption{ Wilcoxon signed-rank test $p$-values testing whether differences in \textbf{relative oracle performance} and \textbf{final test accuracy} between top-$64$ PAC/FSOL variants and base PAC/FSOL are statistically significant. \normalsize}
\label{table:p-val-top-K}
\centering
\begin{adjustbox}{width=0.75\textwidth}
\begin{tabular}{c|cc|cc}
  & \multicolumn{2}{c|}{\textbf{Relative Oracle Performance}} & \multicolumn{2}{c}{\textbf{Final Test Accuracy}} \\
  \midrule
  & Simple Average & Weighted Average & Simple Average & Weighted Average \\
  \midrule
  \textbf{PAC} & 0.162 & 0.0298 & 0.214 & 0.255 \\
  \textbf{FSOL} & 0.00270 & 0.00270 & 0.0130 & 0.0130 \\
  \hline
\end{tabular}
\end{adjustbox}
\end{table}

\textbf{Additional ablations and comparisons.} For the interested reader, in Appendix \ref{other_averaging_schemes} we include empirical performance comparisons of WRS-Augmented Training against two traditional ensembling mechanisms: moving average (e.g., averaging the most-recently-observed $K=64$ weight vectors at each timestep) and exponential average (e.g., forming an ensemble vector $\bar{\bdw}_t = \gamma \bdw_t + (1-\gamma) \bar{\bdw}_{t-1}$ at each timestep), where $\bdw_t$ is the base algorithm's candidate solution at timestep $t$. In short, especially in more real-world, large-scale settings where evaluation and checking are prohibitively expensive, WRS-Augmented Training is the fastest, most accurate, and most reliable method compared to all the aforementioned baselines.

\textbf{Modifying WAT for non-passive aggressive methods.} While the main theoretical and empirical results in this paper were primarily oriented towards passive-aggressive base models, in Appendix \ref{non-passive-aggressive}, we include empirical simulations of applying a modified form of WRS-Augmented Training on top of three non-passive-aggressive online learning methods: Stochastic Gradient Descent with Momentum \cite{sutskever2013importance}, ADAGRAD \cite{duchi2011adaptive}, and Truncated Gradient Descent \cite{langford2009sparse}. In general, our modified WRS-Augmented Training effectively mitigates test accuracy when it exists, and does minimal harm when it does not.

\section{Conclusion, limitations, and future work}
\label{conclusion+future-work}

In this paper, we introduced WRS-Augmented Training (WAT), a procedure that can be used to stabilize any passive-aggressive online learning algorithm, neither requiring a hold-out evaluation set nor additional passes over the training data. We applied WAT to base PAC and FSOL, demonstrating across 16 datasets that WAT is highly effective at mitigating the massive fluctuations in test accuracy between timesteps that affect many online learning algorithms. WAT runs at minimal cost, with only a fixed $K$ multiple on memory for multiple weight vectors to be saved. 

One limitation of this work is that WAT implicitly assumes that the training and test data come from fixed distributions that do not change over time. However, some data distributions will evolve over time. As such, a candidate solution that entered the reservoir early on due to having a large number of subsequent passive steps
might not actually retain its performance as the data distribution evolves over time. While they may eventually be replaced, non-IID adaptions may be useful in the future. 

\bibliographystyle{ACM-Reference-Format}
\bibliography{references}

\newpage

\appendix
\section{Proofs}
\label{appendix-proofs}
\input{maths/3_proofs}

\newpage
\section{Additional details on methods and experimental setups}
\label{methods-details-appendix-pac}

\subsection{Datasets}

Of the 16 datasets used in our study (see Table \ref{table:datasets}), 14 are directly from existing repositories and/or literature, while 2 of them --- Newsgroups (Binary, CS) and SST-2 --- were modified from pre-existing assets. First, Newsgroups (binary, CS) was formed by taking the original multi-class Newsgroups dataset hosted in \cite{scikit-learn} and combining the labels ``Computers" and ``Science" as a $+1$ class, and all other labels as the $-1$ class. Second, SST-2 is a sentiment classification dataset originally containing text excerpts that we transformed into a linear classification dataset by using the \href{https://scikit-learn.org/stable/modules/generated/sklearn.feature_extraction.text.CountVectorizer.html#sklearn.feature_extraction.text.CountVectorizer}{\texttt{CountVectorizer}} from \cite{scikit-learn}.

\subsection{Base PAC and FSOL hyperparameter tuning}

We start our set of experiments by tuning the $C_{err}$ hyperparameter for base PAC, testing values in the set ${ 10^{-3}, 10^{-2}, \dots, 10^2, 10^3 }$. For base FSOL, following Zhao et al.'s approach, we test $\eta$ values in ${ 2^{-3}, 2^{-2}, \dots, 2^{8}, 2^{9} }$ and $\lambda$ values in ${ 0, 10^{-3}, 10^{-2}, \dots, 10^{2}, 10^{3} }$. For each dataset, we perform five runs of our base PAC and FSOL variants with randomized shuffles of the training and test data, running through the training data only once for each run. For each dataset, we pick the $C_{err}$, $\eta$, and $\lambda$ values corresponding to the base PAC and FSOL variants, within the top 2.5\% in terms of final test accuracy, that had the highest ROP, with all metrics averaged across the five runs. This experimental choice was made to strike a balance between simulating as difficult and risky conditions as possible and still choosing useful base model variants.

\subsection{Metrics logging}

Given that some of our datasets contain up to 36 million training points, it would be computationally and storage-wise unfeasible to record metrics at each timestep. As such, for each dataset, we record metrics at $\approx 200$ evenly-spaced timesteps (exact number depends on divisibility and integer division of $N_{\text{train}}$ by $200$) throughout the training stream, in addition to the initial and final timesteps. As such ROP was also approximated by proxying oracle accuracy by taking the cumulative maximum base model test accuracy across timesteps where metrics were recorded. As shown in the over-time figures in Appendix \ref{over-time-appendix}, such a reduced resolution still tells us a very clear picture, while cutting down computation and storage requirements by many orders of magnitude.

\subsection{Compute requirements and code availability}

All experiments were run on a Linux computing cluster with 32 nodes, each with 40 Intel Xeon E5-2650 CPU cores and a total of 500 GB of RAM per node, managed using SLURM. Nonetheless, no experiments require multiprocessing or multiple cores. Depending on dataset size, some trials could take less than a minute to run, while the largest datasets would take a couple hours at max. However, larger datasets like Criteo will require 32 GB of RAM to load the dataset into memory. 

Our source code for reproducing all experiments can be found at \href{https://github.com/FutureComputing4AI/Weighted-Reservoir-Sampling-Augmented-Training/tree/main}{https://github.com/FutureComputing4AI/Weighted-Reservoir-Sampling-Augmented-Training/tree/main}. All experiments were run on CPU.

\newpage

\section{Applications to Malware Classification}
\label{ember}

Given that the original motivation for this work was a real-world need in malware detection, we present empirical results of PAC, PAC-WRS, FSOL, and FSOL-WRS on the EMBER malware classification dataset \cite{2018arXiv180404637A}. The EMBER dataset is comprised of features extracted from 1.1M benign and malicious Windows portable executable files and thus can be considered very close to a real-world test case of our WRS-Augmented Training method on a real-world deployment setting.

In Figure \ref{fig:ember}, we observe that both base PAC and FSOL experience significant and frequent test accuracy fluctuations throughout the training process. However, both PAC-WRS and FSOL-WRS very successfully mitigate these test accuracy fluctuations and even outperform the oracle model, all without the use of a separate validation set.
\begin{figure}[H]
     \centering
     \includegraphics[scale = 0.55]{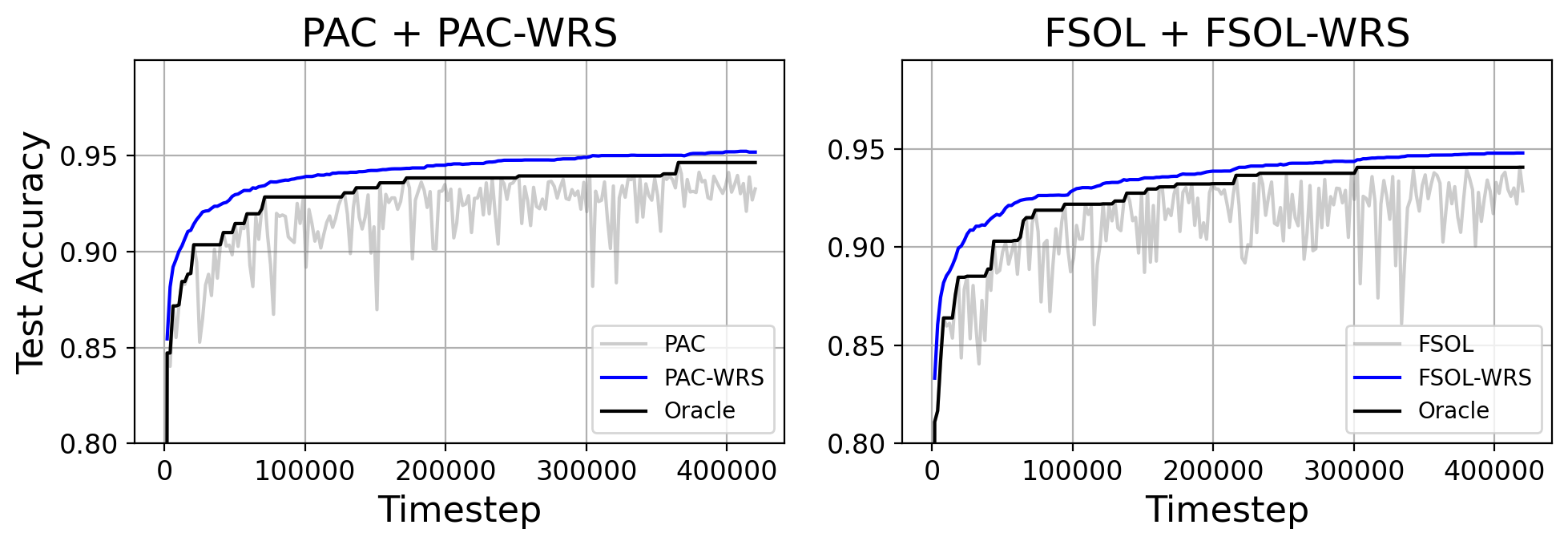}
     \caption{Test accuracies ($y$-axis) over timestep ($x$-axis) for PAC/FSOL and PAC/FSOL-WRS on the EMBER benchmark dataset for malware classification.}
     \label{fig:ember}
\end{figure}
Indeed, WRS-Augmented Training is a promising tool towards stabilizing test accuracy performance in real-world deployment settings like malware detection.

\newpage
\section{Comparisons to other ensembling and averaging schemes}
\label{other_averaging_schemes}

One natural question a reader may ask is: how does WRS-Augmented Training compare to traditional weight-averaging schemes such as moving average (e.g., averaging the most-recently-observed $K=64$ weight vectors at each timestep) or exponential average (e.g., forming an ensemble vector $\bar{\bdw}_t = \gamma \bdw_t + (1-\gamma) \bar{\bdw}_{t-1}$ at each timestep), where $\bdw_t$ is the base algorithm's candidate solution at timestep $t$.

From Figures \ref{fig:PAC_WRS+MA+EA} and \ref{fig:FSOL_WRS+MA+EA}, we observe that the exponential average scheme is consistently ineffective at mitigating the test accuracy instabilities of the base models PAC and FSOL. The stability of the exponential average scheme is usually not much better than that of the base model, which makes sense because it still puts the majority of the weighting on the most recent candidate solution. From these figures, we also observe that the $K=64$ moving average also has very mixed effectiveness. Overall, we see that WRS-Augmented Training is more preferable to the $K=64$ moving average. This makes sense because with WRS-Augmented Training, we are much more selective about the quality of the candidate solutions that we retain in our reservoir, compared to moving average, which necessarily by definition must include poor-performing solutions as they appear.

Third, from Table \ref{table:compute_times}\footnote{For Table \ref{table:compute_times}, we are still waiting on a full set of results over multiple seeds for Criteo and Avazu (Site) for moving average, but because we are studying online algorithms, we can still share full-epoch results for one seed, and partial-epoch results for the other seeds as they exist today. We will update our paper accordingly when the results are finished. We also emphasize that the moving average method, on larger datasets, is multiple times slower per iteration, on average, than our WRS-Augmented Training, contributing to this delay. The general trends are already very clear from the results presented below.}, we see that for datasets with dimension $D > 100K$, the moving average method can be significantly computationally slower per iteration than WRS-Augmented Training. For example, with PAC as the base model, the moving average method was, on average, $6.579$x slower per iteration than WRS-Augmented Training on KDD2010 (Algebra) and $10.050$x slower per iteration than WRS-Augmented Training on URL. Similar trends hold when using FSOL as the base algorithm. These results make sense because with WRS-Augmented Training, we do not always add candidate solutions to our reservoir for averaging, while with the moving average method, we must always add new candidate solutions into our set. These insertion and deletion costs will accrue over time. On smaller datasets, the moving average is faster or slower depending on the dataset, but runtime is dominated by IO and all methods finish within minutes. However, in more real-world, large-scale settings where evaluation and checking are prohibitively expensive, WRS-Augmented Training is the fastest, most accurate, and most reliable method compared to all the baselines.

\begin{figure}[H]
     \centering
     \includegraphics[scale = 0.35]{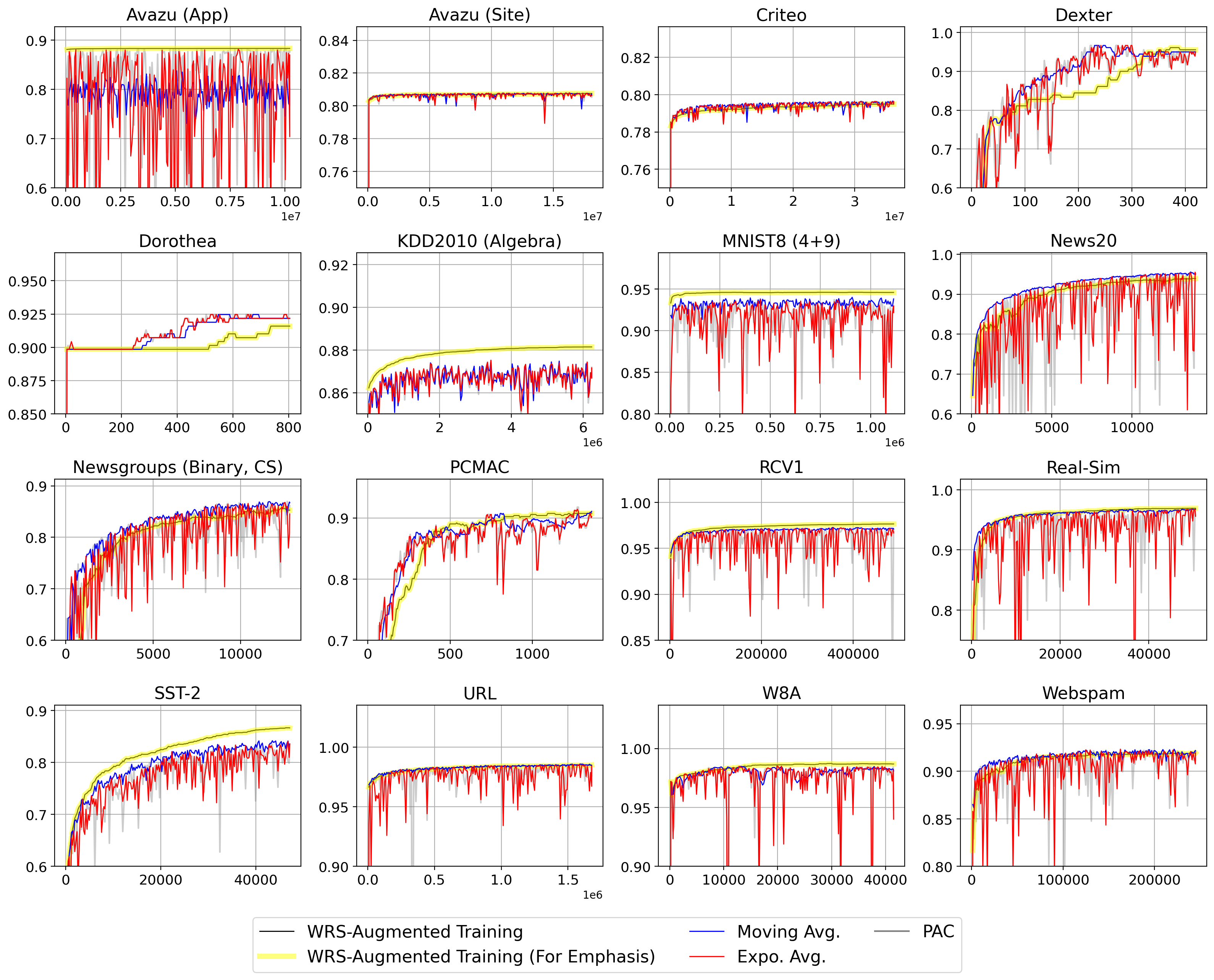}
     \caption{Test accuracies ($y$-axis) over timestep ($x$-axis) for WRS-Augmented Training $(K=64)$, moving average (most recent $K=64$) and exponential average $(\gamma = 0.9)$, using PAC as the base.}
     \label{fig:PAC_WRS+MA+EA}
\end{figure}

\begin{figure}[H]
     \centering
     \includegraphics[scale = 0.35]{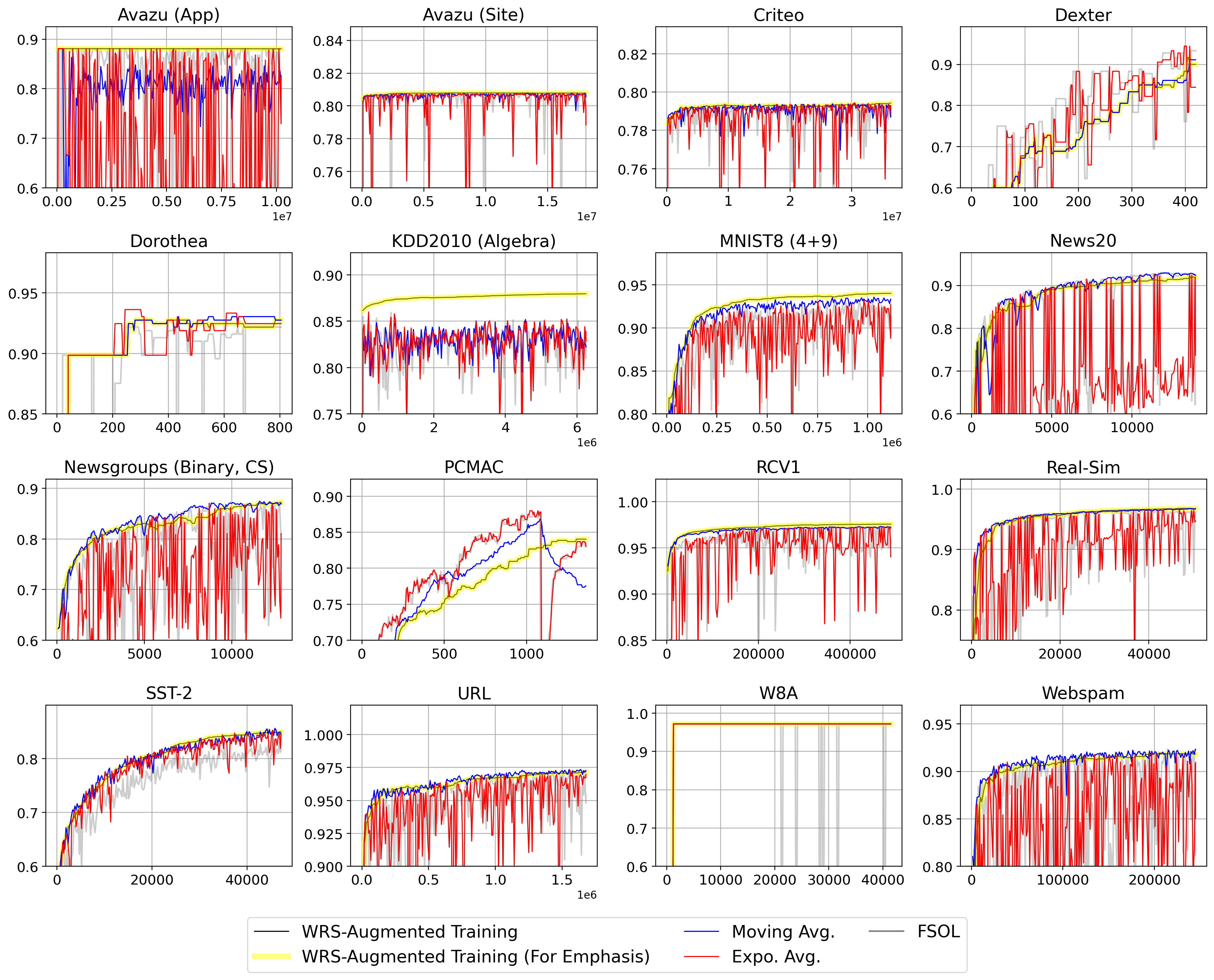}
     \caption{Test accuracies ($y$-axis) over timestep ($x$-axis) for WRS-Augmented Training $(K=64)$, moving average (most recent $K=64$) and exponential average $(\gamma = 0.9)$, using FSOL as the base.}
     \label{fig:FSOL_WRS+MA+EA}
\end{figure}

\begin{table}[H]
    \footnotesize
    \caption{Compute times per iteration (in seconds) of moving average and exponential average \textbf{relative to WRS-augmented training}, with PAC (left) and FSOL (right) as base models. For example, on URL with PAC, \textit{moving average was, on average, $10.050$x slower per iteration than WRS-augmented training.}}
    \centering
    \begin{adjustbox}{width=1.0\textwidth}
    \begin{tabular}{lr|rr|rr}
    & & \multicolumn{2}{c|}{\textbf{PAC}} & \multicolumn{2}{c}{\textbf{FSOL}} \\
    \midrule
    \textbf{Dataset} & \textbf{D} & \textbf{Moving Avg.} & \textbf{Expo. Avg.} & \textbf{Moving Avg.} & \textbf{Expo. Avg.} \\
    \midrule
    Webspam & 254 & 0.953 & 0.717 & 0.738 & 0.613 \\
    W8A & 300 & 0.420 & 0.353 & 0.368 & 0.336 \\
    MNIST8 (4+9) & 784 & 1.125 & 0.885 & 0.968 & 0.887 \\
    PCMAC & 3289 & 0.483 & 0.312 & 0.377 & 0.235 \\
    SST-2 & 13757 & 0.903 & 0.431 & 0.748 & 0.411 \\
    Dexter & 20000 & 0.577 & 0.143 & 0.285 & 0.066 \\
    Real-Sim & 20958 & 0.949 & 0.624 & 0.707 & 0.564 \\
    RCV1 & 47236 & 1.825 & 0.892 & 1.366 & 1.011 \\
    Dorothea & 100000 & 1.194 & 0.190000 & 0.580 & 0.254 \\
    Newsgroups (Binary, CS) & 101631 & 3.587 & 1.042 & 1.376 & 0.543 \\
    Avazu (App) & 1000000 & 8.357 & 2.507 & 8.078 & 2.544 \\
    Avazu (Site) & 1000000 & 18.704 & 3.207 & 11.612 & 3.683 \\
    Criteo & 1000000 & 14.267 & 2.876 & 7.009 & 2.381 \\
    News20 & 1355191 & 4.288 & 1.356 & 3.376 & 1.153 \\
    URL & 3231961 & 10.050 & 3.945 & 9.124 & 4.058 \\
    KDD2010 (Algebra) & 20216830 & 6.579 & 2.402 & 6.778 & 2.301 \\
    \bottomrule
    \end{tabular}
    \end{adjustbox}
    \label{table:compute_times}
\end{table}

\newpage
\section{Modified WRS-Augmented Training on non-passive aggressive online learning methods}
\label{non-passive-aggressive}

In this section, we explore applying a modified form of WRS-Augmented Training (WAT) on top of three non-passive-aggressive online learning methods: Stochastic Gradient Descent with Momentum \cite{sutskever2013importance}, ADAGRAD \cite{duchi2011adaptive}, and Truncated Gradient Descent \cite{langford2009sparse}, all using default hyperparameters. We emphasize that these three methods are not passive-aggressive as they always update their weight vectors at each time step, even if the data point was correctly classified. As such, the original WAT is not directly applicable. 

However, we can modify WAT as follows. First, we define a \textit{pseudo-passive} step as one where the current solution candidate made no classification error (i.e., no more concept of margin). Second, at a time step when the current solution candidate does make an error, we will sample the last weight vector before a mistake was made into our reservoir with probability proportional to the number of pseudo-passive steps, before resetting our counter. The rest of WAT operates as normal under this pseudo-passive step weighting.

From Figures \ref{fig:WRS-SGDM} - \ref{fig:WRS-TGD}, we observe that all three non-passive-aggressive online algorithms, in particular Stochastic Gradient Descent with Momentum (SGD+M) and Truncated Gradient Descent (TGD), are susceptible to experiencing concerning fluctuations in test accuracy. Consistently, across the algorithms, we observe that when such fluctuations are present, modified WAT effectively mitigates such fluctuations very well. On the other hand, when there is little fluctuation in the base model (e.g., see ADAGRAD), applying modified WAT will do little to no harm. We emphasize that WAT was not designed for non-passive-aggressive algorithms, but is still demonstrably useful and adaptable to a wider class of base models, which can constitute fruitful future work.

\begin{figure}[H]
     \centering
     \includegraphics[scale = 0.35]{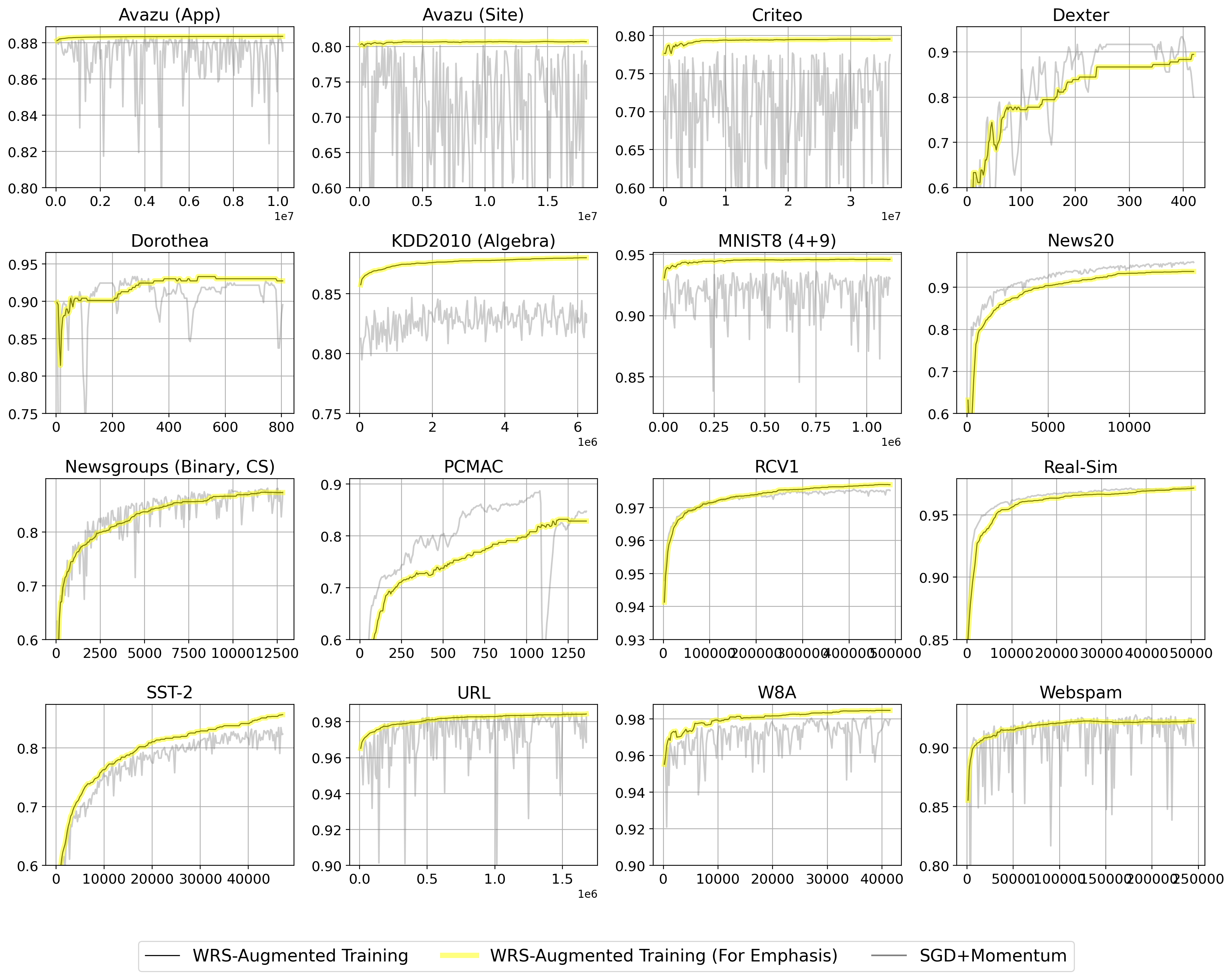}
     \caption{Test accuracies ($y$-axis) over timestep ($x$-axis) for modified WRS-Augmented Training $(K=64)$ on Stochastic Gradient Descent with Momentum.}
     \label{fig:WRS-SGDM}
\end{figure}

\begin{figure}[H]
     \centering
     \includegraphics[scale = 0.35]{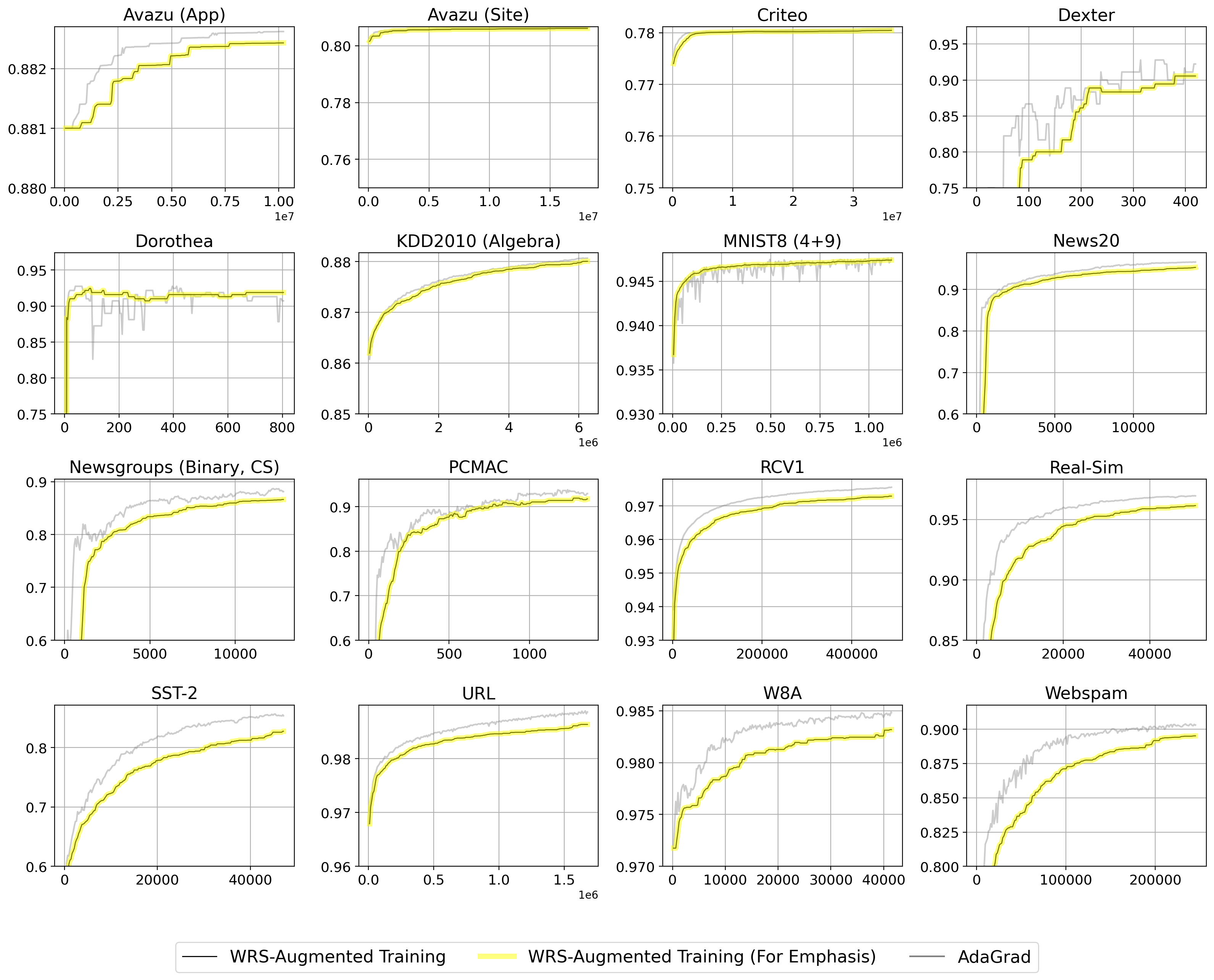}
     \caption{Test accuracies ($y$-axis) over timestep ($x$-axis) for modified WRS-Augmented Training $(K=64)$ on ADAGRAD.}
     \label{fig:WRS-ADA}
\end{figure}

\begin{figure}[H]
     \centering
     \includegraphics[scale = 0.35]{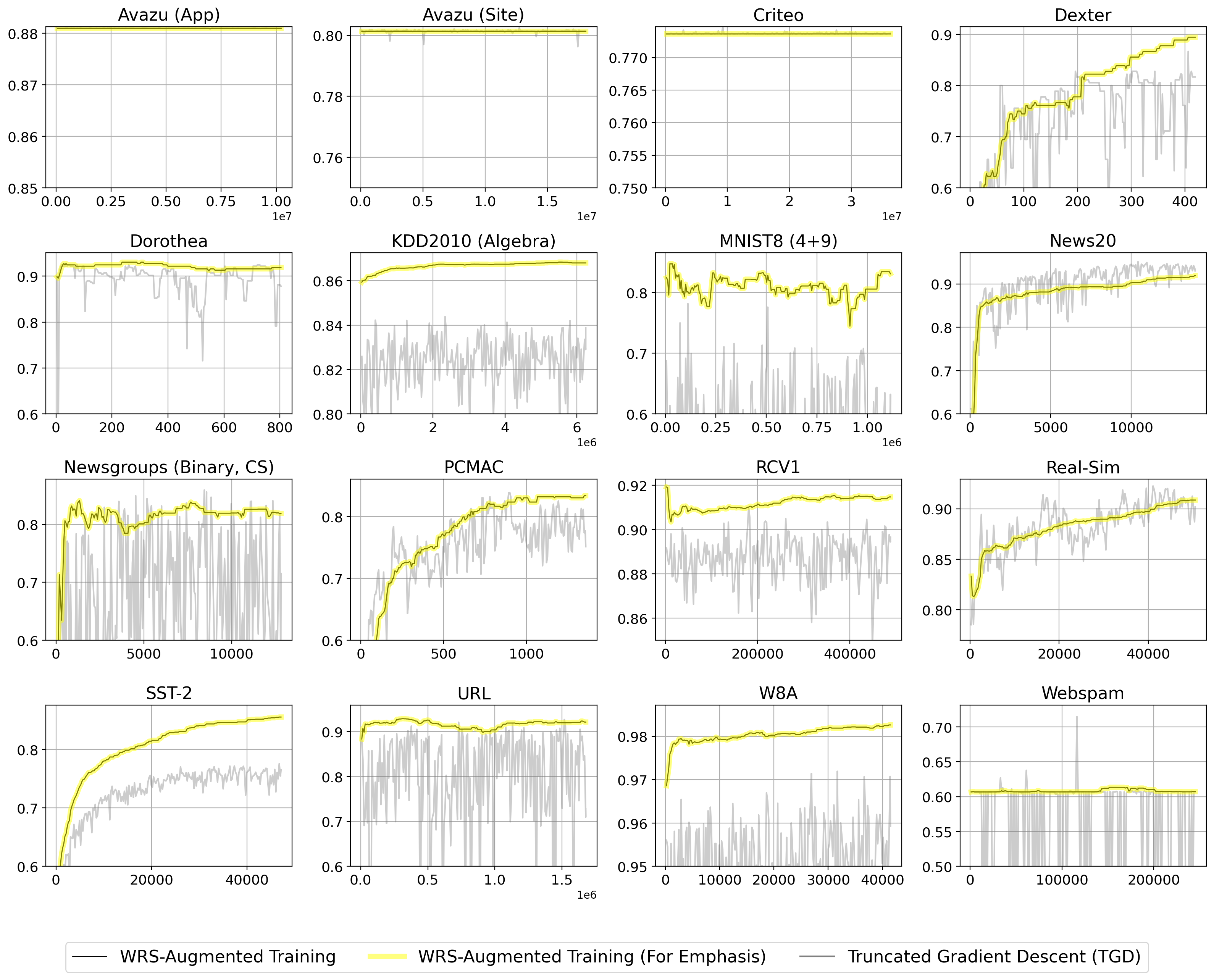}
     \caption{Test accuracies ($y$-axis) over timestep ($x$-axis) for modified WRS-Augmented Training $(K=64)$ on Truncated Gradient Descent.}
     \label{fig:WRS-TGD}
\end{figure}

\newpage

\section{Tables of $p$-values for Wilcoxon Signed-Rank Tests}
\label{wilcoxon-appendix}

Below, we provide full tables of $p$-values computed using Wilcoxon Signed-Rank Tests towards probing the statistical significance of the differences in ROP and final test accuracy between PAC/FSOL-WRS and their base model counterparts.

\begin{table}[H]
\caption{ Wilcoxon signed-rank test $p$-values testing whether differences in \textbf{relative oracle performance} between $K=64$ PAC/FSOL-WRS variants and base PAC/FSOL methods are statistically significant.
\normalsize}
\label{table:wilcoxon-ROP-WRS}
\centering
\begin{adjustbox}{width=0.55\textwidth}
\begin{tabular}{c|cc|cc}
  & \multicolumn{2}{c|}{\textbf{Simple Average}} & \multicolumn{2}{c}{\textbf{Weighted Average}} \\
  \midrule
  & Standard & Exponential & Standard & Exponential \\
  \midrule
  \textbf{PAC} & 0.0150 & 0.0199 & 0.0386 & 0.325 \\
  \textbf{FSOL} & 0.000437 & 0.0879 & 0.000437 & 0.0787 \\
  \hline
\end{tabular}
\end{adjustbox}
\end{table}

\begin{table}[H]
\caption{ Wilcoxon signed-rank test $p$-values testing whether differences in \textbf{final test accuracy} between $K=64$ PAC/FSOL-WRS variants and base PAC/FSOL methods are statistically significant. \normalsize}
\label{table:final-test-acc-WRS-wilcoxon}
\centering
\begin{adjustbox}{width=0.55\textwidth}
\begin{tabular}{c|cc|cc}
  & \multicolumn{2}{c|}{\textbf{Simple Average}} & \multicolumn{2}{c}{\textbf{Weighted Average}} \\
  \midrule
  & Standard & Exponential & Standard & Exponential \\
  \midrule
  \textbf{PAC} & 0.408 & 0.196 & 0.501 & 0.856 \\
  \textbf{FSOL} & 0.00836 & 0.313 & 0.00836 & 0.679 \\
  \hline
\end{tabular}
\end{adjustbox}
\end{table}

\newpage
\section{Additional results figures for PAC-WRS and FSOL-WRS}

\subsection{Uncertainty-quantified aggregate metrics across runs}
\label{errorbars-appendix}

Below, we provide figures with error bars (minimum and maximum across 5 trials) showing the relative oracle performances (ROP), final test accuracies, and final sparsities of PAC-WRS and FSOL-WRS alongside the base PAC and FSOL models, across all 16 datasets. The main takeaways are that a) PAC-WRS and FSOL-WRS overall incur substantially-lower ROP than their corresponding base models; b) PAC-WRS and FSOL-WRS overall achieve comparable, if not improved, final test accuracy compared to their corresponding base models, consistently with lower variance, too; c) the final sparsities of PAC-WRS and FSOL-WRS are overall comparable, if not higher than those of their corresponding base models.

\subsubsection{FSOL and FSOL-WRS}
\label{errorbars-appendix-fsol}

\begin{figure}[H]
     \centering
     \includegraphics[scale = 0.39]{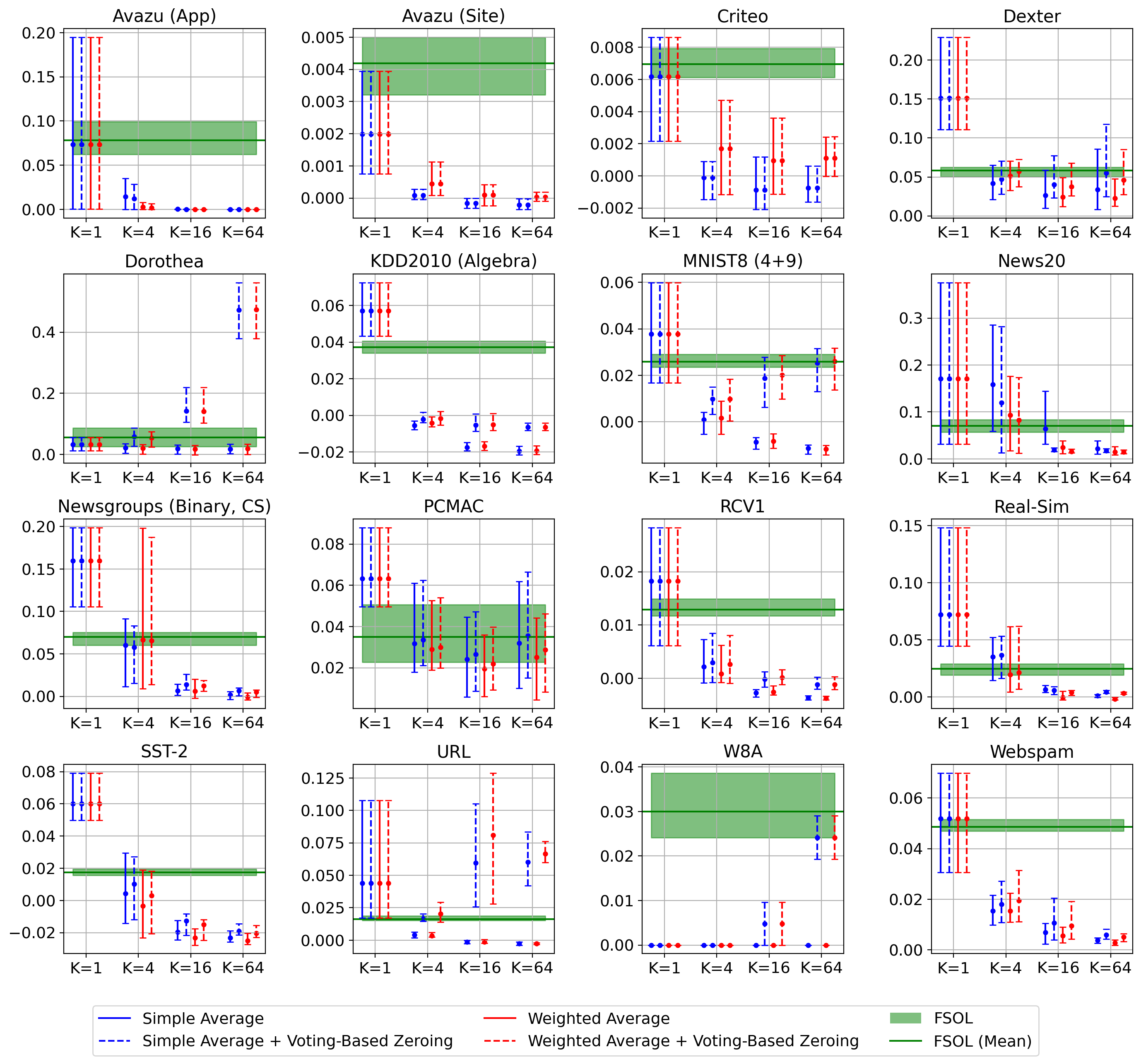}
     \caption{ Relative oracle performances  ($y$-axis) of base FSOL and FSOL-WRS using \textbf{standard} weights over reservoir sizes $K$ ($x$-axis) on all datasets. Error bars represent the minimum and maximum values achieved across 5 randomized trials. \textbf{Blue:} WRS-augmented variants via simple average of reservoir members. \textbf{Red:} WRS-augmented variants via weighted average of reservoir members. \textbf{Dotted lines:} indicates voting-based zeroing was performed for additional sparsity. \textbf{Lower values indicate more stable performance.} \normalsize}
     \label{fig:FSOL_ROP_errbar-dense}
\end{figure}

\begin{figure}[H]
     \centering
     \includegraphics[scale = 0.39]{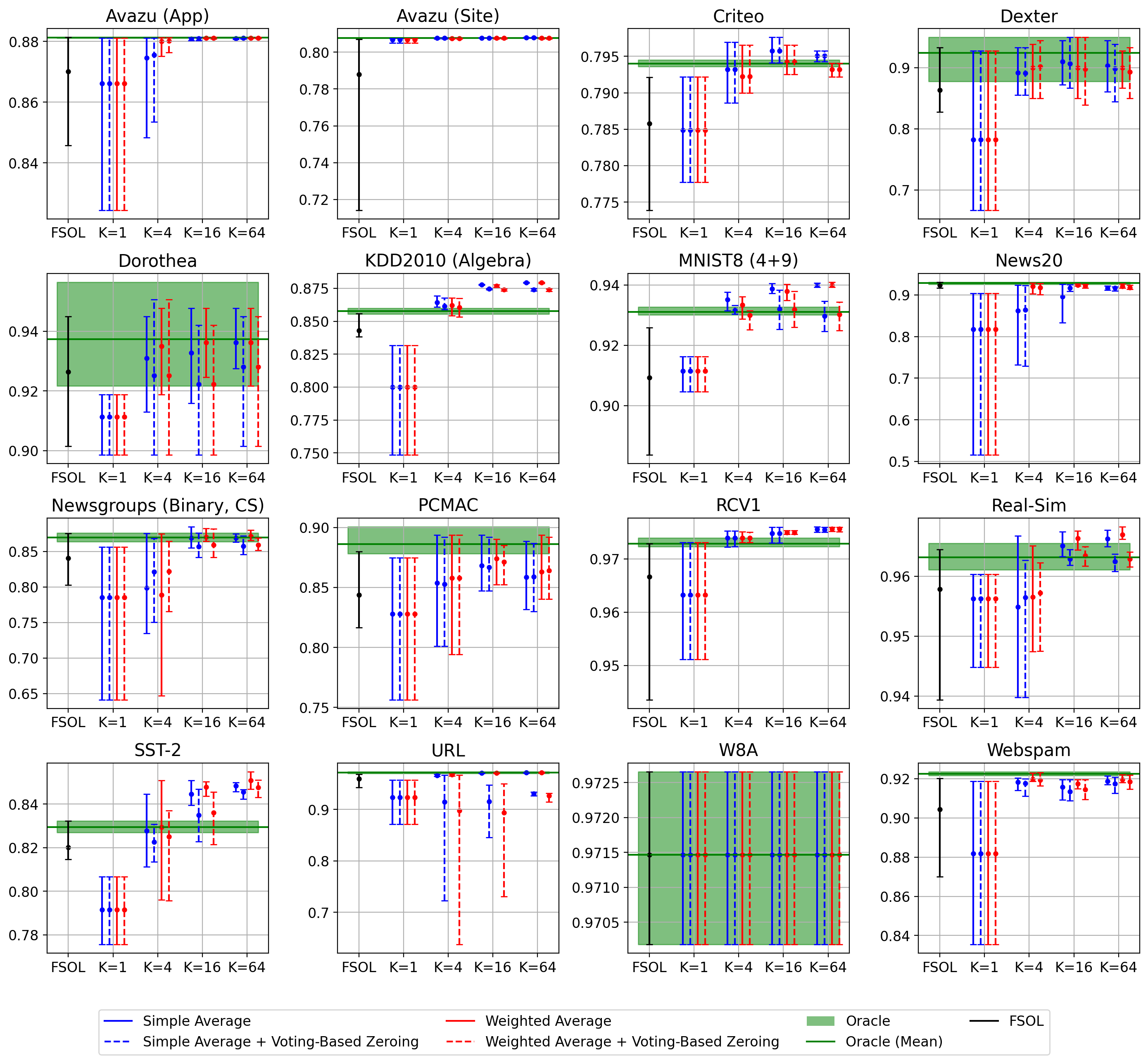}
     \caption{ Final test accuracies  ($y$-axis) of base FSOL and FSOL-WRS using \textbf{standard} weights over reservoir sizes $K$ ($x$-axis) on all datasets. Error bars represent the minimum and maximum values achieved across 5 randomized trials. \textbf{Blue:} WRS-augmented variants via simple average of reservoir members. \textbf{Red:} WRS-augmented variants via weighted average of reservoir members. \textbf{Dotted lines:} indicates voting-based zeroing was performed for additional sparsity.\normalsize}
     \label{fig:FSOL_final-test-acc_errbar-dense}
\end{figure}

\begin{figure}[H]
     \centering
     \includegraphics[scale = 0.39]{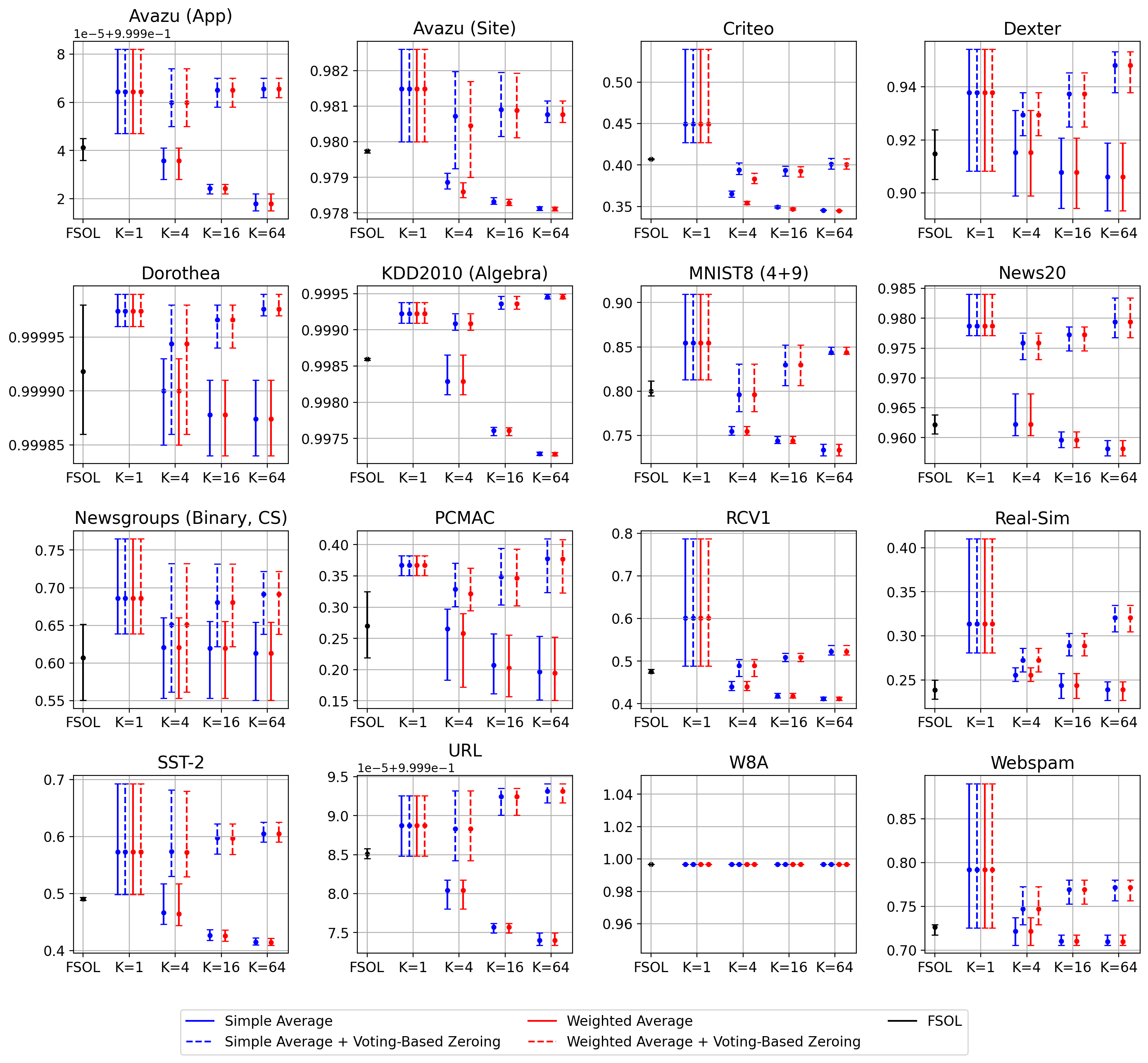}
     \caption{ Final sparsities  ($y$-axis) of base FSOL and FSOL-WRS using \textbf{standard} weights over reservoir sizes $K$ ($x$-axis) on all datasets. Error bars represent the minimum and maximum values achieved across 5 randomized trials. \textbf{Blue:} WRS-augmented variants via simple average of reservoir members. \textbf{Red:} WRS-augmented variants via weighted average of reservoir members. \textbf{Dotted lines:} indicates voting-based zeroing was performed for additional sparsity.\normalsize}
     \label{fig:FSOL_final-sparsities_errbar-dense}
\end{figure}

\begin{figure}[H]
     \centering
     \includegraphics[scale = 0.39]{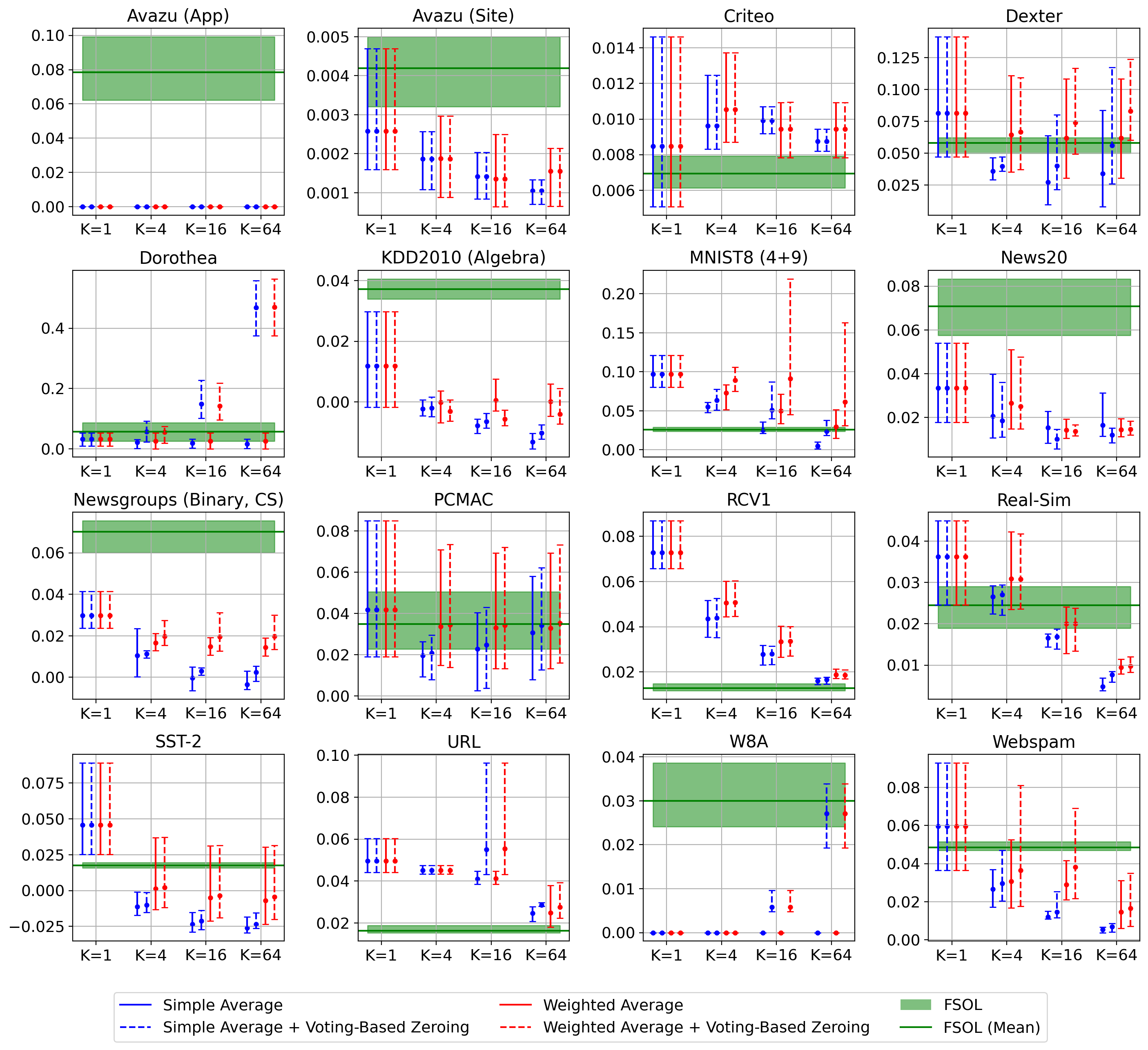}
     \caption{ Relative oracle performances  ($y$-axis) of base FSOL and FSOL-WRS using \textbf{exponential} weights over reservoir sizes $K$ ($x$-axis) on all datasets. Error bars represent the minimum and maximum values achieved across 5 randomized trials. \textbf{Blue:} WRS-augmented variants via simple average of reservoir members. \textbf{Red:} WRS-augmented variants via weighted average of reservoir members. \textbf{Dotted lines:} indicates voting-based zeroing was performed for additional sparsity. \textbf{Lower values indicate more stable performance.} \normalsize}
     \label{fig:FSOL_ROP_errbar-exp-dense}
\end{figure}

\begin{figure}[H]
     \centering
     \includegraphics[scale = 0.39]{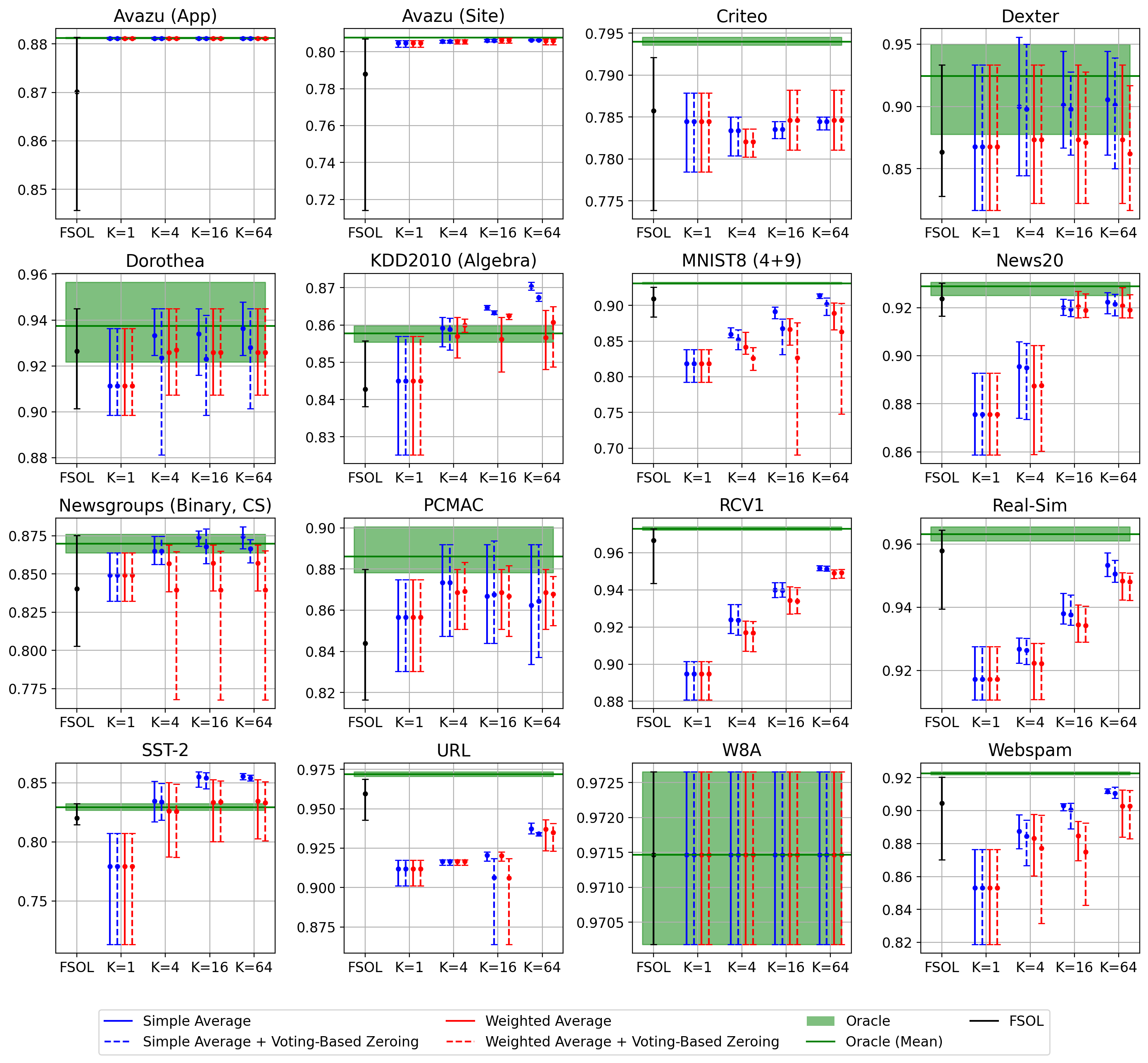}
     \caption{ Final test accuracies  ($y$-axis) of base FSOL and FSOL-WRS using \textbf{exponential} weights over reservoir sizes $K$ ($x$-axis) on all datasets. Error bars represent the minimum and maximum values achieved across 5 randomized trials. \textbf{Blue:} WRS-augmented variants via simple average of reservoir members. \textbf{Red:} WRS-augmented variants via weighted average of reservoir members. \textbf{Dotted lines:} indicates voting-based zeroing was performed for additional sparsity.\normalsize}
     \label{fig:FSOL_final-test-acc_errbar-exp-dense}
\end{figure}

\begin{figure}[H]
     \centering
     \includegraphics[scale = 0.39]{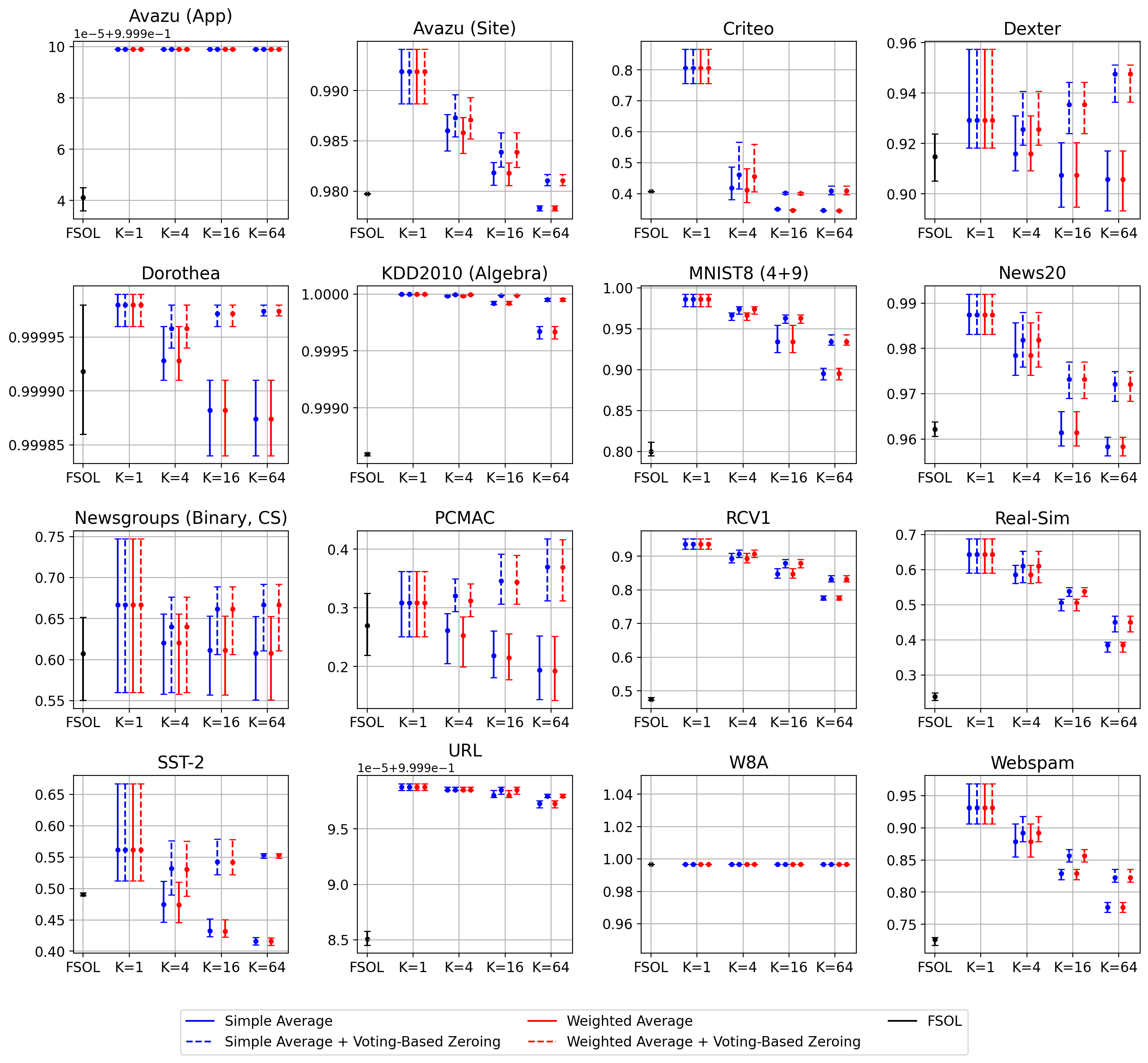}
     \caption{ Final sparsities  ($y$-axis) of base FSOL and FSOL-WRS using \textbf{exponential} weights over reservoir sizes $K$ ($x$-axis) on all datasets. Error bars represent the minimum and maximum values achieved across 5 randomized trials. \textbf{Blue:} WRS-augmented variants via simple average of reservoir members. \textbf{Red:} WRS-augmented variants via weighted average of reservoir members. \textbf{Dotted lines:} indicates voting-based zeroing was performed for additional sparsity.\normalsize}
     \label{fig:FSOL_final-sparsities_errbar-exp-dense}
\end{figure}

\subsubsection{PAC and PAC-WRS}
\label{errorbars-appendix-pac}

\begin{figure}[H]
     \centering
     \includegraphics[scale = 0.39]{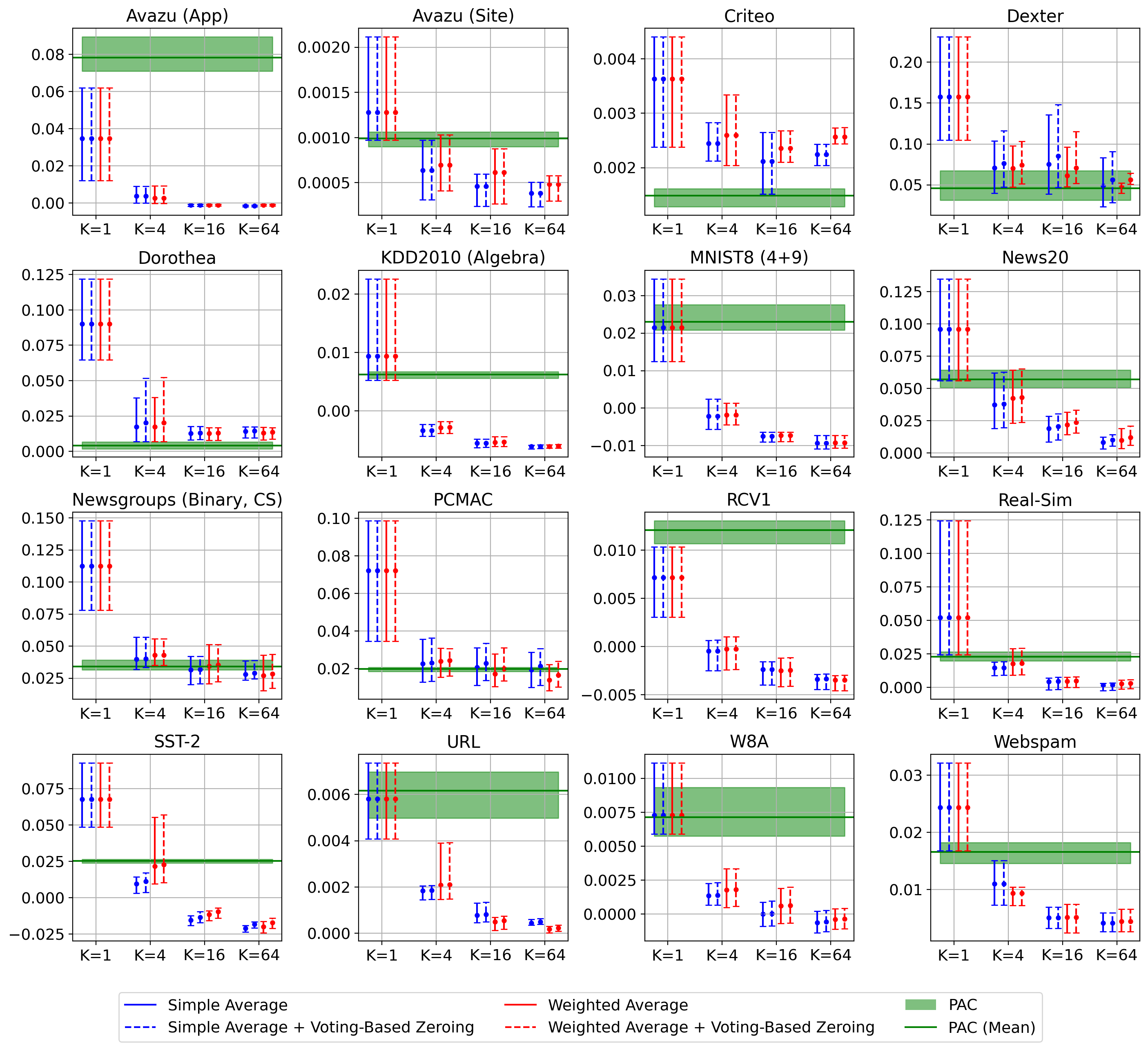}
     \caption{ Relative oracle performances  ($y$-axis) of base PAC and PAC-WRS using \textbf{standard} weights over reservoir sizes $K$ ($x$-axis) on all datasets. Error bars represent the minimum and maximum values achieved across 5 randomized trials. \textbf{Blue:} WRS-augmented variants via simple average of reservoir members. \textbf{Red:} WRS-augmented variants via weighted average of reservoir members. \textbf{Dotted lines:} indicates voting-based zeroing was performed for additional sparsity. \textbf{Lower values indicate more stable performance.} \normalsize}
     \label{fig:PAC_ROP_errbar-dense}
\end{figure}

\begin{figure}[H]
     \centering
     \includegraphics[scale = 0.39]{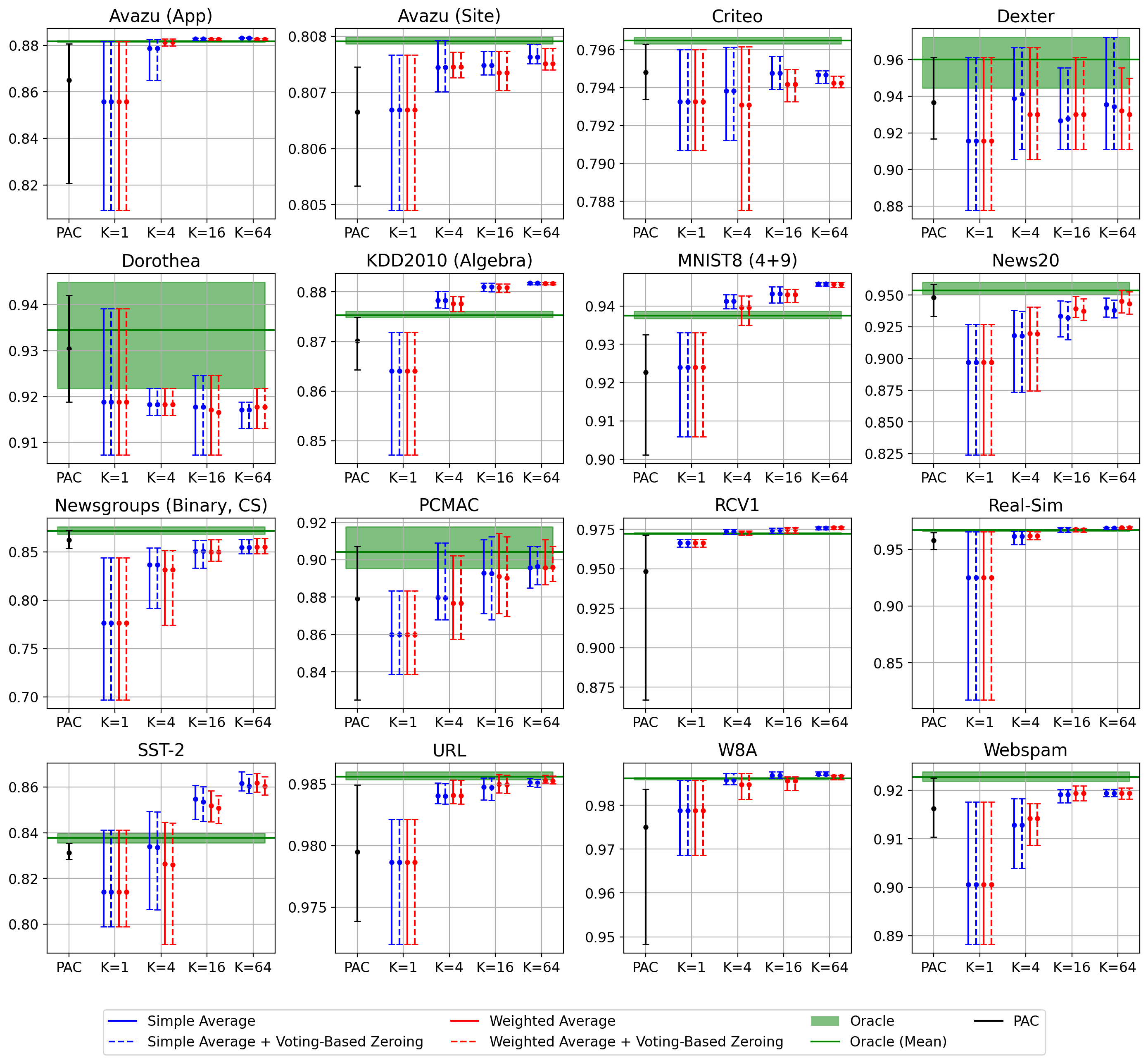}
     \caption{ Final test accuracies  ($y$-axis) of base PAC and PAC-WRS using \textbf{standard} weights over reservoir sizes $K$ ($x$-axis) on all datasets. Error bars represent the minimum and maximum values achieved across 5 randomized trials. \textbf{Blue:} WRS-augmented variants via simple average of reservoir members. \textbf{Red:} WRS-augmented variants via weighted average of reservoir members. \textbf{Dotted lines:} indicates voting-based zeroing was performed for additional sparsity.\normalsize}
     \label{fig:PAC_final-test-acc_errbar-dense}
\end{figure}

\begin{figure}[H]
     \centering
     \includegraphics[scale = 0.39]{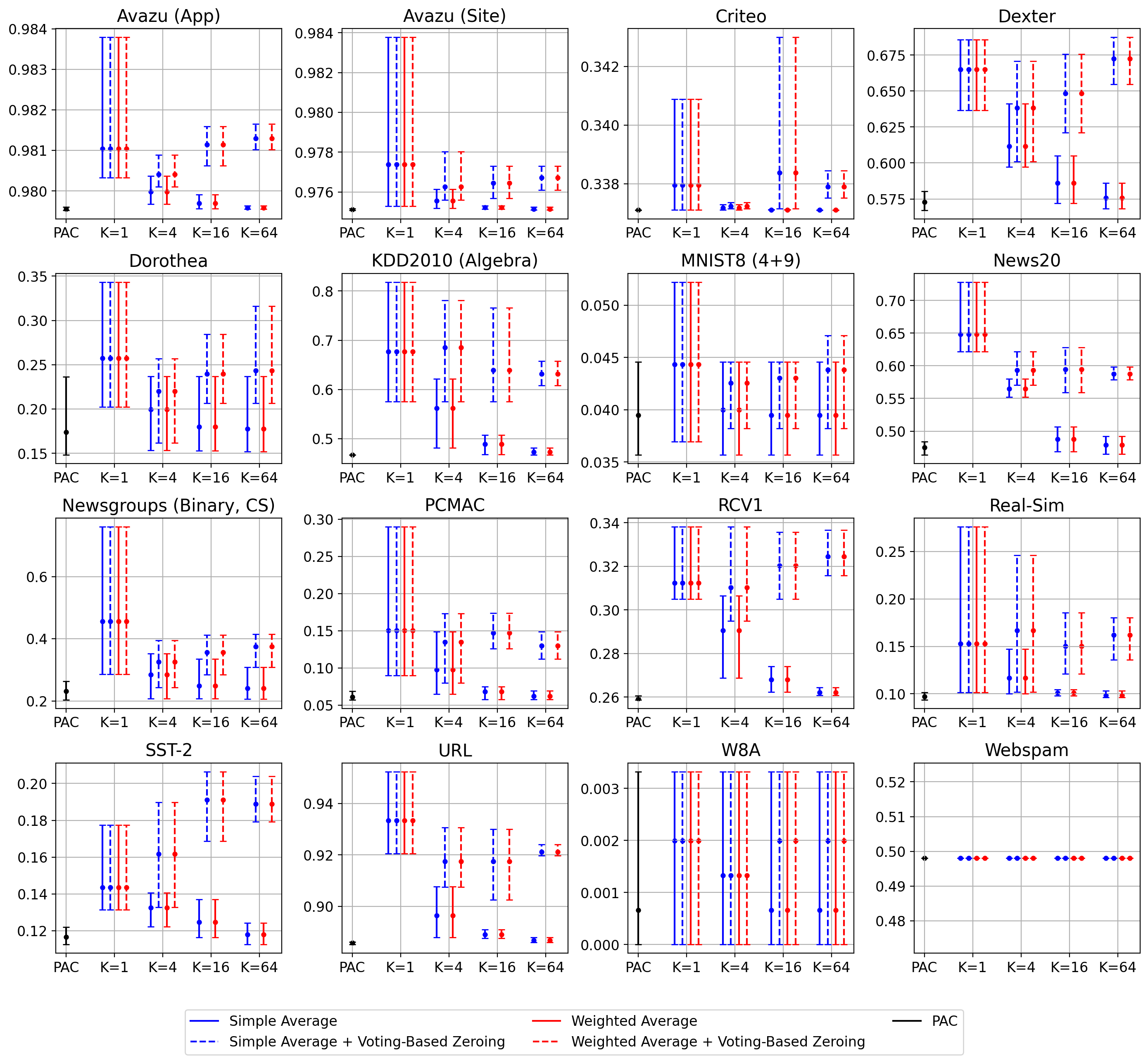}
     \caption{ Final sparsities  ($y$-axis) of base PAC and PAC-WRS using \textbf{standard} weights over reservoir sizes $K$ ($x$-axis) on all datasets. Error bars represent the minimum and maximum values achieved across 5 randomized trials. \textbf{Blue:} WRS-augmented variants via simple average of reservoir members. \textbf{Red:} WRS-augmented variants via weighted average of reservoir members. \textbf{Dotted lines:} indicates voting-based zeroing was performed for additional sparsity.\normalsize}
     \label{fig:PAC_final-sparsities_errbar-dense}
\end{figure}

\begin{figure}[H]
     \centering
     \includegraphics[scale = 0.39]{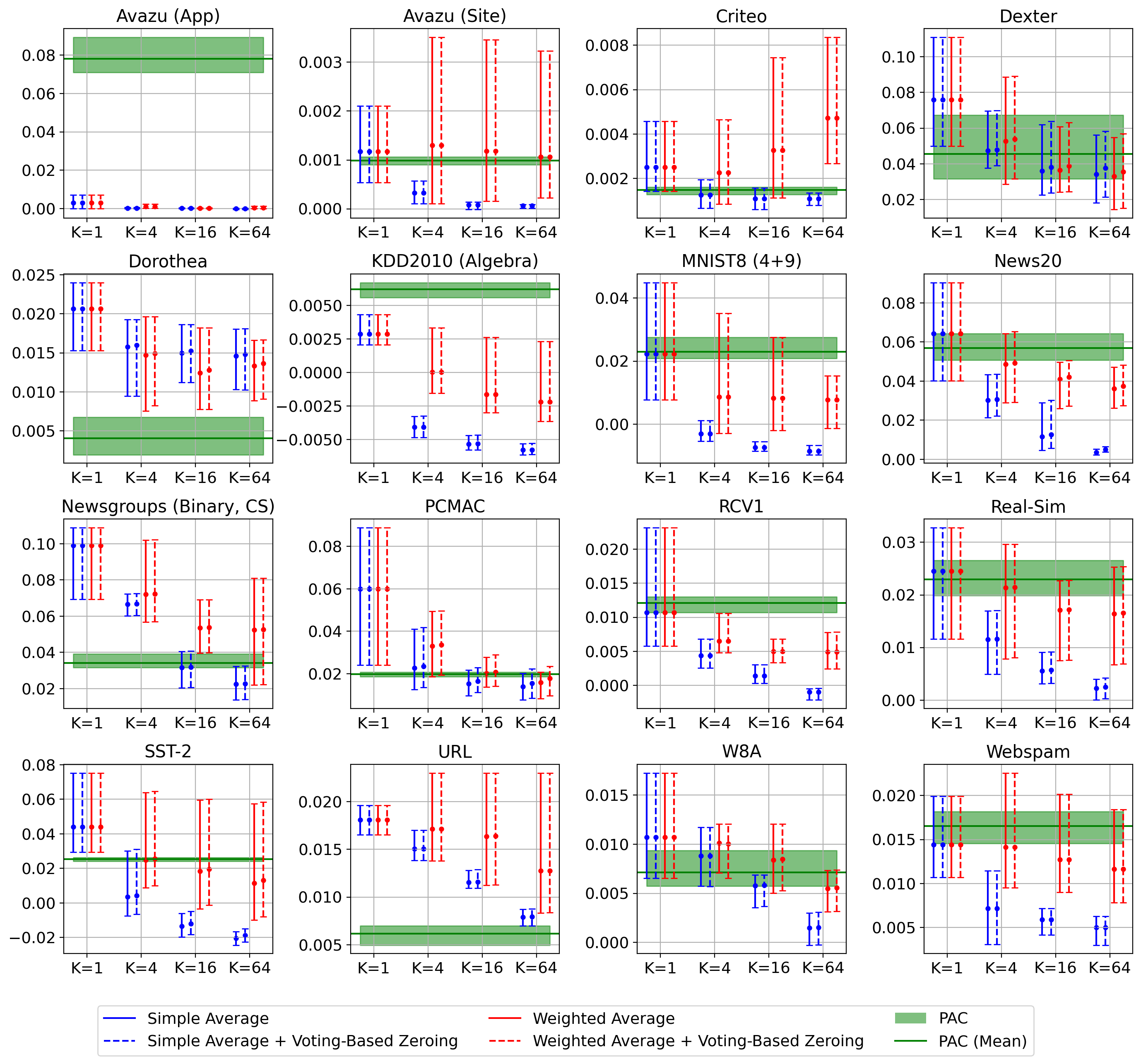}
     \caption{ Relative oracle performances  ($y$-axis) of base PAC and PAC-WRS using \textbf{exponential} weights over reservoir sizes $K$ ($x$-axis) on all datasets. Error bars represent the minimum and maximum values achieved across 5 randomized trials. \textbf{Blue:} WRS-augmented variants via simple average of reservoir members. \textbf{Red:} WRS-augmented variants via weighted average of reservoir members. \textbf{Dotted lines:} indicates voting-based zeroing was performed for additional sparsity. \textbf{Lower values indicate more stable performance.} \normalsize}
     \label{fig:PAC_ROP_errbar-exp-dense}
\end{figure}

\begin{figure}[H]
     \centering
     \includegraphics[scale = 0.39]{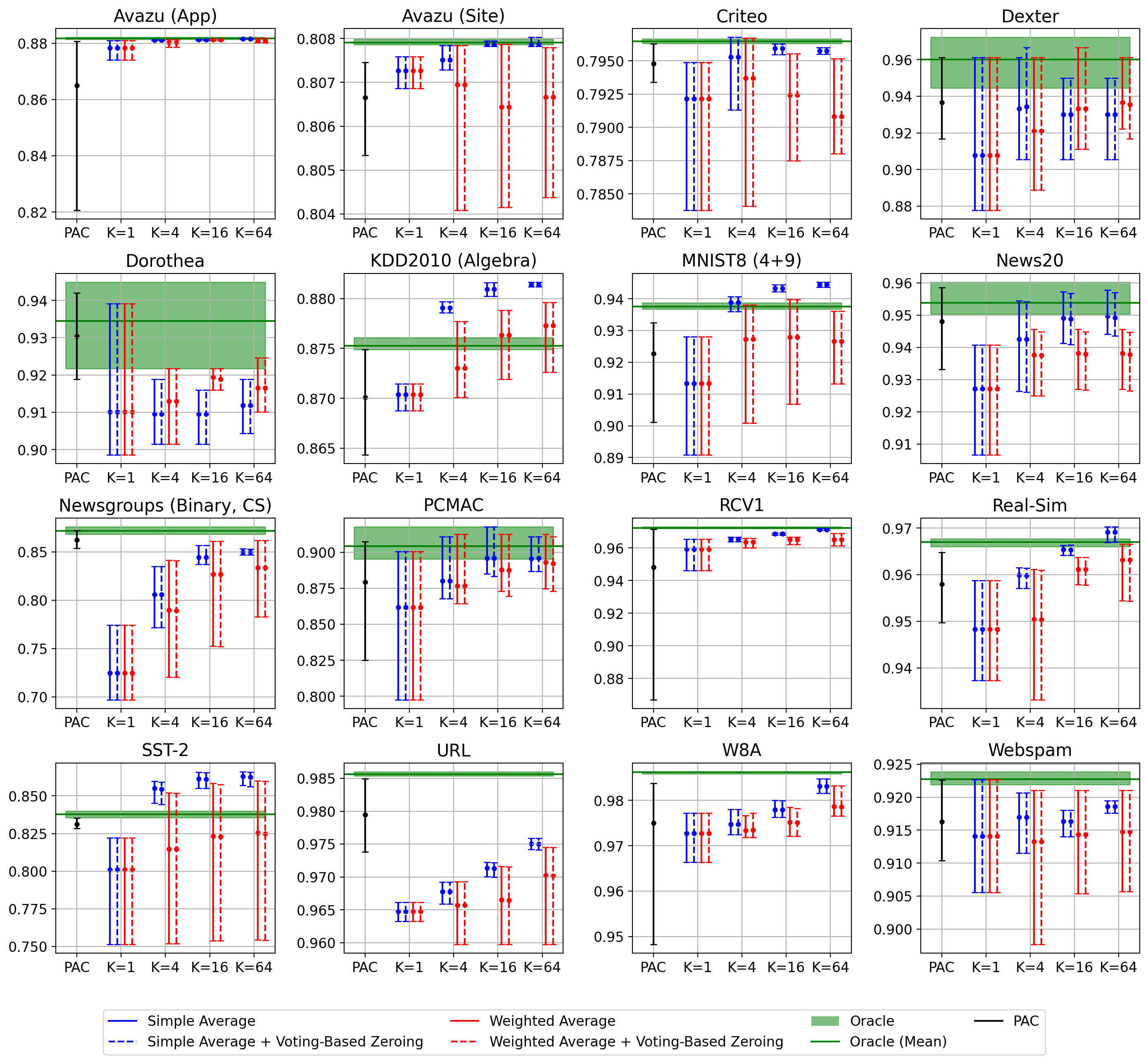}
     \caption{ Final test accuracies  ($y$-axis) of base PAC and PAC-WRS using \textbf{exponential} weights over reservoir sizes $K$ ($x$-axis) on all datasets. Error bars represent the minimum and maximum values achieved across 5 randomized trials. \textbf{Blue:} WRS-augmented variants via simple average of reservoir members. \textbf{Red:} WRS-augmented variants via weighted average of reservoir members. \textbf{Dotted lines:} indicates voting-based zeroing was performed for additional sparsity.\normalsize}
     \label{fig:PAC_final-test-acc_errbar-exp-dense}
\end{figure}

\begin{figure}[H]
     \centering
     \includegraphics[scale = 0.39]{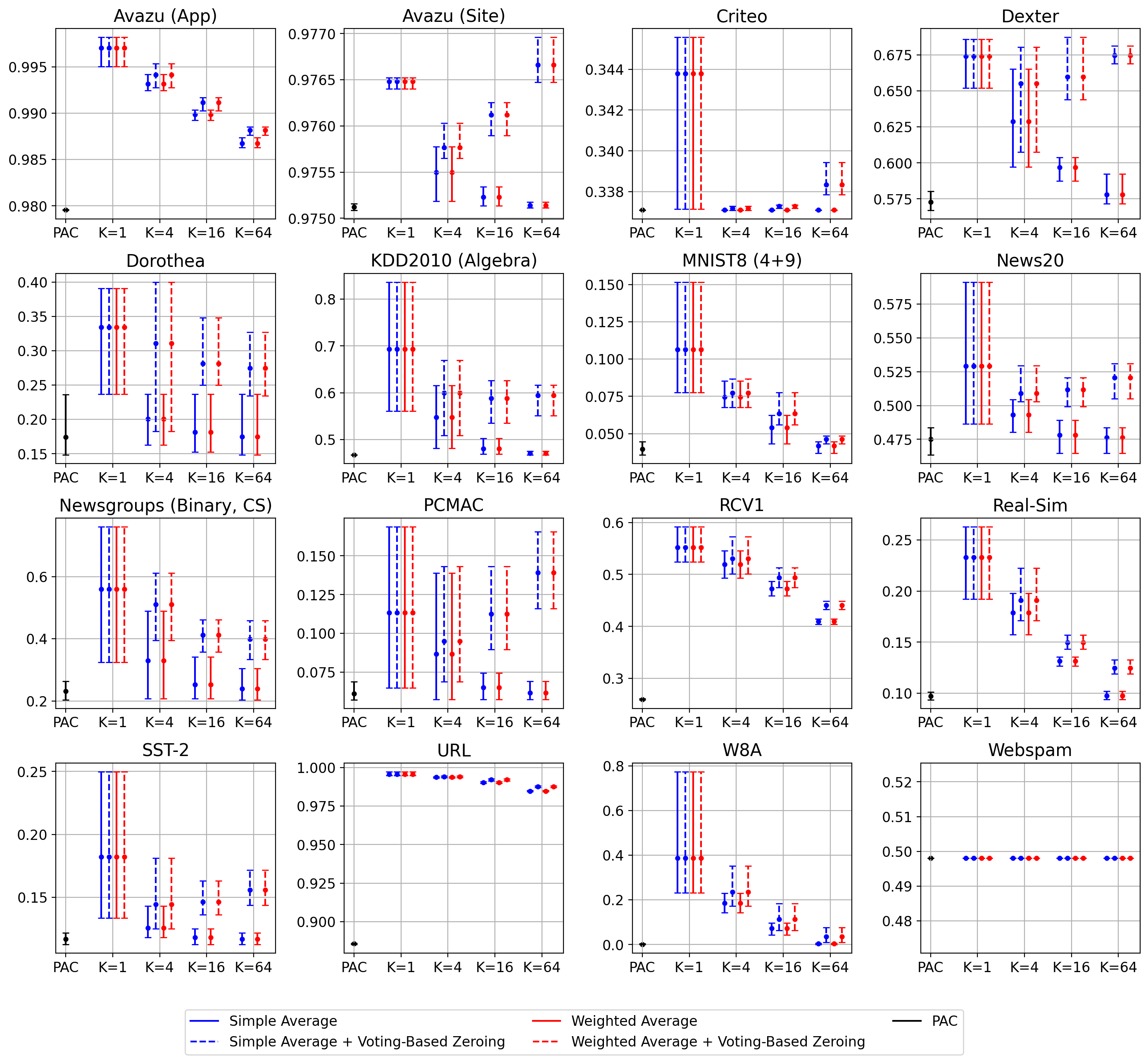}
     \caption{ Final sparsities  ($y$-axis) of base PAC and PAC-WRS using \textbf{exponential} weights over reservoir sizes $K$ ($x$-axis) on all datasets. Error bars represent the minimum and maximum values achieved across 5 randomized trials. \textbf{Blue:} WRS-augmented variants via simple average of reservoir members. \textbf{Red:} WRS-augmented variants via weighted average of reservoir members. \textbf{Dotted lines:} indicates voting-based zeroing was performed for additional sparsity.\normalsize}
     \label{fig:PAC_final-sparsities_errbar-exp-dense}
\end{figure}

\newpage

\subsection{Test accuracies and sparsities over time}
\label{over-time-appendix}

Below, we provide figures showing test accuracies and sparsities over time on individual runs of PAC-WRS and FSOL-WRS, compared against their base model counterparts, on all 16 datasets. The main takeaways are that a) test accuracies over time on PAC-WRS and FSOL-WRS are overall much stabler, if not also higher, than their base model counterparts (and much closer to that of the oracle); b) applying WAT to PAC and FSOL does not significantly impact sparsity --- in some cases, especially with voting-based zeroing, sparsity actually increases compared to the base model!

\subsubsection{FSOL and FSOL-WRS}
\label{over-time-appendix-fsol}

\begin{figure}[H]
     \centering
     \includegraphics[scale = 0.39]{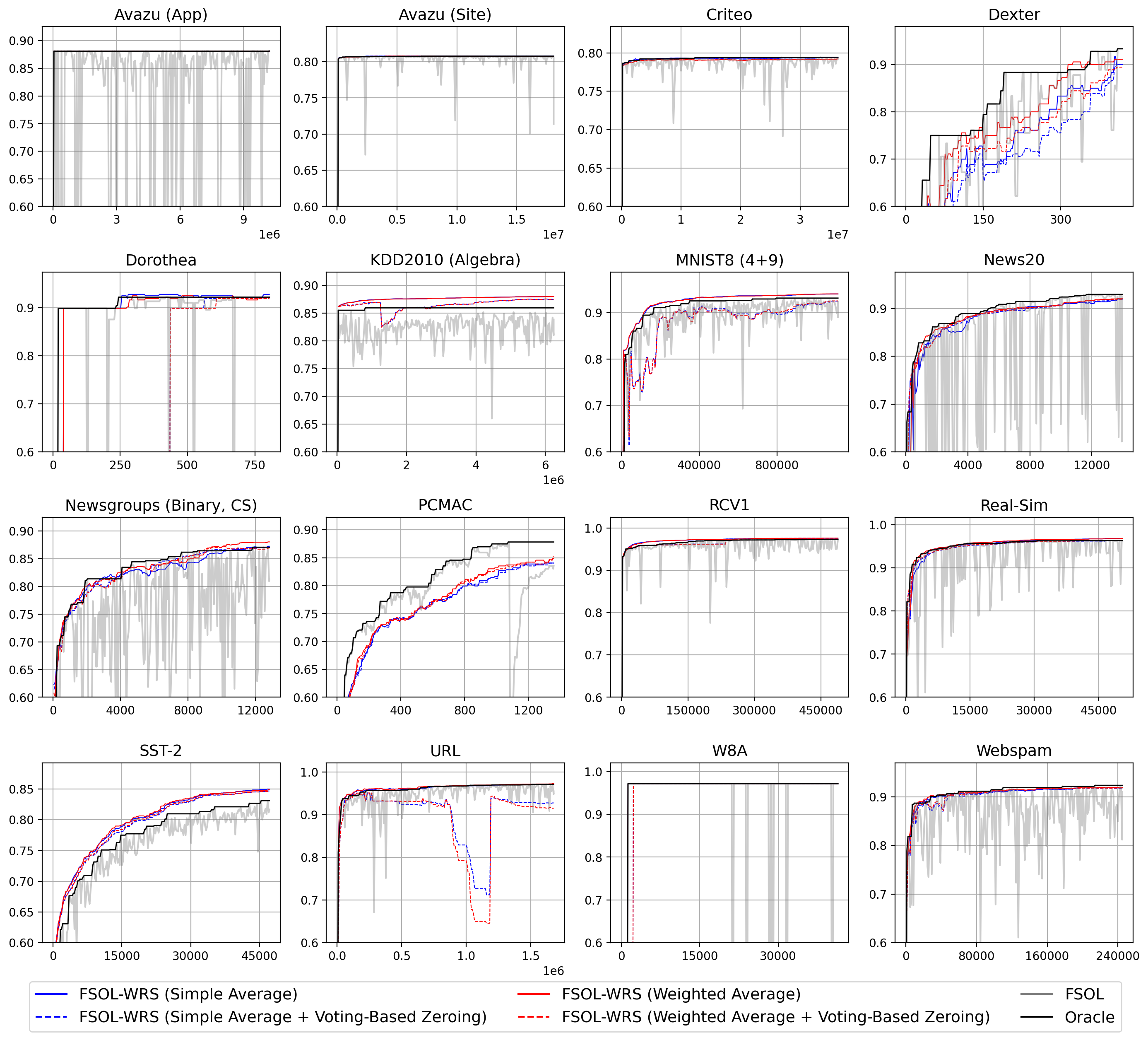}
     \caption{ Test accuracies ($y$-axis) over timestep ($x$-axis) for FSOL-WRS with reservoir size $K=64$ using standard weighting on all 16 tested datasets. \textbf{Light grey lines:} test accuracies of the FSOL baseline methods at each timestep. \textbf{Solid black lines:} test accuracies of the ``oracle" models, computed as the cumulative maximum of the FSOL baselines. \textbf{Blue:} corresponds to FSOL-WRS variants ensembled via simple averaging. \textbf{Red:} corresponds to FSOL-WRS variants ensembled via weighted averaging. \textbf{Dotted lines:} indicate whether voting-based zeroing was applied for additional sparsity. \normalsize}
     \label{fig:FSOL_ws=dense_metric=test-set-acc_K=64}
\end{figure}

\begin{figure}[H]
     \centering
     \includegraphics[scale = 0.39]{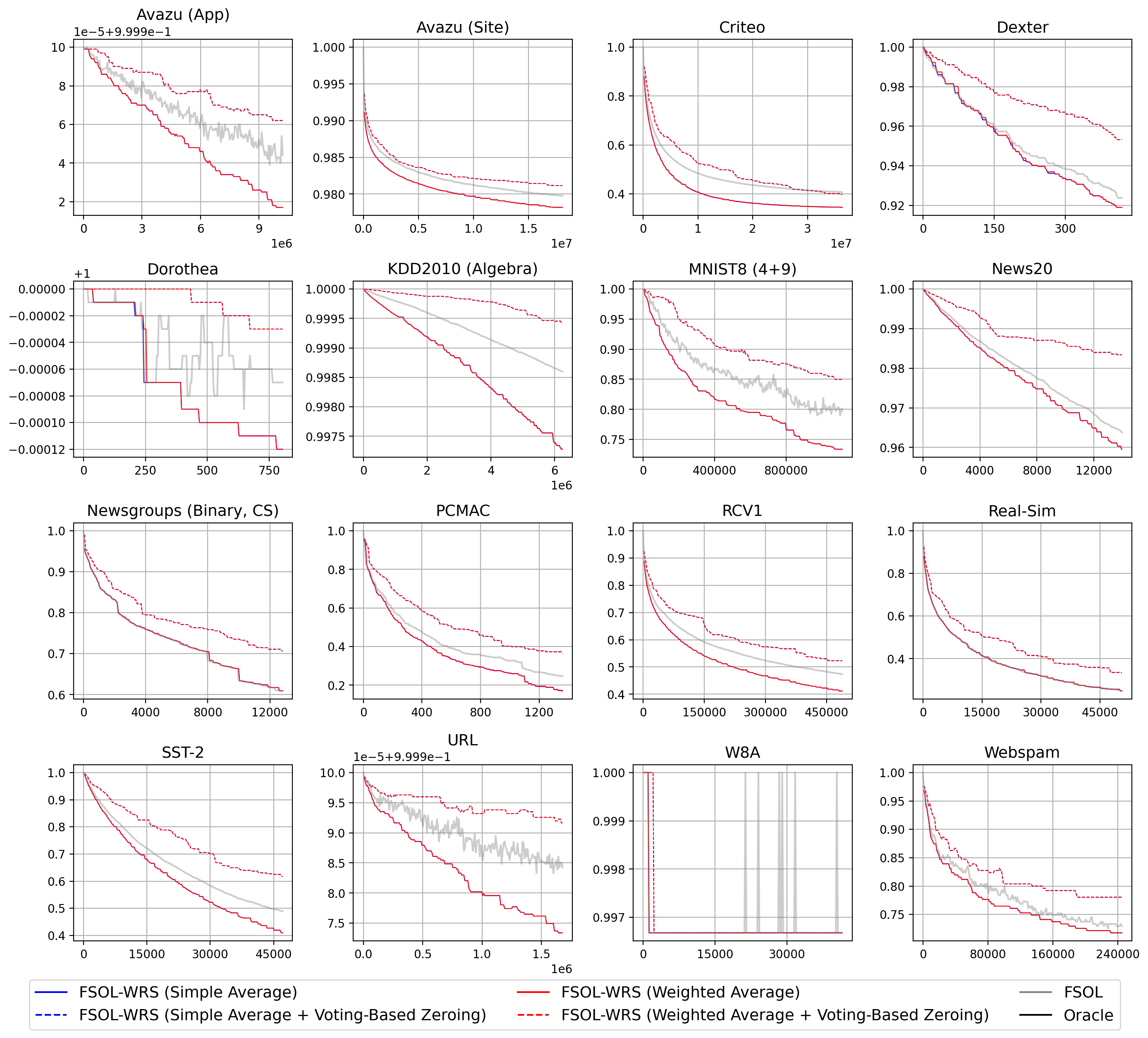}
     \caption{ Sparsity ($y$-axis) over timestep ($x$-axis) for FSOL-WRS with reservoir size $K=64$ using standard weighting on all 16 tested datasets. \textbf{Light grey lines:} sparsities of the FSOL baseline methods at each timestep. \textbf{Blue:} corresponds to FSOL-WRS variants ensembled via simple averaging. \textbf{Red:} corresponds to FSOL-WRS variants ensembled via weighted averaging. \textbf{Dotted lines:} indicate whether voting-based zeroing was applied for additional sparsity. \normalsize}
     \label{fig:FSOL_ws=dense_metric=sparsity_K=64}
\end{figure}

\begin{figure}[H]
     \centering
     \includegraphics[scale = 0.39]{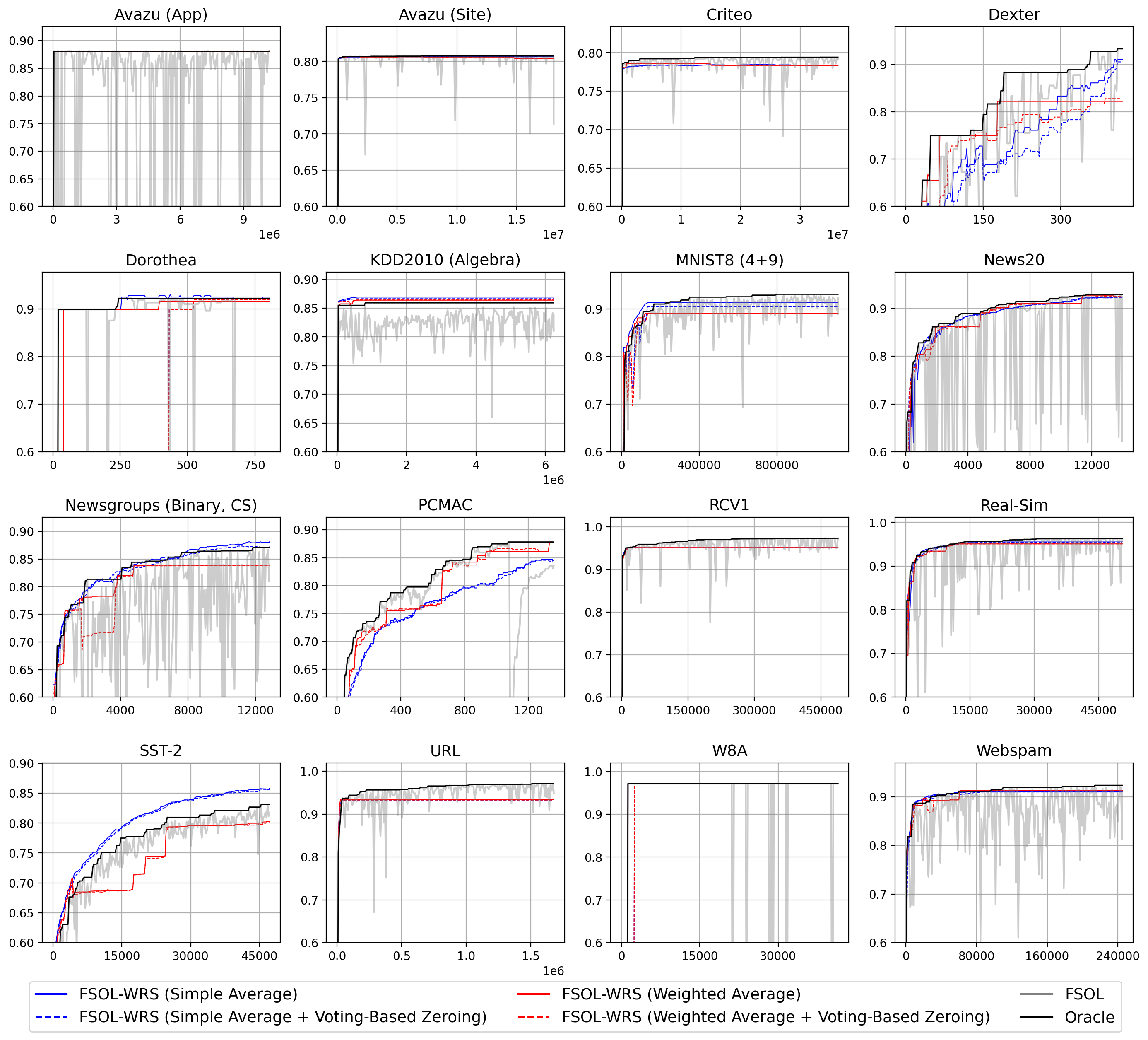}
     \caption{ Test accuracies ($y$-axis) over timestep ($x$-axis) for FSOL-WRS with reservoir size $K=64$ using exponential weighting on all 16 tested datasets. \textbf{Light grey lines:} test accuracies of the FSOL baseline methods at each timestep. \textbf{Solid black lines:} test accuracies of the ``oracle" models, computed as the cumulative maximum of the FSOL baselines. \textbf{Blue:} corresponds to FSOL-WRS variants ensembled via simple averaging. \textbf{Red:} corresponds to FSOL-WRS variants ensembled via weighted averaging. \textbf{Dotted lines:} indicate whether voting-based zeroing was applied for additional sparsity. \normalsize}
     \label{fig:FSOL_ws=exp-dense_metric=test-set-acc_K=64}
\end{figure}

\begin{figure}[H]
     \centering
     \includegraphics[scale = 0.39]{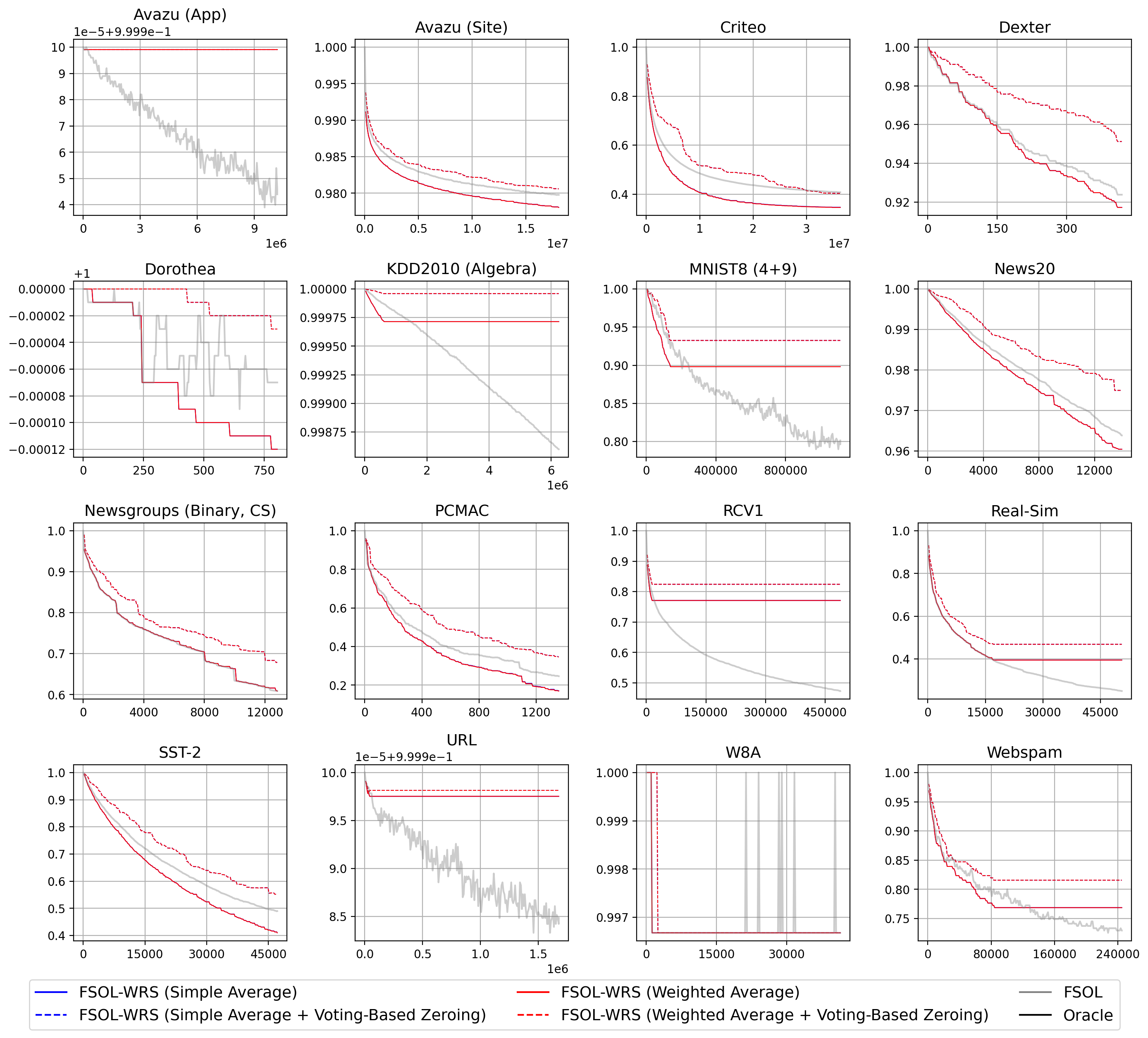}
     \caption{ Sparsity ($y$-axis) over timestep ($x$-axis) for FSOL-WRS with reservoir size $K=64$ using exponential weighting on all 16 tested datasets. \textbf{Light grey lines:} sparsities of the FSOL baseline methods at each timestep. \textbf{Blue:} corresponds to FSOL-WRS variants ensembled via simple averaging. \textbf{Red:} corresponds to FSOL-WRS variants ensembled via weighted averaging. \textbf{Dotted lines:} indicate whether voting-based zeroing was applied for additional sparsity. \normalsize}
     \label{fig:FSOL_ws=exp-dense_metric=sparsity_K=64}
\end{figure}

\subsubsection{PAC and PAC-WRS}
\label{over-time-appendix-pac}

\begin{figure}[H]
     \centering
     \includegraphics[scale = 0.39]{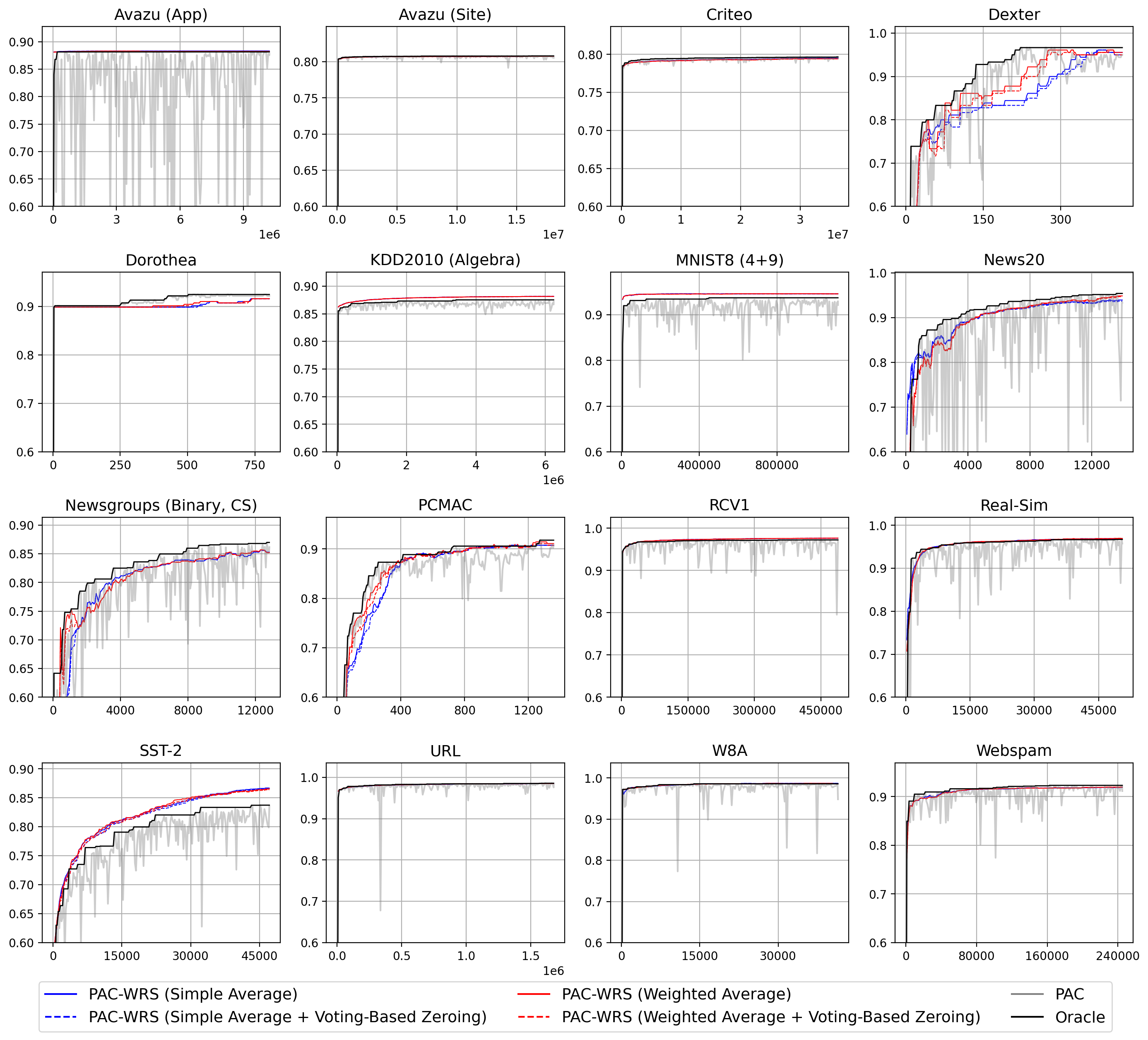}
     \caption{ Test accuracies ($y$-axis) over timestep ($x$-axis) for PAC-WRS with reservoir size $K=64$ using standard weighting on all 16 tested datasets. \textbf{Light grey lines:} test accuracies of the PAC baseline methods at each timestep. \textbf{Solid black lines:} test accuracies of the ``oracle" models, computed as the cumulative maximum of the PAC baselines. \textbf{Blue:} corresponds to PAC-WRS variants ensembled via simple averaging. \textbf{Red:} corresponds to PAC-WRS variants ensembled via weighted averaging. \textbf{Dotted lines:} indicate whether voting-based zeroing was applied for additional sparsity. \normalsize}
     \label{fig:PAC_ws=dense_metric=test-set-acc_K=64}
\end{figure}

\begin{figure}[H]
     \centering
     \includegraphics[scale = 0.39]{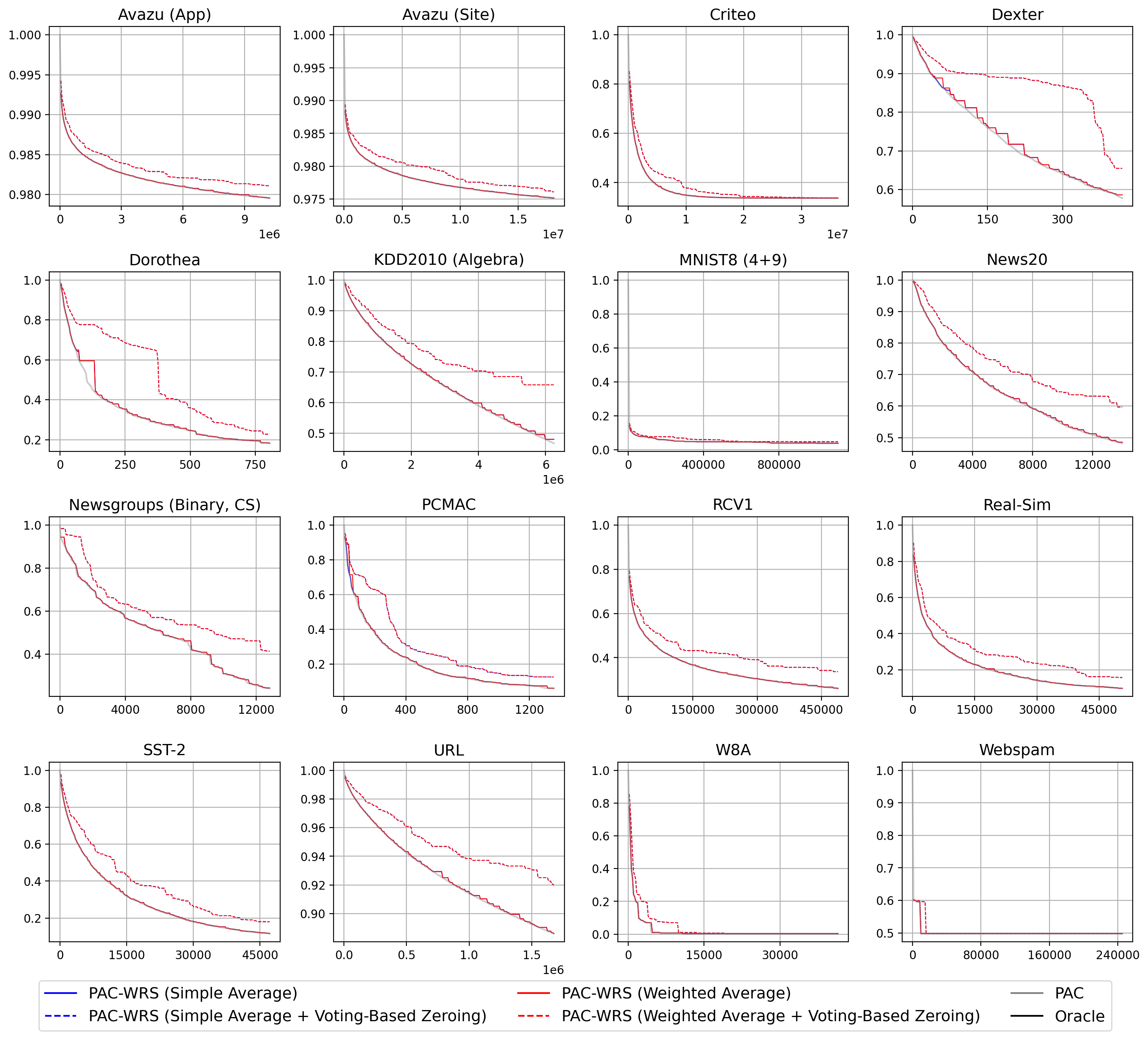}
     \caption{ Sparsity ($y$-axis) over timestep ($x$-axis) for PAC-WRS with reservoir size $K=64$ using standard weighting on all 16 tested datasets. \textbf{Light grey lines:} sparsities of the PAC baseline methods at each timestep. \textbf{Blue:} corresponds to PAC-WRS variants ensembled via simple averaging. \textbf{Red:} corresponds to PAC-WRS variants ensembled via weighted averaging. \textbf{Dotted lines:} indicate whether voting-based zeroing was applied for additional sparsity. \normalsize}
     \label{fig:PAC_ws=dense_metric=sparsity_K=64}
\end{figure}

\begin{figure}[H]
     \centering
     \includegraphics[scale = 0.39]{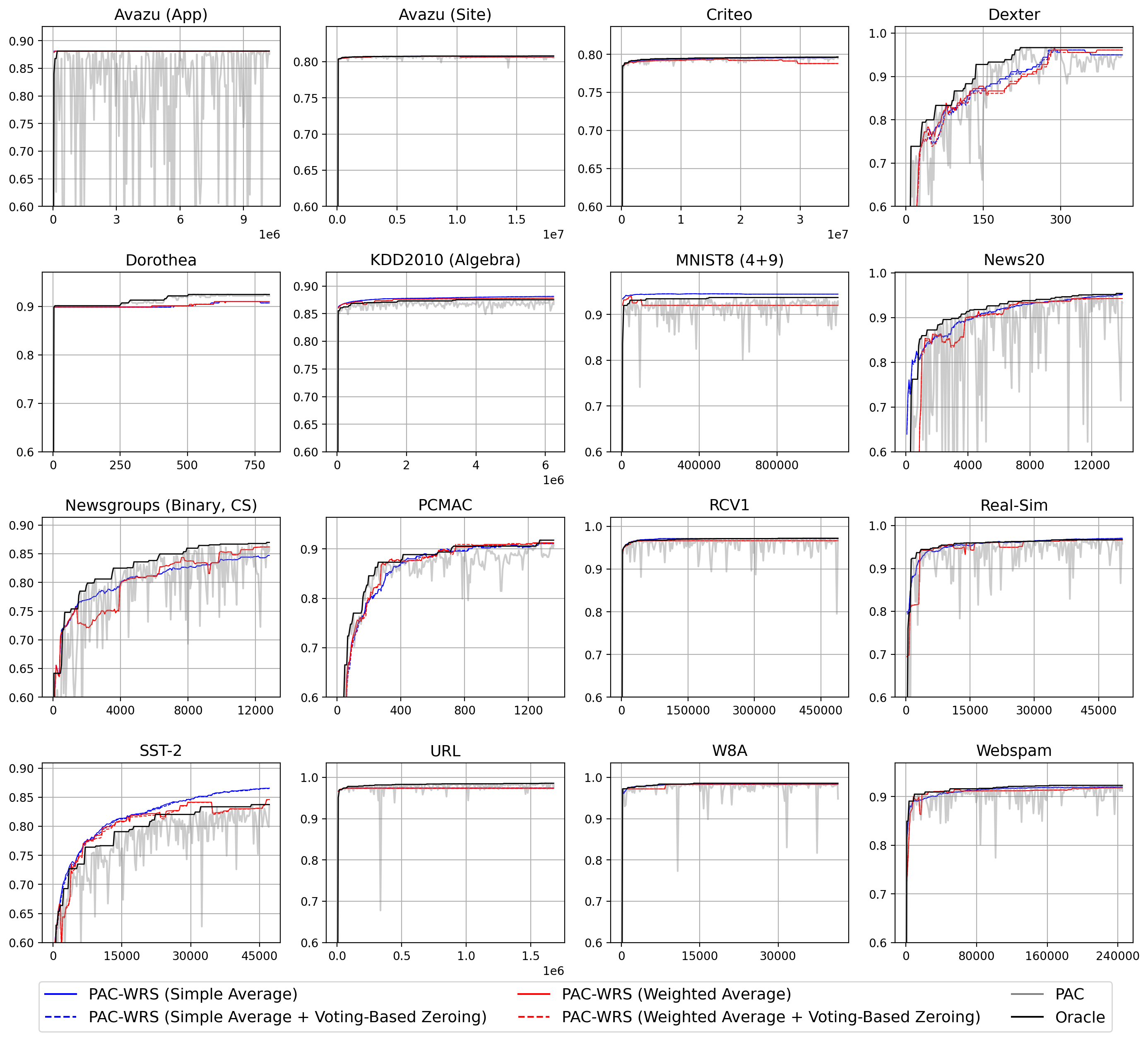}
     \caption{ Test accuracies ($y$-axis) over timestep ($x$-axis) for PAC-WRS with reservoir size $K=64$ using exponential weighting on all 16 tested datasets. \textbf{Light grey lines:} test accuracies of the PAC baseline methods at each timestep. \textbf{Solid black lines:} test accuracies of the ``oracle" models, computed as the cumulative maximum of the PAC baselines. \textbf{Blue:} corresponds to PAC-WRS variants ensembled via simple averaging. \textbf{Red:} corresponds to PAC-WRS variants ensembled via weighted averaging. \textbf{Dotted lines:} indicate whether voting-based zeroing was applied for additional sparsity. \normalsize}
     \label{fig:PAC_ws=exp-dense_metric=test-set-acc_K=64}
\end{figure}

\begin{figure}[H]
     \centering
     \includegraphics[scale = 0.39]{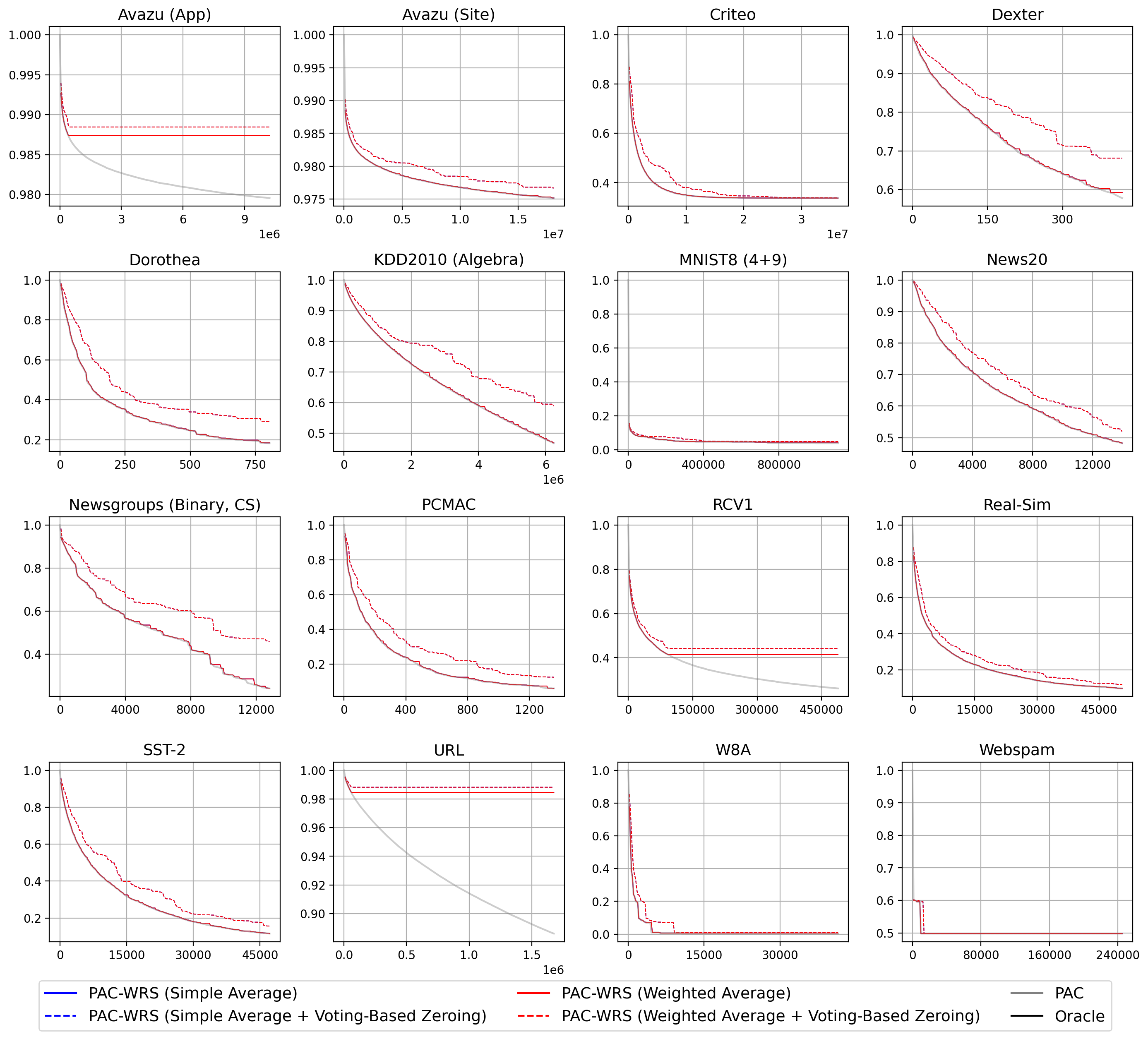}
     \caption{ Sparsity ($y$-axis) over timestep ($x$-axis) for PAC-WRS with reservoir size $K=64$ using exponential weighting on all 16 tested datasets. \textbf{Light grey lines:} sparsities of the PAC baseline methods at each timestep. \textbf{Blue:} corresponds to PAC-WRS variants ensembled via simple averaging. \textbf{Red:} corresponds to PAC-WRS variants ensembled via weighted averaging. \textbf{Dotted lines:} indicate whether voting-based zeroing was applied for additional sparsity. \normalsize}
     \label{fig:PAC_ws=exp-dense_metric=sparsity_K=64}
\end{figure}

\newpage
\section*{NeurIPS Paper Checklist}

\begin{enumerate}

\item {\bf Claims}
    \item[] Question: Do the main claims made in the abstract and introduction accurately reflect the paper's contributions and scope?
    \item[] Answer: \answerYes{} %
    \item[] Justification: Please see our Abstract and Introduction.

\item {\bf Limitations}
    \item[] Question: Does the paper discuss the limitations of the work performed by the authors?
    \item[] Answer: \answerYes{} %
    \item[] Justification: Please see Section \ref{conclusion+future-work}.

\item {\bf Theory Assumptions and Proofs}
    \item[] Question: For each theoretical result, does the paper provide the full set of assumptions and a complete (and correct) proof?
    \item[] Answer: \answerYes{} %
    \item[] Justification: Please see Section \ref{math-theory} and Appendix \ref{appendix-proofs}.

    \item {\bf Experimental Result Reproducibility}
    \item[] Question: Does the paper fully disclose all the information needed to reproduce the main experimental results of the paper to the extent that it affects the main claims and/or conclusions of the paper (regardless of whether the code and data are provided or not)?
    \item[] Answer: \answerYes{} %
    \item[] Justification: Please see Section \ref{numerical-experiments} and Appendix \ref{methods-details-appendix-pac}.

\item {\bf Open access to data and code}
    \item[] Question: Does the paper provide open access to the data and code, with sufficient instructions to faithfully reproduce the main experimental results, as described in supplemental material?
    \item[] Answer: \answerYes{} %
    \item[] Justification: Please see our GitHub at \href{https://github.com/FutureComputing4AI/Weighted-Reservoir-Sampling-Augmented-Training/tree/main}{https://github.com/FutureComputing4AI/Weighted-Reservoir-Sampling-Augmented-Training/tree/main}, with a detailed \texttt{README.md} file.

\item {\bf Experimental Setting/Details}
    \item[] Question: Does the paper specify all the training and test details (e.g., data splits, hyperparameters, how they were chosen, type of optimizer, etc.) necessary to understand the results?
    \item[] Answer: \answerYes{} %
    \item[] Justification: Please see Algorithm \ref{alg:General-WRS}, Section \ref{numerical-experiments}, and Appendix \ref{methods-details-appendix-pac}, as well as our GitHub at \href{https://github.com/FutureComputing4AI/Weighted-Reservoir-Sampling-Augmented-Training/tree/main}{https://github.com/FutureComputing4AI/Weighted-Reservoir-Sampling-Augmented-Training/tree/main}.

\item {\bf Experiment Statistical Significance}
    \item[] Question: Does the paper report error bars suitably and correctly defined or other appropriate information about the statistical significance of the experiments?
    \item[] Answer: \answerYes{} %
    \item[] Justification: Please see error bars and significance tests in Section \ref{numerical-experiments} and Appendix \ref{errorbars-appendix}.

\item {\bf Experiments Compute Resources}
    \item[] Question: For each experiment, does the paper provide sufficient information on the computer resources (type of compute workers, memory, time of execution) needed to reproduce the experiments?
    \item[] Answer: \answerYes{} %
    \item[] Justification: Please see Appendix \ref{methods-details-appendix-pac}.
    
\item {\bf Code Of Ethics}
    \item[] Question: Does the research conducted in the paper conform, in every respect, with the NeurIPS Code of Ethics \url{https://neurips.cc/public/EthicsGuidelines}?
    \item[] Answer: \answerYes{} %
    \item[] Justification: We have carefully read the NeurIPS Code of Ethics and have made sure that our research conforms to it.

\item {\bf Broader Impacts}
    \item[] Question: Does the paper discuss both potential positive societal impacts and negative societal impacts of the work performed?
    \item[] Answer: \answerNA{} %
    \item[] Justification: This paper introduces a foundational research method that does not have any direct paths to negative societal impacts, aside from the general interactions of machine learning with society.
    
\item {\bf Safeguards}
    \item[] Question: Does the paper describe safeguards that have been put in place for responsible release of data or models that have a high risk for misuse (e.g., pretrained language models, image generators, or scraped datasets)?
    \item[] Answer: \answerNA{} %
    \item[] Justification: This paper does not involve any data or models that have a high risk for misuse.

\item {\bf Licenses for existing assets}
    \item[] Question: Are the creators or original owners of assets (e.g., code, data, models), used in the paper, properly credited and are the license and terms of use explicitly mentioned and properly respected?
    \item[] Answer: \answerYes{} %
    \item[] Justification: We enumerate and cite all datasets in Table \ref{table:datasets}. Aside from our own preprocessed versions of Newsgroups (Binary, CS) and SST-2, the 14 other datasets are all publicly available.

\item {\bf New Assets}
    \item[] Question: Are new assets introduced in the paper well documented and is the documentation provided alongside the assets?
    \item[] Answer: \answerYes{} %
    \item[] Justification: Please see the \texttt{README.md} file in our accompanying GitHub \href{https://github.com/FutureComputing4AI/Weighted-Reservoir-Sampling-Augmented-Training/tree/main}{https://github.com/FutureComputing4AI/Weighted-Reservoir-Sampling-Augmented-Training/tree/main}.

\item {\bf Crowdsourcing and Research with Human Subjects}
    \item[] Question: For crowdsourcing experiments and research with human subjects, does the paper include the full text of instructions given to participants and screenshots, if applicable, as well as details about compensation (if any)? 
    \item[] Answer: \answerNA{} %
    \item[] Justification: No crowdsourcing or work with human subjects was performed for this paper.

\item {\bf Institutional Review Board (IRB) Approvals or Equivalent for Research with Human Subjects}
    \item[] Question: Does the paper describe potential risks incurred by study participants, whether such risks were disclosed to the subjects, and whether Institutional Review Board (IRB) approvals (or an equivalent approval/review based on the requirements of your country or institution) were obtained?
    \item[] Answer: \answerNA{} %
    \item[] Justification: This paper does not have any study participants and/or crowdsourcing.

\end{enumerate}

\end{document}

%% file: newc.tex
\newcommand{\newc}{\newcommand}
\newc{\E}{\mbox{E}}
\newc{\V}{\mbox{V}}
\newc{\N}{\mbox{N}}
\newc{\Bern}{\mbox{Bern}}
\newc{\Po}{\mbox{Po}}
\newc{\IG}{\mbox{IG}}
\newc{\Gam}{\mbox{Gam}}
\newc{\bdp}{\mathsf{p}}
\newc{\bdt}{\mathbf{t}}
\newc{\Normal}{\mathcal{N}}
\newc{\Expectation}{\mathbb{E}}
\newc{\odds}{\mbox{odds}}
\newc{\OR}{\mbox{OR}}
\newc{\stderr}{\mbox{s.e.}}
\newc{\logit}{\mbox{logit}}
\newc{\sign}{\mbox{sign}}
\newc{\SD}{\mbox{SD}}
\newc{\bdmu}{\mbox{\boldmath $\mu$}}
\newc{\bdSigma}{\mbox{\boldmath $\Sigma$}}
\newc{\bdLambda}{\mbox{\boldmath $\Lambda$}}
\newc{\bdmuhat}{\mbox{\boldmath $\hat{\mu}$}}
\newc{\bdeta}{\mbox{\boldmath $\eta$}}
\newc{\bdtheta}{\mbox{\boldmath $\theta$}}
\newc{\bdbeta}{\mbox{\boldmath $\beta$}}
\newc{\bdgamma}{\mbox{\boldmath $\gamma$}}
\newc{\bdbetahat}{\mbox{\boldmath $\hat{\beta}$}}
\newc{\bdgammahat}{\mbox{\boldmath $\hat{\gamma}$}}
\newc{\bdthetahat}{\mbox{\boldmath $\hat{\theta}$}}
\newc{\bdvareps}{\mbox{\boldmath $\varepsilon$}}
\newc{\bdzero}{\mbox{\boldmath $0$}}
\newc{\bdone}{\mbox{\boldmath $1$}}
\newc{\bdnu}{\mbox{\boldmath $\nu$}}
\newc{\bdell}{\mbox{\boldmath $\ell$}}
\newc{\bdxi}{\mbox{\boldmath $\xi$}}
\newc{\bdomega}{\mbox{\boldmath $\omega$}}
\newc{\bdepsilon}{\mbox{\boldmath $\varepsilon$}}
\newc{\bdI}{\mathbf{I}}
\newc{\bdP}{\mbox{\boldmath $P$}}
\newc{\bdX}{\mbox{\boldmath $X$}}
\newc{\bdA}{\mbox{\boldmath $A$}}
\newc{\bdB}{\mbox{\boldmath $B$}}
\newc{\bdC}{\mbox{\boldmath $C$}}
\newc{\bdD}{\mbox{\boldmath $D$}}
\newc{\bdG}{\mbox{\boldmath $G$}}
\newc{\bdJ}{\mbox{\boldmath $J$}}
\newc{\bdK}{\mbox{\boldmath $K$}}
\newc{\bda}{\mbox{\boldmath $a$}}
\newc{\bdb}{\mbox{\boldmath $b$}}
\newc{\bdc}{\mbox{\boldmath $c$}}
\newc{\bde}{\mbox{\boldmath $e$}}
\newc{\bdu}{\mbox{\boldmath $u$}}
\newc{\bdv}{\mbox{\boldmath $v$}}
\newc{\bdx}{\mbox{\boldmath $x$}}
\newc{\bdy}{\mbox{\boldmath $y$}}
\newc{\bdz}{\mbox{\boldmath $z$}}
\newc{\bdr}{\mbox{\boldmath $r$}}
\newc{\bdQ}{\mbox{\boldmath $Q$}}
\newc{\bdR}{\mbox{\boldmath $R$}}
\newc{\bdY}{\mbox{\boldmath $Y$}}
\newc{\bdT}{\mbox{\boldmath $T$}}
\newc{\bdW}{\mbox{\boldmath $W$}}
\newc{\bdH}{\mbox{\boldmath $H$}}
\newc{\bdL}{\mbox{\boldmath $L$}}
\newc{\bdU}{\mbox{\boldmath $U$}}
\newc{\bdV}{\mbox{\boldmath $V$}}
\newc{\Multinom}{\mbox{Multinom}}
\newc{\Var}{\mbox{Var}}
\newc{\var}{\mbox{var}}
\newc{\diag}{\mbox{diag}}
\newc{\tr}{\mbox{tr}}
\newc{\phat}{\hat{p}}
\newc{\Xbar}{\bar{X}}
\newc{\xbar}{\bar{x}}
\newc{\Ybar}{\bar{Y}}
\newc{\ybar}{\bar{y}}
\newc{\dbar}{\bar{d}}
\newc{\yhat}{\hat{y}}
\newc{\bdyhat}{\mbox{\boldmath $\hat{y}$}}
\newc{\ytil}{\tilde{y}}
\newc{\ftil}{\tilde{f}}
\newc{\Ho}{\mbox{\bf H}_o}
\newc{\Ha}{\mbox{\bf H}_a}
\newc{\phatYX}{\phat_Y - \phat_X}
\newc{\SSG}{\mbox{SSG}}
\newc{\SSB}{\mbox{SSB}}
\newc{\SSE}{\mbox{SSE}}
\newc{\SST}{\mbox{SST}}
\newc{\SSR}{\mbox{SSR}}
\newc{\SSAB}{\mbox{SSAB}}
\newc{\MSG}{\mbox{MSG}}
\newc{\MSB}{\mbox{MSB}}
\newc{\MSE}{\mbox{MSE}}
\newc{\MST}{\mbox{MST}}
\newc{\MSAB}{\mbox{MSAB}}
\newc{\dfE}{\mbox{dfE}}
\newc{\dfG}{\mbox{dfG}}
\newc{\dfB}{\mbox{dfB}}
\newc{\dfT}{\mbox{dfT}}
\newc{\dfAB}{\mbox{dfAB}}
\newc{\muhat}{\hat{\mu}}
\newc{\betahat}{\hat{\beta}}
\newc{\alphahat}{\hat{\alpha}}
\newc{\etahat}{\hat{\eta}}
\newc{\phihat}{\hat{\phi}}
\newc{\sigmahat}{\hat{\sigma}}
\newc{\cl}{\centerline}
\newc{\redtitle}[1]{ {\color{red}\und{#1}:} }
\newc{\bluetitle}[1]{ {\color{blue}\und{#1}:} }
\newc{\magentatitle}[1]{ {\color{magenta}\und{#1}:} }
\newc{\R}{\mathbb{R}}
\newc{\trans}{^\mathsf{T}}
\newc{\xtx}{\bdX\trans\bdX}
\newc{\xxtxx}{\bdX(\xtx)^{-1}\bdX\trans}
\newc{\bdk}{\mathbf{k}}

%% file: maths/1_validity.tex
\textbf{Validity analysis.} We assume a given dataset $z_1^T = \{z_t\}_{t=1}^T$ is an i.i.d. sequence sampled from a generating distribution $\mathcal{D}$.
We first suppose that $\ell(\bdw; z)$ is the 0-1 loss,
that is $\ell(\bdw; (x, y)) = \mathbbm{1}(\hat{y} \neq y)$.
Then the risk $R_\mathcal{D}(\bdw)$ is the probability that a random $z\sim\mathcal{D}$ is misclassified by $\bdw$.
At any time $t$, we have an online model $\bdw_t$, for which we define the \textit{survival time} $s_{\bdw_t}$ to be the number of subsequent correct classifications, stopping when $\bdw_t$ misclassifies a sample.
As $\bdw_t$ does not change until an error occurs, at any finite time we only collect $K$ updated models $\{\bdw^{(j)}\}_{j=1}^K$ into $\mathcal{R}$.

Our first result bounds the difference in risk among two models in $\mathcal{R}$:
(1) the minimal-risk hypothesis $\tilde{\bdw} = \argmin R_{\mathcal{D}}(\bdw^{(j)})$ which we do not know, and
(2) the longest-surviving hypothesis $\bdw_s = \argmin s_{\bdw^{(j)}}$ which we observe.
We also denote $s_{\tilde{\bdw}}$ and $s_{\bdw_s}$ to be their respective survival times.

\begin{thm} \label{thm:top1_bound}
    Let $\bdw^{(1)}, \ldots, \bdw^{(k)} \in \mathcal{R}$ be the updated outputs of an online PA algorithm on inputs $Z_1^T \sim \mathcal{D}$.
    Also define $\rmin$ as the minimal achievable risk of any model, such that $\rmin \leq R_\mathcal{D}(\tilde{\bdw})$ almost surely. Then
    $$P(R_\mathcal{D}(\bdw_s) > R_\mathcal{D}(\tilde{\bdw}) + \varepsilon ) \leq \min\bigg\{K \frac{\rmin}{ 2\rmin + \varepsilon - \rmin (\rmin + \varepsilon) },  e^{-\rmin^K (\rmin + \varepsilon)} \bigg\} $$
\end{thm}

This immediately leads to
\begin{cor}
    With prob. $1-\delta$,
    $R_\mathcal{D}(\bdw_s) \leq R_\mathcal{D}(\tilde{\bdw}) + \min\bigg\{ \frac{1 - \delta - \rmin^{K+1}}{\rmin^K},  \frac{\rmin(K - \delta(2-\rmin))}{\delta(1-\rmin)}  \bigg\}$.
\end{cor}

While this bound explains the use of top-1 survival in the worst case,
we further justify the use of an ensemble (or reservoir) of top-surviving models.
Specifically, given the use of top-$B$ models, there is always a certain probability that the top-$(B+1)$ can lower the risk. 
Since $B=1$ is a base case, this implies that any value of $B$ is worthy of consideration, until this probability decays to 0.

To simplify the notation, we sort the models in $\mathcal{R}$ by decreasing survival time $s'_1, \ldots, s'_k$, with corresponding re-indexed weights $w'_1, \ldots, w'_k$ and risks $R'_1, \ldots, R'_k$.
We also add an assumption, that there is a partition of $\mathcal{R} = \mathcal{R}_g \cup \mathcal{R}_b$ such that 
$\mathcal{R}_g = \{ \bdw \in \mathcal{R} : \rmin \leq R_\mathcal{D}(\bdw) \leq   \rmin + \varepsilon \} $
and 
$\mathcal{R}_b = \{ \bdw \in \mathcal{R} : \rmin + \varepsilon < R_\mathcal{D}(\bdw) \leq \rmin + 3\varepsilon \} $,
and that $|\mathcal{R}_b| \geq B $.
That is, there is a set of \textit{good} and \textit{bad} models in terms of risk. In the Appendix we show that the assumption is readily satisfied.

\begin{thm} \label{thm:topB_bound}
    Let $\bar R_\mathcal{D}(\bdw_{B})$ be the averaged risk of the top-$B$ surviving models, and let $\bar R_\mathcal{D}(\bdw_{B+1})$ be the averaged risk including the next highest survival model.
    Then $\bar R_\mathcal{D}(\bdw_{B+1}) \leq \bar R_\mathcal{D}(\bdw_{B})$
    with probability at least
    $|\mathcal{R}_g|\binom{|\mathcal{R}_b|}{B} \rmin^2 (\rmin + 3\varepsilon)^{|\mathcal{R}_b| - B}(1 - \rmin - 3\varepsilon)^B$.
\end{thm}

Finally note that we can define an averaged model (weighted or unweighted) $\bar\bdw_B$. For any convex $\ell_h(\bdw; z)$ with risk $R_\mathcal{D}^h$, such as the hinge loss, Jensen's inequality gives 
$R_\mathcal{D}^h(\bar\bdw_{B+1}) \leq \bar R_\mathcal{D}(\bdw_{B+1})$.

%% file: maths/2_learning.tex
\textbf{Learning bounds.} Now we turn to generalization bounds of our method,
when run to a fixed stopping time $T$.
We assume more generally that $\ell(\bdw; z)$ is convex, such as the hinge loss. Furthermore we suppose that the loss of any point in the training set is bounded by $C$.
This is a safe assumption for many passive-aggressive algorithms, where the input data is normalized and the update steps are not too large. 

\begin{thm} \label{thm:wrs_bound}
    Suppose $\ell$ is convex and bounded from above by $C$. By time $T$, suppose the reservoir contains $K_T$ models, with survival time of each at least $s_T$.
    Let $M_{wrs}=\sum_{t=1}^T \pi_t \ell(\bdw_{t-1}; z_t)$ be the cumulative loss of the WAT reservoir sequence,
    with $\pi_t$ formally defined in the proof.
    Then w.p. $1-\delta$,
    $$ R_\mathcal{D}(\bdw_{wrs}) \leq \frac{M_{wrs}}{K_T s_T} + \sqrt{\frac{2 C \log \Big(\frac{T}{\delta}\Big) M_{wrs}}{(K_T s_T)^2}} + \frac{7C \log \Big(\frac{T}{\delta}\Big)}{K_T s_T}$$
\end{thm}

The result shows that the risk of the ensemble model is stochastically bounded by the cumulative loss of the online procedure.
By applying known regret bounds for the underlying online algorithms (e.g. PAC or FSOL), we can further bound the risk in terms of the original learner, and subsequently the optimal risk.
In particular, we can ``abstract out'' the actual online learning method in the proof,
as a generic regret bound of form $ \frac{M_T}{T} \leq \frac{1}{T}\sum_{t=1}^T \ell(\bdw; z_t) + \frac{r(T)}{T}$, for any $\bdw$.
Depending on the algorithm often $r(T) = O(\sqrt{T})$ or even $O(1)$.

With this in hand, we can show that the risk actually approaches the minimal risk. One additional assumption is required: $M_{wrs}/\sum_{t=1}^T \pi_t \leq \frac{1}{T} \sum_{t} \ell(\bdw_{t-1}; z_t) = \frac{M_T}{T}$ (the WRS sequence has lower cumulative loss than the original learner). By definition, the reservoir contains the models which have longest survivals, and hence lowest regret density. So this statement is readily satisfied, as long as the distributions of non-zero losses are not different among the two sequences. In fact, under certain conditions, the inequality is likely strict which further improves the bound. Altogether we obtain:

\begin{thm}  \label{thm:risk_bound}
   Given a PA algorithm, let $M_T$ be its cumulative loss, and $r(T)$ be the algorithm-specific excess regret. Then with probability $1 - \delta$, the deviation in risk of our WRS model $\bdw_{wrs}$ from the optimal model $\bdw^*$ is bounded as
    $$
R_\mathcal{D}(\bdw_{wrs}) \leq R_\mathcal{D}(\bdw^*) + \frac{r(T)}{T} + \sqrt{\frac{2 C \log \Big(\frac{2T}{\delta}\Big) M_T}{T K_T s_T}} + \frac{7C \log \Big(\frac{2T}{\delta}\Big)}{K_T s_T} + C \sqrt{\frac{\log \Big(\frac{2}{\delta}\Big)}{2T}}
$$
\end{thm}

%% file: maths/3_proofs.tex
\subsection{Validity analysis}

\textbf{Proof of Theorem~\ref{thm:top1_bound}}. The main statement involves two independent bounds. See the proofs for Theorem~\ref{thm:top1_union} and Theorem~\ref{thm:top1_direct}.

\begin{thm} \label{thm:top1_union}
    Let $\bdw^{(1)}, \ldots, \bdw^{(k)} \in \mathcal{R}$ be the updated outputs of an online PA algorithm on inputs $z_1^T \sim \mathcal{D}$.
    Define $\bdw_s$ and $\tilde{\bdw}$ as previously.
    Also define $\rmin$ as the minimal achievable risk of any model, such that $\rmin \leq R_\mathcal{D}(\tilde{\bdw})$ almost surely.

    Then $$P(R_\mathcal{D}(\bdw_s) > R_\mathcal{D}(\tilde{\bdw}) + \varepsilon ) \leq K \frac{\rmin}{ 2\rmin + \varepsilon - \rmin (\rmin + \varepsilon) } $$
\end{thm}
\begin{proof}
    First note that while $R_\mathcal{D}(\bdw)$ is a constant value for any fixed $\bdw$,
$R_\mathcal{D}(\bdw_s)$ and $R_\mathcal{D}(\tilde{\bdw})$ are random variables induced by the distribution of $Z_1^T$.
The sequence of observed samples affects both the algorithm's output $\bdw$ as well as the subsequent survival of $\bdw$.

Next, we define the event of interest $A = \{ R_\mathcal{D}(\bdw_s) > R_\mathcal{D}(\tilde{\bdw}) + \varepsilon \}$.
For a given realization of $Z_1^T$ let the set $\mathcal{R}_\varepsilon = \{ \bdw \in \mathcal{R} : R_\mathcal{D}(\bdw) > R_\mathcal{D}(\tilde{\bdw}) + \varepsilon \}$. Then
\begin{align*}
    A &\subseteq \{\exists \bdw \in \mathcal{R} : (s_\bdw \geq s_{\tilde{\bdw}}) \cap (R_\mathcal{D}(\bdw) > R_\mathcal{D}(\tilde{\bdw}) + \varepsilon)\} \\
      &= \{ \exists \bdw \in \mathcal{R}_\varepsilon : s_\bdw \geq s_{\tilde{\bdw}} \}
\end{align*}
We aim to split this into separate events and apply a union bound, but the contents and size of $\mathcal{R}$ depend on $Z_1^T$.

Therefore we enumerate the models within each $\mathcal{R}_\varepsilon$ in the order they are received: $\{\bdw_\varepsilon^{(1)}, \bdw_\varepsilon^{(2)}, \ldots\}$, up to $\bdw_\varepsilon^{(K)}$,
enabling the index $i$ to refer to the corresponding model $\bdw_\varepsilon^{(i)}$ in any realization of $Z_1^T$.
Then we can split the event that a hypothesis in the reservoir has longer survival into the union of events on each ordered hypothesis.

We have
\begin{align*}
    A &\subseteq \{ \exists \bdw \in \mathcal{R}_\varepsilon : s_\bdw \geq s_{\tilde{\bdw}} \} \\
    &\subseteq \bigcup_{i} \{ s_{\bdw_\varepsilon^{(i)}} \geq s_{\tilde{\bdw}} \}
\end{align*}

Applying the union bound, Lemma~\ref{lem:same_dist}, and then Lemma~\ref{lem:geom},
\begin{align*}
    P(A) &\leq \sum_{i} P(s_{\bdw_\varepsilon^{(i)}} \geq s_{\tilde{\bdw}})\\
    &\leq \sum_{i} \mathbb{E}_{\bdw_\varepsilon^{(i)}, \tilde{\bdw}}P(s_{\bdw_\varepsilon^{(i)}} \geq s_{\tilde{\bdw}}  | \bdw_\varepsilon^{(i)}, \tilde{\bdw}  )\\
    &= \sum_{i} \mathbb{E}_{\bdw_\varepsilon^{(i)}, \tilde{\bdw}}P(S(\bdw_\varepsilon^{(i)}) \geq S(\tilde{\bdw}) | \bdw_\varepsilon^{(i)}, \tilde{\bdw}  ) \\
    &= \sum_{i} \mathbb{E}_{\bdw_\varepsilon^{(i)}, \tilde{\bdw}}\frac{ R_\mathcal{D}(\tilde{\bdw}) }{ R_\mathcal{D}(\bdw_\varepsilon^{(i)}) + R_\mathcal{D}(\tilde{\bdw}) - R_\mathcal{D}(\bdw_\varepsilon^{(i)}) R_\mathcal{D}(\tilde{\bdw}) } \\
    &\leq K \frac{\rmin}{ 2\rmin + \varepsilon - \rmin (\rmin + \varepsilon) }
\end{align*}

To further clarify, the probabilities are computed over the probability space of $Z_1^T\sim \mathcal{D}$, and $\tilde{\bdw}$ depends on the realization of $Z_1^T$.
Lemma~\ref{lem:same_dist} allows us to reason about the dependent model instances as independent geometric variables,
but we need to extract $\tilde{\bdw}$ into a instance-independent state.
Thus we need to condition on the weight values, and then bound the probability uniformly across the expectation, using the premise that the weight risks are $\varepsilon$-distant.
\end{proof}

\newpage
\begin{thm} \label{thm:top1_direct}
    (Alternative bound) Let the same conditions hold as in Theorem~\ref{thm:top1_union}. Then
    $$P(R_\mathcal{D}(\bdw_s) > R_\mathcal{D}(\tilde{\bdw}) + \varepsilon ) \leq e^{-\rmin^K (\rmin + \varepsilon)} $$
\end{thm}
\begin{proof}
As before, let $A = \{ R_\mathcal{D}(\bdw_s) > R_\mathcal{D}(\tilde{\bdw}) + \varepsilon \}$.
Then \begin{align*}
    A \subseteq \{ R_\mathcal{D}(\bdw_s) > R_\mathcal{D}(\tilde{\bdw}) \} &\subseteq \{ \bdw_s \neq \tilde{\bdw}\} \\
    &= \{ \bdw_s = \tilde{\bdw} \}^c \\
    &= \{s_{\tilde{\bdw}} \geq s_\bdw, \forall \bdw \in \mathcal{R} \}^c
\end{align*}

The event $\{ \bdw_s = \tilde{\bdw} \}^c$ implies $\tilde{\bdw}$ had or was tied for the highest survival,
and we make the simplification that ties are resolved in favor of $\tilde{\bdw}$. Then 
\begin{align*}
    P(A) &\leq 1 - P(s_{\tilde{\bdw}} \geq s_\bdw, \forall \bdw \in \mathcal{R})\\
    &= 1 - P( s_{\tilde{\bdw}} \geq \max_\bdw s_{\bdw}, \forall \bdw \in \mathcal{R})\\
    &= 1 - \sum_{x=1}^\infty R_\mathcal{D}(\tilde{\bdw}) (1 - R_\mathcal{D}(\tilde{\bdw}))^{x-1} \prod_{j=1}^K \Big(1 - (1 - R_\mathcal{D}(\bdw^{(j)}))^x\Big)\\
    &\leq  1 - \sum_{x=1}^\infty R_\mathcal{D}(\tilde{\bdw}) (1 - R_\mathcal{D}(\tilde{\bdw}))^{x-1} \Big(1 - (1 - R_\mathcal{D}(\tilde{\bdw})^x\Big)^{K-1}\Big(1 - (1 - R_\mathcal{D}(\tilde{\bdw}) - \varepsilon)^x\Big)
\end{align*}
where in the third step, we apply Lemma~\ref{lem:order}.
In the last step, we use the premise to suppose that there is at least one hypothesis which is $\varepsilon$-worse than $\tilde{\bdw}$. 

Considering the first-term only,
\begin{align*}
    P(A) &\leq 1 - R_\mathcal{D}(\tilde{\bdw}) ( R_\mathcal{D}(\tilde{\bdw}) )^{K-1} (R_\mathcal{D}(\tilde{\bdw}) + \varepsilon) \\
    &= 1 - R_\mathcal{D}(\tilde{\bdw})^K (R_\mathcal{D}(\tilde{\bdw}) + \varepsilon) \\
    &\leq \exp\big(-R_\mathcal{D}(\tilde{\bdw})^K (R_\mathcal{D}(\tilde{\bdw}) + \varepsilon)\big)\\
    &\leq \exp\big(-\rmin^K (\rmin + \varepsilon)\big)
\end{align*}
\end{proof}

\begin{remark}
    Note that this gives the worst-case bound presuming no knowledge over the distribution of risks among the $K$ hypotheses.
In such a scenario,
the worst case is that all hypotheses achieve $R_\mathcal{D}(\tilde{\bdw})$ except for one which is $R_\mathcal{D}(\tilde{\bdw}) + \varepsilon$.
This is unrealistically adversarial,
but the bound can easily be tightened by adding assumptions that the risks $R_\mathcal{D}(\bdw^{(j)})$ are well-behaved and applying those in the last step.
\end{remark}

\begin{lemma} \label{lem:same_dist}
    Let $R_\mathcal{D}(\bdw)$ be the misclassification risk. Suppose an online learner outputs hypothesis $\bdw_t$ at time $t$ after observing samples $Z_1^t$. Let $s_t$ be the subsequent survival time of $\bdw_t$, conditional on the observed sequence $Z_1^t$.
    Also let $S(\bdw_t)$ be the survival of $\bdw_t$ over the distribution of sample sequences from $\mathcal{D}$.
    Then $S(\bdw_t)$ has a Geometric distribution with parameter $R(\bdw_t)$. 
    In addition $s_t \,{\buildrel d \over =}\,S(\bdw)$.
\end{lemma}
\begin{proof}
For the first statement, observe that when sampling $z$ independently from $\mathcal{D}$,
$$R(\bdw_t) = \mathbb{E}_\mathcal{D}[\mathbbm{1}(\bdw_t(x) \neq y)] = P_{(x,y) \sim \mathcal{D}}(\bdw_t(x) \neq y)$$
Since each i.i.d. sample has a probability $R(\bdw_t)$ to be misclassified,
the number of trials until the first misclassification is a Geometric distribution with parameter $R(\bdw_t)$.

For the second, we need to consider a sequence of correct/incorrect classifications by $\bdw_t$ as a Bernoulli process of r.v. $W_i$ with parameter $R(\bdw_t)$,
with a stopping time $\tau$ indicated by the first misclassification. 
$\tau$ is a random variable and we see that $S_t = \tau - t$.
Similarly $S(\bdw_t)$ can be defined on another Bernoulli process starting from $t=0$.
From the memoryless property, these have the same distribution.
\end{proof}

\begin{lemma} \label{lem:geom}
    Given random variables $X$ and $Y$ which are independent Geometric with parameters $p$ and $q$ respectively (e.g. PMF of $X$ is $P(X = k) = (1 - p)^{k - 1} p$, then
$$ P (X \geq Y) = \frac{q}{p + q - pq}$$
and $$P(X > Y) = \frac{q - pq}{p + q - pq}$$
\end{lemma}

\begin{proof}
For the first statement,
    \begin{align*}
        P(X \geq Y) &= \sum_{y=1}^\infty P(Y = y \cap X \geq y)\\
           &= \sum_{y=1}^\infty P(X \geq y | Y = y) P(Y = y) \\
           &= \sum_{y=1}^\infty P(X \geq y) P(Y = y) \\
           &= \sum_{y=1}^\infty \sum_{x=y}^\infty P(X = x) P(Y = y)\\
           &= \sum_{y=1}^\infty \sum_{x=y}^\infty (1 - p)^{x - 1} p \cdot (1 - q)^{y-1} q\\
           &= pq \sum_{y=1}^\infty (1-q)^{y-1} \sum_{x=y}^\infty (1-p)^{x-1}\\
           &= pq \cdot \sum_{y=1}^\infty (1-q)^{y-1} \cdot \frac{(1-p)^{y-1}}{p}\\
           &= \frac{q}{1 - (1 - p)(1 - q)} \\
           &= \frac{q}{p + q - pq}
    \end{align*}

For the second, note that
\begin{align*}
    P(X > Y) = 1 - P(Y \geq X) &= 1 - \frac{p}{p + q - pq}\\
    &= \frac{q - pq}{p + q - pq}
\end{align*}
\end{proof}

\begin{lemma} \label{lem:order}
Let $Y_1, \ldots, Y_t$ be independent Geometric random variables with respective parameters $q_1, \ldots, q_t$. Let $X\sim \mathrm{Geom}(p)$ be independent. Then
$$ P(X \geq \max_i Y_i) = \sum_{x=1}^\infty p(1 - p)^{x-1} \prod_{i=1}^t \Big(1 - (1 - q_t)^x\Big)  $$
\end{lemma}
\begin{proof}
    \begin{align*}
        P(X \geq \max_i Y_i) &= \sum_{x=1}^\infty P(X = x \cap \max_i Y_i \leq x) \\
        &= \sum_{x=1}^\infty P(X = x) P(\max_i Y_i \leq x)\\
        &= \sum_{x=1}^\infty P(X = x) \prod_{i=1}^t P(Y_i \leq x)
    \end{align*}
    by independence. Substitute the PMF and CMFs of the geometric distribution to finish.
\end{proof}

\textbf{Proof of Theorem~\ref{thm:topB_bound}}.
\begin{proof}
Note that
\begin{align*}
    A&=\Big\{\frac{1}{B+1} \sum_{j=1}^{B+1} R'_j < \frac{1}{B} \sum_{j=1}^B R'_j \Big\}\\
    &= \Big\{ B \sum_{j=1}^{B+1} R'_j < (B+1)  \sum_{j=1}^B R'_j \Big\}\\
    &= \Big\{ R'_{B+1} < \frac{1}{B} \sum_{j=1}^B R'_j \Big\}\\
\end{align*}

We now consider a subset of this event. We will compute the probability that $R'_1, \ldots, R'_B \in \mathcal{R}_b$ and that $R'_{B+1} \in \mathcal{R}_g$.
By definition, this latter event falls into $A$.

For this to occur, note from Lemma~\ref{lem:same_dist} that each survival is drawn from a Geometric distribution.
We want the chance that the $B$ largest survivals are from set $\mathcal{R}_b$ and then the next largest is from $\mathcal{R}_g$.
For additional simplicity, we further minimize the probability of this event by assigning all entries in $\mathcal{R}_g$ risk $\rmin$ and all entries in $\mathcal{R}_b$ risk $\rmin + 3\varepsilon$.

Conditioning on $s'_{B+1}=x$, all entries from $\mathcal{R}_g$ are below $x$ except for one, which is equal to $x$. There are $|\mathcal{R}_g|$ ways to pick this largest one.

Then $|\mathcal{R}| - B$ of the models in $\mathcal{R}_b$ are also below $x$, and $B$ are above $x$. There are $\binom{|\mathcal{R}_b|}{B}$ ways to choose such $B$.

Finally this probability is summed over all values of $x$. This is
\begin{align*}
    \sum_{x=1}^\infty |\mathcal{R}_g| (1-(1-\rmin)^x) (1-\rmin)^{x-1} \rmin \cdot \binom{|\mathcal{R}_b|}{B}(1 - (1 - \rmin - 3\varepsilon)^x)^{|\mathcal{R}_b|-B} (1-\rmin - 3\varepsilon)^{Bx}
\end{align*}
As before, to get a simpler expression, we just consider $x=1$, yielding
$$P(A) \geq |\mathcal{R}_g|\binom{|\mathcal{R}_b|}{B} \rmin^2 (\rmin + 3\varepsilon)^{|\mathcal{R}_b| - B}(1 - \rmin - 3\varepsilon)^B$$

\end{proof}

We also discuss the additional assumption of this Theorem to show that it is readily satisfied.

\textbf{Assumption.} There is a partition of $\mathcal{R} = \mathcal{R}_g \cup \mathcal{R}_b$ such that 
$\mathcal{R}_g = \{ \bdw \in \mathcal{R} : \rmin \leq R_\mathcal{D}(\bdw) \leq   \rmin + \varepsilon \} $
and 
$\mathcal{R}_b = \{ \bdw \in \mathcal{R} : \rmin + \varepsilon < R_\mathcal{D}(\bdw) \leq \rmin + 3\varepsilon \} $,
and that $|\mathcal{R}_b| \geq B $.
That is, there is a set of \textit{good} and \textit{bad} models in terms of risk. 

We can set $\varepsilon>0$ to be any small number that satisfies a partition within the reservoir separating "good" from "bad" models -- e.g. $\varepsilon = (\max_k R'_k - r_m) / 4$ works.
For demonstration, suppose the underlying risks of the models in the reservoir range from 0.1 to 0.3, with 0.1 being the best risk $r_m$.
Then $\varepsilon = 0.05$.
$\mathcal{R}_g$ then contains the models with risks in $[0.1, 0.15]$ and $\mathcal{R}_b$ contains those with risks in $(0.15, 0.3]$. The only time this isn't satisfied for any $\varepsilon$ is if one of the partitions always ends up empty, which only happens if all the $K$ models have the same risk, a highly unlikely scenario where it is obvious that an ensemble isn't needed.

\subsection{Learning bounds}

Our bound follows from the following lemma from~\cite{dekel2008online}. This itself is based on martingale inequalities from~\cite{freedman1975tail} and developments in~\cite{cesa2008improved}.

\begin{lemma}[Dekel] \label{lem:dekel}
    Let $L_1, \ldots, L_T$ be a sequence of real-valued random variables and let $Z_1, \ldots, Z_T$ be a sequence of random variables such that $L_t = \mathbb{E}[L_t | Z_1, \ldots, Z_t]$. Also assume that $L_t \in [0, C]$ for all $t$. Define $U_t = \mathbb{E}[L_t|Z_1, \ldots, Z_{t-1}]$,
    and let $\tilde{L}_t = \sum_{i=1}^t L_i$. Then for any $T\geq 4$ and $\delta \in (0, 1)$,
    $$P\bigg(\forall t \in \{1, \ldots, T\}, \hspace{1em} \sum_{i=1}^t U_i > \tilde{L}_t  + \sqrt{2C\log\Big(\frac{T}{\delta}\Big) \tilde{L}_t} + 7 C \log\Big(\frac{T}{\delta}\Big)  \bigg) < \delta $$
\end{lemma}

From this we obtain our learning bound. The proof follows analogously to the one in~\cite{dekel2008online}, with modifications to handle our specific algorithm.

\textbf{Proof of Theorem~\ref{thm:wrs_bound}}.
\begin{proof}
In our method we have a reservoir $\mathcal{R}$ containing $K$ models $\bdw$ with associated weights $b_\bdw$.
We focus on the case where the weight is equal to the passive steps survived by the model.
To simplify this proof we consider a slightly modified algorithm, where all models with survival exceeding a threshold $s$ are included in the reservoir.

\begin{equation}
    \pi_t = 
    \begin{cases}
         0, & \text{if}\ s(\bdw_t) < s \\
         s, & \text{if}\ s(\bdw_t) = s \\
         1, & \text{otherwise}
    \end{cases}
\end{equation}

We rewrite the weighted average of the reservoir models as a function of the entire training sequence of length $T$ by defining $\pi_t$ to be 1 if $\bdw_t$ is in $\mathcal{R}$ and 0 otherwise.
Furthermore on the first instance that $\bdw_t$ joins $\mathcal{R}$, the $\pi_t$ is instead set to $s$.
This makes the weighting deterministic with respect to the information at time $t - 1$, at the cost of the reservoir size not being explicitly defined.
While this is analogous to a top-k averaging scheme, the full WRS scheme will follow from the same reasoning, with a more complicated definition of $\pi_t$. This is because all decision-making is performed using information that has already been seen.

Then we write $\bdw_{wrs} = \frac{1}{\sum_{t=1}^T \pi_t} \sum_{t=1}^T \pi_t \bdw_t$, 
and our goal is to control $R_\mathcal{D}(\bdw_{wrs})$.
Furthermore, define the cumulative loss as $M_{wrs} = \sum_{t=1}^T \pi_t \ell(\bdw_{t-1}, z_t)$. (We note that this term is a little different from the observed cumulative WRS loss because of the requirement of reweighting on the $s$-th survival.)

In the lemma, let $L_t = \pi_t \ell(\bdw_{t-1}, z_t)$ and let $U_t = \mathbb{E}[\pi_t \ell(\bdw_{t-1}, z_t) | z_1^{t-1}]$. Noting that $L_t$ is deterministic when conditioning on $z_1^{t}$, the assumptions are met.
Then with probability $1-\delta$, we have $\sum_{i=1}^T U_i < \tilde{L}_T + \sqrt{2C\log\Big(\frac{T}{\delta}\Big) \tilde{L}_T} + 7 C \log\Big(\frac{T}{\delta}\Big)$.

Now $\sum_{i=1}^T U_i = \sum_{t=1}^T \pi_t \mathbb{E} [\ell(\bdw_{t-1}, z_t) | z_1^{t-1}] = \sum_{t=1}^T \pi_t R_{\mathcal{D}}(\bdw_{t-1})$,
where the first equality uses the fact that $\pi_t$ is determined by the filtration $\mathcal{F}_{t-1}$,
and the second similarly uses the conditioning information to fix $\bdw_{t-1}$ and the i.i.d. assumption over $z$.

Finally, we have $R_\mathcal{D}(\bdw_{wrs}) = R_\mathcal{D}\Big( \frac{1}{\sum_{t=1}^T \pi_t} \sum_{t=1}^T \pi_t \bdw_t \Big) \leq \frac{1}{\sum_{t=1}^T \pi_t} \sum_{t=1}^T \pi_t R_\mathcal{D}(\bdw_t)$ from Jensen's inequality.

Dividing the previous statement by $\sum_{t=1}^T \pi_t$, we put the pieces together to conclude
$$ R_\mathcal{D}(\bdw_{wrs}) \leq \frac{M_{wrs}}{\sum_{t} \pi_t} + \sqrt{\frac{2 C \log \Big(\frac{T}{\delta}\Big) M_{wrs}}{(\sum_t \pi_t)^2}} + \frac{7C \log \Big(\frac{T}{\delta}\Big)}{\sum_t \pi_t}$$

Finally, note that $\sum_t \pi_t$ is the total time steps involved with the reservoir weights, which is at least $K_T s_T$ with $K_T$ weights and $s_T$ the minimal survival time of the reservoir. Substituting completes the proof.

\end{proof}

\textbf{Proof of Theorem~\ref{thm:risk_bound}}.

\begin{proof}
    Combining the penultimate form of Theorem~\ref{thm:wrs_bound} (with $\sum_t \pi_t$ instead of $K_T s_T$) and the deduction that $M_{wrs}/\sum_t \pi_t \leq M_T/T$ we have
$$ R_\mathcal{D}(\bdw_{wrs}) \leq \frac{M_T}{T} + \sqrt{\frac{2 C \log \Big(\frac{T}{\delta}\Big) M_T}{T\sum_t \pi_t}} + \frac{7C \log \Big(\frac{T}{\delta}\Big)}{\sum_t \pi_t}$$

Next, we turn our attention to the generic base-model online regret bound. From this we get $M_T/T \leq \frac{1}{T} \sum_t \ell(\bdw^*; z_t) + r(T)/T$.

We need to bound the cumulative loss of $\bdw^*$, that is $\frac{1}{T} \sum_t \ell(\bdw^*; z_t)$, in terms of the risk: $R_\mathcal{D}(\bdw^*)$.
Using an application of Hoeffding bound, we get
$$ \frac{1}{T}\sum_{t=1}^T \ell(\bdw^*, z_t) \leq R_\mathcal{D}(\bdw^*) + C \sqrt{\frac{\log \Big(\frac{1}{\delta}\Big)}{2T}} $$
which holds with probability $1-\delta$.

Combining the undesired tail events under the union bound, we get $$
R_\mathcal{D}(\bdw_{wrs}) \leq R_\mathcal{D}(\bdw^*) + \frac{r(T)}{T} + \sqrt{\frac{2 C \log \Big(\frac{T}{\delta}\Big) M_T}{T\sum_t \pi_t}} + \frac{7C \log \Big(\frac{T}{\delta}\Big)}{\sum_t \pi_t} + C \sqrt{\frac{\log \Big(\frac{1}{\delta}\Big)}{2T}}
$$
with probability $1- 2\delta$.
As before, substitute in $K_T s_T$ for $\sum_{t} \pi_t$ to complete.

\end{proof}

\textbf{Remark.} The use of the online regret bound also indicates that the base models $\bdw_t$ are improving over time. This means that their risks are decreasing, and hence the observed survival times are increasing. Thus even in the error term without $T$ in the denominator, $s_T$ is an increasing function of $T$ so we expect the error terms to decay in $T$.